\PassOptionsToPackage{sort, numbers, compress}{natbib}

\documentclass{article}

\usepackage{amsmath}    % standard maths package
\usepackage{amssymb}    % more mathematical symbols
\usepackage{amsthm}     % theorem environments
\usepackage{booktabs}   % for professional tables
\usepackage{dsfont}     % proper double-stroke fonts
\usepackage{etoolbox}   % for patching commands and making them robust
\usepackage{graphicx}   % figures
\usepackage{hyperref}   % hyperlinks
\usepackage{microtype}  % automated typographical adjustments
\usepackage{mleftright} % space fixes for equations
\usepackage{multirow}   % row/col spanning
\usepackage{nicefrac}   % nice fractions (inline)
\usepackage{paralist}   % in-paragraph enumerations
\usepackage{siunitx}    % SI units (and nice tables)
\usepackage{subcaption} % proper subcaptions
\usepackage{tikz}       % graphics
\usepackage{pgfplots}   % statistical graphics
\usepackage{soul}       % better underlining
\usepackage{tabu}       % improved tables
\usepackage{wrapfig}    % wrapping figures
\usepackage{xspace}     % the 'eater of spaces' (simplifies macro use in text)

\usepackage[final]{changes}

\pgfplotsset{
  compat                 = 1.16,  % Okay for Overleaf, I guess...
  filter discard warning = false, % Suppress warnings for filtered plots
  discard if not/.style 2 args={
    x filter/.append code={
      \edef\tempa{\thisrow{#1}}
      \edef\tempb{#2}
      \ifx\tempa%
        \tempb%
      \else%
        \def\pgfmathresult{inf}
      \fi
    }
  },
}

\pgfplotsset{%
  perfplot/.style = {%
    axis x line*      = bottom,
    axis y line*      = left,
    tick align        = outside,
    legend cell align = left,
    legend pos        = south east,
    legend columns    = -1,
    legend style      = {%
      draw         = lightgrey,
      font         = \fontsize{4}{5}\selectfont,
      at           = {(0.50,0.225)},
      anchor       = north,
    },
    legend image post style = {%
      scale = 0.5
    },
    grid              = major,
    major grid style  = {%
      lightgrey,
      thin
    },
    yticklabel style  = {
      /pgf/number format/precision=1,
      /pgf/number format/fixed,
      /pgf/number format/fixed zerofill,
    },
    thick,
    ymin              = 0.0,
    ymax              = 105,
    xlabel            = {Number of GCN layers},
    ylabel            = {Test accuracy~(in \%)},
    xtick             = {0,1, 2, 3, 4, 5, 6, 7, 8},
    ytick             = {0, 25, 50, 75, 100},
    height            = 3.75cm,
    width             = 0.45\linewidth,
    cycle list/Set1,
    /tikz/font = {\small},
  }
}

\usepgfplotslibrary{colorbrewer}

\usetikzlibrary{decorations.pathreplacing}
\usetikzlibrary{arrows.meta}
\newcommand{\method}{TOGL\xspace}

\definecolor{bleu}     {RGB}{ 49,140,231}
\definecolor{cardinal} {RGB}{196, 30, 58}
\definecolor{emerald}  {RGB}{ 80,200,120}
\definecolor{lightgrey}{RGB}{230,230,230}

\newtheorem{lemma}    {Lemma}
\newtheorem{theorem}  {Theorem}

\newcommand{\ball}[2]                   {\ensuremath{\mathrm{B}_{#1}\mleft(#2\mright)}}
\newcommand{\betti}[1]                  {\ensuremath{\beta_{#1}}}
\newcommand{\boundary}[1]               {\ensuremath{\partial_{#1}}\xspace}
\newcommand{\chaingroup}[1]             {\ensuremath{C_{#1}}\xspace}
\newcommand{\diagram}                   {\ensuremath{\mathcal{D}}\xspace}
\newcommand{\homologygroup}[1]          {\ensuremath{H_{#1}}\xspace}
\newcommand{\multiplicity}[2][0]        {\ensuremath{\mu_{#1}^{(#2)}}\xspace}
\newcommand{\naturals}                  {\ensuremath{\mathds{N}}\xspace}
\newcommand{\persistentbetti}        [2]{\ensuremath{\betti{#1}^{(#2)}}}
\newcommand{\persistenthomologygroup}[2]{\ensuremath{\homologygroup{#1}^{(#2)}}}
\newcommand{\reals}                     {\ensuremath{\mathds{R}}\xspace}
\newcommand{\simplicialcomplex}         {\ensuremath{\mathrm{K}}\xspace}
\newcommand{\wl}          [1]           {\ensuremath{\phi^{(#1)}}\xspace}

\DeclareMathOperator{\im}         {im}
\DeclareMathOperator{\ph}         {\mathbf{ph}}
\DeclareMathOperator{\persistence}{pers}
\DeclareMathOperator{\rank}       {rank}

\newcommand{\data}[1]{\textsc{#1}}

\usepackage{iclr2022_conference}
\usepackage{times}

\iclrfinalcopy

\usepackage{authblk}

\def\maketitle
{{\renewenvironment{tabular}[2][]{\begin{flushleft}}
                                 {\end{flushleft}}
\AB@maketitle}}

\begin{document}

\title{Topological Graph Neural Networks}

\author[1, 2, $\ast$]{Max Horn}
\author[3, $\ast$]{Edward De Brouwer}
\author[1, 2]{Michael Moor}
\author[3]{Yves Moreau}
\author[1, 2, 4, 5, $\dagger$]{Bastian~Rieck}
\author[1, 2, $\dagger$]{Karsten Borgwardt}
\affil[1]{Department of Biosystems Science and Engineering, ETH
Zurich, 4058 Basel, Switzerland}
\affil[2]{SIB Swiss Institute of Bioinformatics, Switzerland}
\affil[3]{ESAT-STADIUS, KU Leuven, 3001 Leuven, Belgium}
\affil[4]{Institute of AI for Health, Helmholtz Munich, 85764 Neuherberg, Germany}
\affil[5]{Technical University of Munich, 80333 Munich, Germany}

\affil[$\ast$]{These authors contributed equally.}
\affil[$\dagger$]{These authors jointly supervised this work.}

\maketitle

\begin{abstract}
Graph neural networks~(GNNs) are a powerful architecture for tackling
graph learning tasks, yet have been shown to be oblivious to eminent
substructures such as cycles.
We present \method, a novel \emph{layer} that incorporates global topological
information of a graph using persistent homology.
\method can be easily integrated into \emph{any type} of GNN and is
strictly more expressive~(in terms the Weisfeiler--Lehman graph
isomorphism test) than message-passing GNNs.
Augmenting GNNs with \method leads to improved predictive
performance for graph and node classification tasks, both on synthetic
data sets, which can be classified by humans using their topology but
not by ordinary GNNs, and on
real-world data.
\end{abstract}

%%%%%%%%%%%%%%%%%%%%%%%%%%%%%%%%%%%%%%%%%%%%%%%%%%%%%%%%%%%%%%%%%%%%%%%%
\section{Introduction}
%%%%%%%%%%%%%%%%%%%%%%%%%%%%%%%%%%%%%%%%%%%%%%%%%%%%%%%%%%%%%%%%%%%%%%%%

Graphs are a natural representation of structured data sets in many
domains, including bioinformatics, image processing, and social network
analysis.
Numerous methods address the two dominant graph learning tasks of graph
classification or node classification. In particular, \emph{graph neural
networks}~(GNNs) describe a flexible set of architectures for such
tasks and have seen many successful applications over recent
years~\citep{Wu21}. % TODO: check whether we have a better reference for this
At their core, many GNNs are based on iterative message passing
schemes \added{(see \citet{Shervashidze09b} for an introduction to iterative message passing in graphs and \citet{Sanchez21} for an introduction to GNNs)}.
Since these schemes are collating information over the
neighbours of every node, GNNs cannot necessarily capture certain
topological structures in graphs, such as cycles~\citep{bouritsas2021improving}.
These structures are highly relevant for applications that require
connectivity information, such as the analysis of molecular graphs~\citep{Hofer20, swenson2020persgnn}.

We address this issue by proposing a \ul{To}pological
\ul{G}raph \ul{L}ayer~(\method) that can be easily integrated into any
GNN to make it `topology-aware.' Our method is rooted in the emerging
field of topological data analysis~(TDA), which focuses on describing
coarse structures that can be used to study the shape of complex
structured and unstructured data sets at multiple scales. We thus obtain a generic way to
augment existing GNNs and increase their expressivity in graph learning
tasks.
\autoref{fig:Synthetic data performance} provides a motivational example that showcases the potential
benefits of using topological information: 
\begin{inparaenum}[(i)]
  \item high predictive performance is reached \emph{earlier} for
    a \emph{smaller} number of layers, and
  \item learnable topological representations outperform fixed ones if
    more complex topological structures are present in a data set.
\end{inparaenum}

\paragraph{Our contributions.}
We propose \method, a novel layer based on TDA concepts that can be
integrated into any GNN. Our layer is differentiable and capable of
learning contrasting topological representations of a graph.  We prove
that \method enhances expressivity of a GNN since it incorporates the
ability to work with multi-scale topological information in a graph.
Moreover, we show that \method improves predictive performance of several GNN
architectures when topological information is relevant for the
respective task.

%%%%%%%%%%%%%%%%%%%%%%%%%%%%%%%%%%%%%%%%%%%%%%%%%%%%%%%%%%%%%%%%%%%%%%%%
\begin{figure}[tbp]
  \centering

  \pgfplotsset{%
    jitter/.style = {%
      y filter/.code ={%
        \pgfmathparse{\pgfmathresult+rnd*#1}
      }
    },
  }

  \subcaptionbox{\data{Cycles}\label{sfig:Results cycles}}{%
    \begin{tikzpicture}
      \begin{axis}[%
        perfplot,
        /tikz/font            = {\tiny},
        ylabel                = {Test accuracy~(in \%)},
        xlabel                = {Number of layers/iterations},
        width                 = 0.40\linewidth,
        xmin                  = 1,
        xmax                  = 8, % Ignore deeper WL results; but we have them if anyone
                                   % wants to look at them!
        legend style          = {font=\tiny, at={(0.5,0.25)}},
        enlarge x limits      = 0.05,
        error bars/y dir      = both,
        error bars/y explicit = true,
        ymin =  20.0,
        ymax = 105.0,
      ]
        \addplot+[%
          x filter/.expression = {x < 1 ? nan : x},
          discard if not = {dataset}{Cycles},
          ] table[x = depth, y expr = 100 * \thisrow{val_acc_mean}, y error expr = 100 * \thisrow{val_acc_std}, col sep = comma] {Data/GCN_summary.csv};

        \addlegendentry{GCN}

        \addplot+[%
          x filter/.expression = {x < 1 ? nan : x},
          discard if not = {dataset}{Cycles},
          discard if not = {dim1}{True},
          discard if not = {fake}{False},
          jitter         = 2.5,
        ] table[x = depth, y expr = 100 * \thisrow{val_acc_mean}, y error expr = 100 * \thisrow{val_acc_std}, col sep = comma] {Data/TopoGNN_summary.csv};

        \addlegendentry{GCN-\method}

        \addplot+[%
          discard if not = {Name}{WL_Cycles},
        ] table[x = depth, y = accuracy, col sep = comma, y = accuracy, y error = sdev] {Data/WL_synthetic.csv};

        \addlegendentry{WL}

        \addplot+[%
          discard if not = {dataset}{Cycles},
          jitter         = -2.5,
        ] table[x = depth, y expr = 100 * \thisrow{val_acc_mean}, y error expr = 100 * \thisrow{val_acc_std}, col sep = comma] {Data/PH_fixed_filtration.csv};

        \addlegendentry{PH}
      \end{axis}
    \end{tikzpicture}%
        \begin{tikzpicture}
      %%%%%%%%%%%%%%%%%%%%%%%%%%%%%%%%%%%%%%%%%%%%%%%%%%%%%%%%%%%%%%%%%
      % Graph 1
      %%%%%%%%%%%%%%%%%%%%%%%%%%%%%%%%%%%%%%%%%%%%%%%%%%%%%%%%%%%%%%%%%

      \begin{scope}[xshift=0.20cm, yshift=1.00cm, scale=0.20]
        \coordinate (00) at (1.68, -0.07);
        \coordinate (01) at (1.61, 0.51);
        \coordinate (02) at (1.33, 1.02);
        \coordinate (03) at (0.88, 1.40);
        \coordinate (04) at (0.33, 1.59);
        \coordinate (05) at (-0.25, 1.57);
        \coordinate (06) at (-0.79, 1.34);
        \coordinate (07) at (-1.20, 0.93);
        \coordinate (08) at (-1.45, 0.40);
        \coordinate (09) at (-1.48, -0.18);
        \coordinate (10) at (-1.30, -0.73);
        \coordinate (11) at (-0.93, -1.19);
        \coordinate (12) at (-0.43, -1.48);
        \coordinate (13) at (0.15, -1.56);
        \coordinate (14) at (0.72, -1.44);
        \coordinate (15) at (1.20, -1.11);
        \coordinate (16) at (1.54, -0.64);
      \end{scope}

      \draw (00) -- (01);
      \draw (01) -- (02);
      \draw (02) -- (03);
      \draw (03) -- (04);
      \draw (04) -- (05);
      \draw (05) -- (06);
      \draw (06) -- (07);
      \draw (07) -- (08);
      \draw (08) -- (09);
      \draw (09) -- (10);
      \draw (10) -- (11);
      \draw (11) -- (12);
      \draw (12) -- (13);
      \draw (13) -- (14);
      \draw (14) -- (15);
      \draw (15) -- (16);
      \draw (00) -- (16);
      \foreach \v in {00,01,02,03,04,05,06,07,08,09,10,11,12,13,14,15,16}
      {
        \filldraw (\v) circle (1pt);
      }
  
      %%%%%%%%%%%%%%%%%%%%%%%%%%%%%%%%%%%%%%%%%%%%%%%%%%%%%%%%%%%%%%%%%
      % Graph 2
      %%%%%%%%%%%%%%%%%%%%%%%%%%%%%%%%%%%%%%%%%%%%%%%%%%%%%%%%%%%%%%%%%

      \begin{scope}[xshift=0.125cm, scale = 0.125]
        \coordinate (00) at (0.80, 2.73);
        \coordinate (01) at (1.65, 3.88);
        \coordinate (02) at (0.11, 3.97);
        \coordinate (03) at (-0.73, 1.87);
        \coordinate (04) at (-2.15, 1.74);
        \coordinate (05) at (-1.43, 3.09);
        \coordinate (06) at (-1.08, -1.32);
        \coordinate (07) at (-2.08, -0.11);
        \coordinate (08) at (-0.75, 0.10);
        \coordinate (09) at (2.32, 1.83);
        \coordinate (10) at (3.19, 2.87);
        \coordinate (11) at (3.71, 1.39);
        \coordinate (12) at (0.69, -1.99);
        \coordinate (13) at (2.21, -1.55);
        \coordinate (14) at (3.36, -0.47);
        \coordinate (15) at (2.11, 0.18);
        \coordinate (16) at (0.78, -0.58);
      \end{scope}

      \draw (00) -- (01);
      \draw (01) -- (02);
      \draw (00) -- (02);
      \draw (03) -- (04);
      \draw (04) -- (05);
      \draw (03) -- (05);
      \draw (06) -- (07);
      \draw (07) -- (08);
      \draw (06) -- (08);
      \draw (09) -- (10);
      \draw (10) -- (11);
      \draw (09) -- (11);
      \draw (12) -- (13);
      \draw (13) -- (14);
      \draw (14) -- (15);
      \draw (15) -- (16);
      \draw (12) -- (16);
      \foreach \v in {00,01,02,03,04,05,06,07,08,09,10,11,12,13,14,15,16}
      {
        \filldraw (\v) circle (1pt);
      }
    \end{tikzpicture}
  }
  \subcaptionbox{\data{Necklaces}\label{sfig:Results necklaces}}{%
    \begin{tikzpicture}
      \begin{axis}[%
        perfplot,
        /tikz/font            = {\tiny},
        ylabel                = {Test accuracy~(in \%)},
        xlabel                = {Number of layers/iterations},
        width                 = 0.40\linewidth,
        xmin                  = 1,
        xmax                  = 8,  % Ignore deeper WL results; but we have them if anyone
                                    % wants to look at them!
        %
        legend style          = {font=\tiny, at={(0.5,0.25)}},
        enlarge x limits      = 0.05,
        error bars/y dir      = both,
        error bars/y explicit = true,
        ymin =  23.0,
        ymax = 105.0,
      ]
        \addplot+[%
          x filter/.expression = {x < 1 ? nan : x},
          discard if not = {dataset}{Necklaces},
        ] table[x = depth, y expr = 100 * \thisrow{val_acc_mean}, y error expr = 100 * \thisrow{val_acc_std}, col sep = comma] {Data/GCN_summary.csv};

        \addlegendentry{GCN}

        \addplot+[%
          x filter/.expression = {x < 1 ? nan : x},
          discard if not = {dataset}{Necklaces},
          discard if not = {dim1}{True},
          discard if not = {fake}{False},
        ] table[x = depth, y expr = 100 * \thisrow{val_acc_mean}, y error expr = 100 * \thisrow{val_acc_std}, col sep = comma] {Data/TopoGNN_summary.csv};

        \addlegendentry{GCN-\method}

        \addplot+[%
          x filter/.expression = {x >= 9 ? nan : x},
          discard if not       = {Name}{WL_Necklaces},
        ] table[x = depth, y = accuracy, col sep = comma, y = accuracy, y error = sdev] {Data/WL_synthetic.csv};

        \addlegendentry{WL}

        \addplot+[%
          discard if not = {dataset}{Necklaces},
        ] table[x = depth, y expr = 100 * \thisrow{val_acc_mean},
        y error expr = 100 * \thisrow{val_acc_std}, col sep = comma] {Data/PH_fixed_filtration.csv};

        \addlegendentry{PH}

      \end{axis}
    \end{tikzpicture}
       \begin{tikzpicture}
      \begin{scope}[yshift = 1.50cm, scale = 0.25]
        \coordinate (00) at (1.92, 0.70);
        \coordinate (01) at (1.56, 0.53);
        \coordinate (02) at (1.18, 0.36);
        \coordinate (03) at (0.81, 0.20);
        \coordinate (04) at (0.41, 0.03);
        \coordinate (05) at (0.03, 0.37);
        \coordinate (06) at (-0.44, 0.28);
        \coordinate (07) at (-0.68, -0.14);
        \coordinate (08) at (-0.49, -0.61);
        \coordinate (09) at (-0.78, -0.92);
        \coordinate (10) at (-1.07, -1.22);
        \coordinate (11) at (-1.35, -1.51);
        \coordinate (12) at (0.41, -0.49);
        \coordinate (13) at (0.00, -0.78);
        \coordinate (14) at (0.12, -0.16);
        \coordinate (15) at (-0.22, -0.40);
      \end{scope}
      \draw (00) -- (01);
      \draw (01) -- (02);
      \draw (02) -- (03);
      \draw (03) -- (04);
      \draw (04) -- (05);
      \draw (05) -- (06);
      \draw (06) -- (07);
      \draw (07) -- (08);
      \draw (08) -- (09);
      \draw (09) -- (10);
      \draw (10) -- (11);
      \draw (12) -- (13);
      \draw (14) -- (15);
      \draw (04) -- (12);
      \draw (04) -- (14);
      \draw (08) -- (13);
      \draw (08) -- (15);
      \foreach \v in {00,01,02,03,04,05,06,07,08,09,10,11,12,13,14,15}
      {
        \filldraw (\v) circle (1pt);
      }
      \begin{scope}[rotate=60, scale = 0.50]
        \coordinate (00) at (1.11, 1.60);
        \coordinate (01) at (0.95, 1.27);
        \coordinate (02) at (0.81, 0.92);
        \coordinate (03) at (0.70, 0.55);
        \coordinate (04) at (0.66, 0.13);
        \coordinate (05) at (0.28, -0.03);
        \coordinate (06) at (-0.06, -0.22);
        \coordinate (07) at (-0.37, -0.45);
        \coordinate (08) at (-0.68, -0.73);
        \coordinate (09) at (-1.07, -0.59);
        \coordinate (10) at (-1.46, -0.54);
        \coordinate (11) at (-1.83, -0.54);
        \coordinate (12) at (0.91, -0.19);
        \coordinate (13) at (1.05, 0.06);
        \coordinate (14) at (-0.50, -1.10);
        \coordinate (15) at (-0.79, -1.11);
      \end{scope}
      \draw (00) -- (01);
      \draw (01) -- (02);
      \draw (02) -- (03);
      \draw (03) -- (04);
      \draw (04) -- (05);
      \draw (05) -- (06);
      \draw (06) -- (07);
      \draw (07) -- (08);
      \draw (08) -- (09);
      \draw (09) -- (10);
      \draw (10) -- (11);
      \draw (12) -- (13);
      \draw (14) -- (15);
      \draw (04) -- (12);
      \draw (04) -- (13);
      \draw (08) -- (14);
      \draw (08) -- (15);
      \foreach \v in {00,01,02,03,04,05,06,07,08,09,10,11,12,13,14,15}
      {
        \filldraw (\v) circle (1pt);
      }
    \end{tikzpicture} 
  }
  \caption{%
    % TODO: should we show the zero layer results as well or is this
    % hard to explain for such a figure?
    As a motivating example%~(see \autoref{sec:Synthetic Data Sets})
    ,
    we introduce two topology-based data sets whose
    graphs can be easily distinguished by humans; the left data set can
    be trivially classified by all topology-based methods, while the
    right data set necessitates \emph{learnable} topological features.
    We show the performance of
    %
    % FIXME: don't use `inparaenum` here or it will mess up the
    % `autoref`?
    (i)~a GCN with $k$ layers,
    (ii)~our layer \method~(integrated into a GCN with $k-1$ layers): GCN-\method,
    (iii)~the Weisfeiler--Lehman~(WL) graph kernel~\added{using vertex degrees as the node features}, and
    (iv)~a method based on static topological features~(PH).
 Next to the performance charts, we display examples of graphs of each class for each of the data sets.
  }
  \label{fig:Synthetic data performance}
\end{figure}
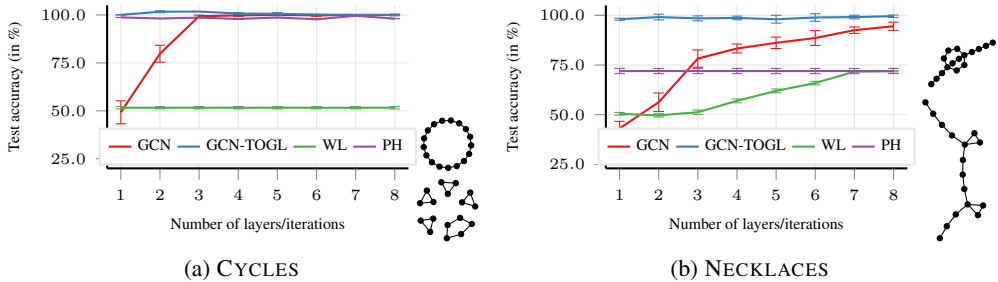
%%%%%%%%%%%%%%%%%%%%%%%%%%%%%%%%%%%%%%%%%%%%%%%%%%%%%%%%%%%%%%%%%%%%%%%%

%%%%%%%%%%%%%%%%%%%%%%%%%%%%%%%%%%%%%%%%%%%%%%%%%%%%%%%%%%%%%%%%%%%%%%%%
\section{Background: Computational Topology}
\label{sec:Computational Topology}
%%%%%%%%%%%%%%%%%%%%%%%%%%%%%%%%%%%%%%%%%%%%%%%%%%%%%%%%%%%%%%%%%%%%%%%%

We consider undirected graphs of the
form $G = (V, E)$ with a set of vertices~$V$ and a set of edges~$E
\subseteq V \times V$. The basic topological features of such
a graph~$G$ are the number of connected components~\betti{0} and
the number of cycles~\betti{1}. These counts are also known as
the $0$-dimensional and $1$-dimensional \emph{Betti numbers},
respectively; they are invariant under graph isomorphism~\citep[pp.~103--133]{Hatcher02}
and can be computed efficiently.
The expressivity of Betti numbers can be increased using a \emph{graph
filtration}, i.e.\ a sequence of nested subgraphs of~$G$ such that 
$\emptyset = G^{(0)} \subseteq G^{(1)} \subseteq G^{(2)} \subseteq
\dots \subseteq G^{(n-1)} \subseteq G^{(n)} = G$.
A filtration makes it possible to obtain more insights into the graph by
`monitoring' topological features of \emph{each}~$G^{(i)}$ and
calculating their topological relevance, also referred to as their
\emph{persistence}. If a topological feature appears for the first time
in~$G^{(i)}$ and disappears in~$G^{(j)}$, we assign this feature
a persistence of $j - i$. Equivalently, we can represent the
feature as a tuple~$(i, j)$, which we collect in
a \emph{persistence diagram}~\diagram. If a feature never disappears, we
represent it by a tuple~$(i, \infty)$; such features are the ones that
are counted for the respective Betti numbers.
This process was formalised and extended to a wider class of structured
data sets, namely simplicial complexes, and is known under the name of
\emph{persistent homology}. One of its core concepts is the use of
a filtration function~$f\colon V \to \reals^d$, with $d = 1$ typically,
to accentuate certain structural features of a graph. This replaces the
aforementioned tuples of the form~$(i,j)$ by tuples based on the values
of~$f$, i.e.\ $(f_i, f_j)$. Persistent homology has shown excellent
promise in different areas of machine learning research~(see
\citet{Hensel21} for a recent survey and \autoref{sec:TDA} for a more
technical description of persistent homology), with existing work stressing that
choosing or learning an appropriate filtration function~$f$ is crucial
for high predictive performance~\citep{Hofer20, Zhao19}.

\paragraph{Notation.}
We denote the calculation of persistence diagrams of a graph~$G$ under
some filtration $f$ by $\ph(G, f)$. This will result in two persistence
diagrams~$\diagram^{(0)}, \diagram^{(1)}$, containing information about
topological features in dimension~$0$~(connected components) and
dimension~$1$~(cycles).
The cardinality of $\diagram^{(0)}$ is equal to the number
of nodes~$n$ in the graphs and each tuple in the \mbox{$0$}-dimensional diagram
is associated with the \emph{vertex} that created it. The cardinality of
$\diagram^{(1)}$ is the number of cycles; we pair each tuple in
$\diagram^{(1)}$ with the \emph{edge} that created it. Unpaired
edges---edges that are not used to create a cycle---are assigned a `dummy'
tuple value, such as $(0,0)$. All other edges will be paired with the
maximum value of the filtration, following previous work by
\citet{Hofer17}.

%%%%%%%%%%%%%%%%%%%%%%%%%%%%%%%%%%%%%%%%%%%%%%%%%%%%%%%%%%%%%%%%%%%%%%%%
\section{Related Work}
%%%%%%%%%%%%%%%%%%%%%%%%%%%%%%%%%%%%%%%%%%%%%%%%%%%%%%%%%%%%%%%%%%%%%%%%

Graph representation learning has received a large amount of attention
by the machine learning community. \emph{Graph kernel methods} address
graph classification via~(implicit or explicit) embeddings in
Reproducing Kernel Hilbert Spaces~\citep{Borgwardt20, Kriege20,
Nikolentzos19}. While powerful and expressive, they cannot capture
partial similarities between neighbourhoods. This can be achieved by
\emph{graph neural networks}, which typically employ message passing on
graphs to learn hidden representations of graph structures~\citep{Kipf17, Wu21}.
Recent work in this domain is abundant and includes attempts to utilise
additional substructures~\citep{bouritsas2021improving} as well as
defining higher-order message passing schemes~\citep{Morris19} or
algorithms that generalise message passing to more complex
domains~\citep{Bodnar21a}.

Our approach falls into the realm of topological data
analysis~\citep{Edelsbrunner10} and employs \emph{persistent homology},
a technique for calculating topological features---such as connected
components and cycles---of structured data sets. These features are
known to be highly characteristic, leading to successful topology-driven
graph machine learning approaches~\citep{Hofer17, Hofer20, Rieck19b,
Zhao19}. At their core is the notion of
a \emph{filtration}, i.e.\ a sequence of nested subgraphs~(or simplicial
complexes in a higher-dimensional setting). Choosing the right
filtration is known to be crucial for obtaining good
performance~\citep{Zhao19}.
This used to be a daunting task because persistent homology calculations
are inherently discrete. Recent advances in proving differentiability
enable proper end-to-end training of persistent
homology~\citep{carriere2021optimizing}, thus opening the door for
hybrid methods of increased expressivity by integrating the somewhat
complementary view of topological features.
Our method \method builds on the theoretical framework by
\citet{Hofer20}, who 
\begin{inparaenum}[(i)]
  \item demonstrated that the \emph{output} of a GNN can be used to `learn'
    one task-specific filtration function, and
  \item described the conditions under which a filtration function~$f$
    is differentiable. 
\end{inparaenum}
This work culminated in GFL, a topological \texttt{readout} function that
exhibited improved predictive performance for graph classification
tasks.
We substantially extend the utility of topological features by making
existing GNNs `topology-aware' through the development of a generic
layer that makes topological information available to \emph{all}
downstream GNN layers: \method can be integrated into any GNN
architecture, enabling the creation of hybrid models whose
expressivity is provably more powerful than that of a GNN alone.
Moreover, while GFL only uses the output of a GNN to drive the
calculation of topological features by means of a single
filtration~(thus limiting the applicability of the approach, as the
topological features cannot inform the remainder of a network),
\method learns multiple filtrations of a graph in an end-to-end manner.
More precisely, \method includes topological information in the hidden
representations of nodes, enabling networks to change the importance of
the topological signal.
Closest to the scope of \method is \citet{Zhao20}, who enhanced
GNNs using topological information for node classification.
In their framework, however, topology is only used to
provide additional scalar-valued weights for the message passing scheme,
and topological features are only calculated over small neighbourhoods,
making use of a static vectorisation technique of persistence
diagrams. \added{Similarly, \citet{wong2021persistent} use static, i.e.\ non-learnable, topological features for 3D shape segmentation.}
By contrast, \method, being end-to-end differentiable, is more general and
permits the calculation of topological features at all
scales---including graph-level features---as well as an integration into
arbitrary GNN architectures.

%%%%%%%%%%%%%%%%%%%%%%%%%%%%%%%%%%%%%%%%%%%%%%%%%%%%%%%%%%%%%%%%%%%%%%%%
\section{\method: A Topological Graph Layer}
\label{sec:Our Method} % The label is different from the subsection name
                       % on purpose.
%%%%%%%%%%%%%%%%%%%%%%%%%%%%%%%%%%%%%%%%%%%%%%%%%%%%%%%%%%%%%%%%%%%%%%%%

\method is a new type of graph neural network layer that is
capable of utilising multi-scale topological information of input
graphs.  In this section, we give a brief overview of the components of
this layer before discussing algorithmic details, theoretical
expressivity, computational complexity, and limitations.
\autoref{fig:overview} presents an overview of our method~(we show only
a single graph being encoded, but in practice, the layer operates on
\emph{batches} of graphs).
%
%%%%%%%%%%%%%%%%%%%%%%%%%%%%%%%%%%%%%%%%%%%%%%%%%%%%%%%%%%%%%%%%%%%%%%%%
\begin{figure*}[t]
    \captionsetup[subfigure]{font=scriptsize} %tiny} %labelformat=empty} 
  \centering
  \resizebox{\textwidth}{!}{%
    \subcaptionbox{ {\scriptsize node attributes $x^{(v)}$ } \label{fig:graph1}}{%
    \begin{tikzpicture}
      \hspace{0.1cm}
      \coordinate (A) at (0.0, 0.65); 
      \coordinate (B) at (0.5, 0.65); 
      \coordinate (C) at (0.0, 1.15); 

        \coordinate[label={[xshift=0.2cm]\scriptsize \textcolor{cardinal} {$x^{(v)} \in \reals^d $}}] (D) at (0.5, 1.15); 
      \coordinate (E) at (-0.25, 1.4); 

      \foreach \c in {A,B,C,E}
      {
        \filldraw[black] (\c) circle (1pt);
      }
      \filldraw[cardinal] (D) circle (1pt);
 
      \draw (A) -- (B) -- (D) -- (C) -- cycle;
      \draw (C) -- (E);

       \draw[rounded corners=1ex] (-0.5,-0.5) rectangle (1.5,2.8);
    \end{tikzpicture}
    \hspace{0.3cm}%\vspace{1cm}
    }
    \hspace{0.1cm}
    \subcaptionbox{{\scriptsize node map $\Phi$} \label{fig:mlp}}{%

    \begin{tikzpicture} 
        \node[anchor=south west,inner sep=0] at (1.2,1.0) {\includegraphics[width=0.05\textwidth]{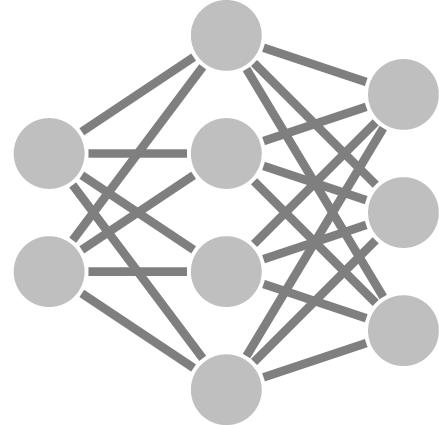}};
        
        \draw [-{Stealth[round]}] (1.0,0.8) -- (2.2,0.8);

        %\draw[color=cardinal] (0.88,1.15) rectangle (1.22,1.65);
        \node[color=cardinal] at (1.075, 1.39) { \scriptsize $x^{(v)}$ };
        
        \node[color=cardinal] at (2.3, 1.7) {\scriptsize $a_1^{(v)}$ };
        \node[color=cardinal, rotate=90] at (2.20,1.45) { \scalebox{.4}{\ldots}  };
        \node[color=cardinal] at (2.3, 1.25) {\scriptsize $a_k^{(v)}$ };

    \end{tikzpicture}
    \vspace{0.95cm}
    \hspace{0.2cm}
  }
  \hspace{0.05cm} %0.1
    \subcaptionbox{{ \scriptsize $k$ views of graph $G$  } \label{fig:graph2}}{%
    \hspace{0.25cm}
    \begin{tikzpicture}
      \coordinate[label={[xshift=0.1cm]\scriptsize 2}] (A) at (0.1, 0.0); 
      \coordinate[label={[xshift=0.1cm]\scriptsize 1}] (B) at (0.6, 0.0); 
      \coordinate[label={[xshift=0.0cm]\scriptsize 2}] (C) at (0.1, 0.5); 

        \coordinate[label={[xshift=0.22cm, yshift=-0.05cm]\scriptsize \textcolor{cardinal}{$1 = a_k^{(v)}$}}] (D) at (0.6, 0.5); 
      \coordinate[label={[xshift=0.0cm]\scriptsize 1}] (E) at (-0.15, 0.75); 

      \foreach \c in {A,B,C,D,E}
      {
        \filldraw[black] (\c) circle (1pt);
      }
      
      \filldraw[cardinal] (D) circle (1pt);
      \draw (A) -- (B) -- (D) -- (C) -- cycle;
      \draw (C) -- (E);
      
      \node[rotate=90] at (0.35,1.1) {\ldots};

      \coordinate[label={[xshift=0.1cm]\scriptsize 3}] (A2) at (0.1, 1.6); 
      \coordinate[label={[xshift=0.1cm]\scriptsize 1}] (B2) at (0.6, 1.6); 
      \coordinate[label={[xshift=0.0cm]\scriptsize 1}] (C2) at (0.1, 2.1); 

        \coordinate[label={[xshift=0.22cm, yshift=-0.05cm]\scriptsize \textcolor{cardinal}{$2 = a_1^{(v)}$}}] (D2) at (0.6, 2.1); 
      \coordinate[label={[xshift=0.0cm]\scriptsize 2}] (E2) at (-0.15, 2.35); 

      \foreach \c in {A2,B2,C2,D2,E2}
      {
        \filldraw[black] (\c) circle (1pt);
      }
 
    \filldraw[cardinal] (D2) circle (1pt);
      \draw (A2) -- (B2) -- (D2) -- (C2) -- cycle;
      \draw (C2) -- (E2);

      %\draw [decorate,decoration={brace,amplitude=5pt},xshift=-0.2cm,yshift=0pt]
      %      (-0.5,0) -- (-0.5,2.0) node [black,midway,xshift=-0.6cm]{};
      \draw [rounded corners=1ex] (-0.5,-0.5) rectangle (1.5,2.8);
    \end{tikzpicture}
    }
    \hspace{0.15cm}
  \subcaptionbox{ {\scriptsize filtrations $f_i$ } \label{fig:sublevelset}}{%
    \hspace{0.25cm}
    \begin{tikzpicture}
      \node[anchor=south west,inner sep=0] at (1.2,1.0) {\includegraphics[width=0.05\textwidth]{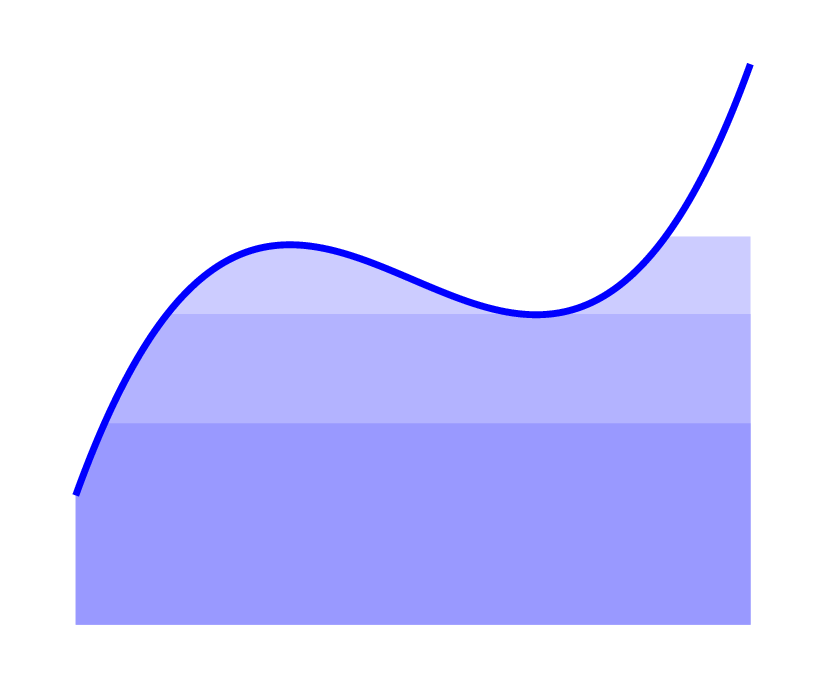}};     
        \draw [-{Stealth[round]}] (1.2,0.9) -- (2.2,0.9);
    \end{tikzpicture}
    \hspace{0.25cm}
    \vspace{1cm}
  }
  \subcaptionbox{ {\scriptsize Persistence diagrams} \label{fig:graph3}}{%
  %\vspace{-1cm}
  \hspace{0.5cm}
  \begin{tikzpicture}
      \vspace{-1cm}
  \fill (0.1,0.5)  circle[radius=0.5pt];
  \fill (0.1,0.6)  circle[radius=0.5pt];
  
  \draw (0.0,0.2) -- (0.5,0.2);
  \draw (0.0,0.2) -- (0.5,0.7);
  \draw (0.0,0.2) -- (0.0,0.7);

  \node[rotate=90] at (0.2,1.2) {\ldots};

  \fill (0.1,2.1)  circle[radius=0.5pt];
  \fill (0.3,2.1)  circle[radius=0.5pt];

  \draw (0.0,1.7) -- (0.5,1.7);
  \draw (0.0,1.7) -- (0.5,2.2);
  \draw (0.0,1.7) -- (0.0,2.2);

  \draw [-{Stealth[round]}] (1.5,0.57) -- (2.2,0.57);
      \node at (1.8,0.85) {$\Psi$};

  \draw [rounded corners=1ex] (-0.5,-0.5) rectangle (1.0,2.8);

  \end{tikzpicture}
  %\vspace{0.5cm}
  %\hspace{0.5cm} %0.5
  }
  %\hspace{0.1cm}
  \subcaptionbox{ {\scriptsize aggregation } \label{fig:graph4}}{% 
  \begin{tikzpicture}
    \draw (-0.60,0.0) rectangle (-0.05,0.5);
      \node at (-0.35, 0.2) { \fontsize{5}{5}\selectfont $\Psi[v]$ };
  
    %\node[rotate=90] at (-0.375,1.0) {\ldots};
    \node[] at (-0.385,1.0) {$+$};

    %\draw (-0.57,1.5) rectangle (-0.23,2.0);
    %\draw[color=cardinal] (-0.57,2.05) rectangle (-0.23,2.55);
    %\node[color=cardinal] at (-0.375, 2.29) { \scriptsize $x^{(v)}$ };

    \draw[color=cardinal] (-0.60,1.5) rectangle (-0.05,2.0);
      \node[color=cardinal] at (-0.32, 1.74) { \scriptsize $x^{(v)}$ };

    \draw [decorate,decoration={brace,amplitude=5pt, mirror},xshift=0.2cm,yshift=0pt]
           (-0.2,0) -- (-0.2,2.0) node [black,midway,xshift=0.6cm]{};
    
    \draw[color=cardinal] (0.33,0.82) rectangle (0.90,1.32);
    \node[color=cardinal] at (0.63, 1.05) { \scriptsize $ \tilde{x}^{(v)} $ };   
    %\node[inner sep=0] at (0.3,1.2) {\includegraphics[width=0.035\textwidth]{figures/outlayer.png}};
  \end{tikzpicture}
  %\hspace{0.2cm} %0.5
  \vspace{0.5cm}
  }
  %\hspace{0.1cm}
  \subcaptionbox{ {\scriptsize output $\tilde{x}^{(v)}$ }  \label{fig:graph5}}{%
  %\hspace{0.2cm} %0.5
  \begin{tikzpicture}
  %\draw[color=cardinal] (-1.07,1.25) rectangle (-0.63,1.75);
  %\node[color=cardinal] at (-0.875, 1.50) { \scriptsize $ \tilde{x}^{(v)} $ };   
  %\node[inner sep=0] at (-1.8,1.5) {\includegraphics[width=0.035\textwidth]{figures/outlayer.png}};

      \coordinate (A) at (0.0, 0.65); 
      \coordinate (B) at (0.5, 0.65); 
      \coordinate (C) at (0.0, 1.15); 

        \coordinate[label={[xshift=0.2cm]\scriptsize \textcolor{cardinal}{$\tilde{x}^{(v)} \in \reals^d $}}] (D) at (0.5, 1.15); 
      \coordinate (E) at (-0.25, 1.4); 

      \foreach \c in {A,B,C,E}
      {
        \filldraw[black] (\c) circle (1pt);
      }
      \filldraw[cardinal] (D) circle (1pt);
 
      \draw (A) -- (B) -- (D) -- (C) -- cycle;
      \draw (C) -- (E);

     \draw [rounded corners=1ex] (-0.5,-0.5) rectangle (1.5,2.8);
  \end{tikzpicture}
  %\vspace{1.0cm}
  }
}
  \caption{%
      Overview of TOGL, our topological graph layer. \subref{fig:graph1}) The node attributes $x^{(v)} \in \reals^d $ of graph $G$ serve as input. 
      \subref{fig:mlp}) A network $\Phi$ maps $x^{(v)}$ to $k$ node values $\{a_1^{(v)}, \dots,  a_k^{(v)}\} \subset \reals$.  
      \subref{fig:graph2}) Applying $\Phi_k$ to the attributes of each vertex $v$ results in $k$ views of~$G$. 
      \subref{fig:sublevelset}) A vertex filtration $f_i$ is computed for the $i$th view of~$G$. 
      \subref{fig:graph3}) A set of $k$ persistence diagrams is encoded via the embedding function $\Psi$, where $\Psi[v]$ denotes the embedding of vertex~$v$.
      \subref{fig:graph4}) The embedded topological features are combined with the input attributes $x^{(v)}$.
      \subref{fig:graph5}) Finally, this yields $\tilde{x}^{(v)} \in \reals^{d}$, which acts as a new node representation augmented with multi-scale topological information.
  } 
  \label{fig:overview}
\end{figure*}
%%%%%%%%%%%%%%%%%%%%%%%%%%%%%%%%%%%%%%%%%%%%%%%%%%%%%%%%%%%%%%%%%%%%%%%%

The layer takes as input a graph $G = (V, E)$ equipped with a set of~$n$ vertices~$V$
and a set of edges~$E$, along with a set of $d$-dimensional node
attribute vectors $x^{(v)} \in \reals^d$ for $v \in V$. These node
attributes can either be node features of a data set or hidden
representations learnt by some GNN.
We employ a family of $k$ vertex filtration functions of the form $f_i
\colon \reals^d \to \reals$ for $i = 1, \dots, k$.
Each filtration function~$f_i$ can focus on different properties of the
graph. The image of~$f_i$ is finite and results in a set of node
values $a_i^{(1)} < \dots < a_i^{(n)}$ such that the graph~$G$ is
filtered according to $\emptyset =  G_i^{(0)} \subseteq G_i^{(1)}
\subseteq \dots \subseteq G_i^{(n)} = G$, where $G_i^{(j)}
= \left(V_i^{(j)}, E_i^{(j)}\right)$,
with $V_i^{(j)} := \left\{v \in V \mid f_i\left(x^{(v)}\right) \leq a_i^{(j)} \right\}$, and
$E_i^{(j)} := \left\{ v,w \in E \mid \max\left\{f_i\left(x^{(v)}\right), f_i\left(x^{(w)}\right)\right\} \leq a_i^{(j)} \right\}$.
Given this filtration, we calculate a set of persistence diagrams, i.e.\
$\ph(G, f_i) = \left\{\diagram_i^{(0)}, \dots, \diagram_i^{(l)}
\right\}$.
We fix $l = 1$~(i.e.\ we are capturing connected components and cycles)
to simplify our current implementation, but our layer can
be extended to arbitrary values of~$l$~(see
\autoref{sec:Regular Graphs} for a discussion).
In order to benefit from representations that are trainable end-to-end,
we use an \emph{embedding function} $\Psi^{(l)}\colon\left\{\diagram_1^{(l)},
\dots, \diagram_k^{(l)}\right\} \to\reals^{ n' \times d}$ for embedding persistence
diagrams into a high-dimensional space that will be used to obtain the
vertex representations, where $n'$ is the number of \emph{vertices} $n$ if $l=0$ and the number of \emph{edges} if $l=1$.  This step is crucial as it enables us to use
the resulting topological features as node features, making \method
a layer that can be integrated into arbitrary GNNs. We later
explain the precise mapping of $\Psi^{(l)}$ from a set of diagrams to the
elements of a graph.

\textbf{Details on filtration computation and output generation.}
We compute our family of~$k$ vertex-based filtrations using
$\Phi\colon\reals^d\to\reals^k$, an MLP with a single hidden layer, such
that $f_i := \pi_i \circ \Phi$, i.e.\ the projection of $\Phi$ to the
$i$th dimension. We apply $\Phi$ to the hidden representations $x^{(v)}$
of all vertices in the graph.
Moreover, we treat the resulting persistence diagrams in dimension~$0$
and $1$ differently.
For dimension~$0$, we have a bijective mapping of tuples in the
persistence diagram to the vertices of the graph, which was previously
exploited in topological representation learning~\citep{Moor20}.
Therefore, we aggregate $\Psi^{(0)}$ with the original node attribute
vector~$x^{(v)}$ of the graph in a residual fashion, i.e.\
$\tilde{x}^{(v)} = x^{(v)} + \Psi^{(0)}\left(\diagram_1^{(0)}, \dots,
\diagram_k^{(0)}\right) \big[v\big]$, where $\Psi^{(0)}[v]$ denotes
taking $v$th row of $\Psi^{(0)}$~(i.e\ the topological embedding of vertex~$v$).
The output of our layer for dimension~$0$ therefore results in a new
representation~$\tilde{x}^{(v)} \in \reals^{d}$ for each vertex~$v$,
making it compatible with any subsequent~(GNN) layers.
By contrast, $\Psi^{(1)}$ is pooled into a graph-level representation,
to be used in the final classification layer of a GNN. This is necessary
because there is no bijective mapping to the vertices, but rather to edges. For stability reasons~\citep{Bendich20},
we consider it more useful to have this information available \emph{only} on the graph level.
%Yet, it was
%shown that the precise selection of edges that give rise to topological
%features can be unstable~\citep{Bendich20}, so we consider it more
%useful to have this information available \emph{only} on the graph
%level.
For further details on the computational aspects, please refer to
\autoref{sec:Computational details}.

\textbf{Complexity and limitations.}
Persistent homology can be calculated efficiently for dimensions~$0$
and~$1$, having a worst-case complexity of $\mathcal{O}\left(m \alpha\left(m\right)\right)$
for a graph with~$m$ sorted edges, where $\alpha(\cdot)$  is the extremely
slow-growing inverse Ackermann function, which can be treated as
constant for all intents and purposes. The calculation of~$\ph$ is
therefore dominated by the complexity of sorting all edges, i.e.\
$\mathcal{O}\left(m \log m\right)$, making our approach efficient and
scalable. Higher-dimensional persistent homology calculations
unfortunately do \emph{not} scale well, having a worst-case complexity of
$\mathcal{O}\left(m^d\right)$ for calculating \mbox{$d$-dimensional}
topological features, which is why we restrict ourselves to~$l = 1$ here.
Our approach is therefore limited to connected components and cycles.
Plus, our filtrations are incapable of assessing topological feature interactions; this would require learning 
\emph{multifiltrations}~\citep{Carlsson09b}, which
do not afford a concise, efficient representation as the scalar-valued
filtrations discussed in this paper. We therefore leave their treatment
to future work.

%%%%%%%%%%%%%%%%%%%%%%%%%%%%%%%%%%%%%%%%%%%%%%%%%%%%%%%%%%%%%%%%%%%%%%%%
\subsection{Choosing an Embedding Function $\Psi$}
%%%%%%%%%%%%%%%%%%%%%%%%%%%%%%%%%%%%%%%%%%%%%%%%%%%%%%%%%%%%%%%%%%%%%%%%

The embedding function~$\Psi$ influences the resulting
representation of the persistence diagrams calculated by our approach.
It is therefore crucial to pick a class of functions that are
sufficiently powerful to result in expressive representations of
a persistence diagram~$\diagram$.
We consider multiple types of embedding functions~$\Psi$, namely
\begin{inparaenum}[(i)]
  \item a novel approach based on \texttt{DeepSets}~\citep{zaheer2017deepsets},
  \item the \emph{rational hat} function introduced by \citet{hofer2019learning}, as well as 
  \item the triangle point transformation,
  \item the Gaussian point transformation, and
  \item the line point transformation, 
\end{inparaenum}
with the last three transformations being introduced in~\citet{Carriere20a}.
Except for the deep sets approach, all of these transformations are
\emph{local} in that they apply to a single point in
a persistence diagram without taking the other points into account.
These functions can therefore be decomposed as $\Psi^{(j)}\left(\diagram_1^{(j)}, \dots,
\diagram_k^{(j)}\right) \big[v\big]  = 
\widetilde{\Psi}^{(j)}\left(\diagram_1^{(j)}[v], \dots,
\diagram_k^{(j)}[v]\right)$ for an index~$v$.
By contrast, our novel deep sets approach uses \emph{all tuples} in the
persistence diagrams to compute embeddings.  \added{In practice, we did not 
find a significant advantage of using any of the functions defined above; we thus treat them
as a hyperparameter and refer to
\autoref{sec:Structured-Based Graph Classification Extended} for a detailed analysis.}

%%%%%%%%%%%%%%%%%%%%%%%%%%%%%%%%%%%%%%%%%%%%%%%%%%%%%%%%%%%%%%%%%%%%%%%%
\subsection{Differentiability \& Expressive Power}\label{sec:Expressive Power}
%%%%%%%%%%%%%%%%%%%%%%%%%%%%%%%%%%%%%%%%%%%%%%%%%%%%%%%%%%%%%%%%%%%%%%%%
%
The right choice of~$\Psi$ will lead to a differentiable downstream
representation. The map $\ph(\cdot)$ was shown to be differentiable~\citep{Gameiro16,
Hofer20, Moor20, Poulenard18}, provided the filtration satisfies injectivity at
the vertices.
We have the following theorem, whose proof is due to \citet[Lemma~1]{Hofer20}.
\begin{theorem}
  Let~\added{$f_\theta$} be a vertex filtration function
  \added{$f_\theta\colon V \to \reals$} with
  continuous parameters~$\theta$, and let $\Psi$ be a differentiable
  embedding function of unspecified dimensions.
  If the vertex function values of~\added{$f_\theta$} are distinct for a specific
  set of parameters~$\theta'$, i.e.\ \added{$f_\theta(v) \neq f_\theta(w)$} for
  $v \neq w$, then the map $\theta \mapsto \Psi\left(\ph(G, f_\theta)\right)$ is differentiable
  at $\theta'$.
  \label{thm:Differentiability}
\end{theorem}
This theorem is the basis for \method, as it states that the
filtrations, and thus the resulting `views' on a graph~$G$, can be
trained end-to-end. While \citet{Hofer20} describe this for
a \emph{single} filtration, their proof can be directly extended to
multiple filtrations as used in our approach.

The expressive power of graph neural networks is
well-studied~\citep{xu2018how, chen2020graph}
and typically assessed via the iterative Weisfeiler--Lehman
label refinement scheme, denoted as WL[$1$]. Given a graph with an initial set of
vertex labels, WL[$1$] collects the labels of neighbouring vertices for
each vertex in a multiset and `hashes' them into a new label, using
a perfect hashing scheme so that vertices/neighbourhoods with the same
labels are hashed to the same value. This
procedure is repeated and stops either when a maximum number of
iterations has been reached or no more label updates happen. The result
of each iteration~$h$ of the algorithm for a graph~$G$ is a feature
vector $\wl{h}_G$ that contains individual label counts. 
Originally conceived as a test for graph isomorphism~\citep{Weisfeiler68},
WL[$1$] has been successfully used for graph classification~\citep{Shervashidze09b}.
The test runs in polynomial time, but is known to fail to distinguish
between certain graphs, i.e.\ there are non-isomorphic graphs $G$ and $G'$
that obtain the same labelling by WL[$1$]~\cite{Fuerer17}.
Surprisingly, \citet{xu2018how} showed that standard graph neural networks based on
message passing are \emph{no more powerful} than WL[$1$].
Higher-order refinement schemes, which pass information over tuples of
nodes, for instance, can be defined~\citep{Maron19, Morris19}.
Some of these variants are strictly more powerful~(they can distinguish
between more classes of graphs) than WL[$1$] but also computationally more
expensive.

To prove the expressivity of our method, we will show that
\begin{inparaenum}[(i)]
  \item \method can distinguish all the graphs WL[$1$] can distinguish, and
  \item that there are graphs that WL[$1$] cannot distinguish but
    \method can.
\end{inparaenum}
The higher expressivity of \method does not necessarily imply that our
approach will perform generally better. In fact, WL[$1$] and, by
extension, GNNs, are capable of identifying \emph{almost all} non-isomorphic graphs~\citep{Babai80}.
However, the difference in expressivity implies that \method can
capture features that cannot be captured by GNNs, which can improve
predictive performance if those features cannot be obtained otherwise.
Since persistent homology is an isomorphism invariant, we first show
that we distinguish the same graphs that WL[$1$] distinguishes. We
do this by showing the existence of an injective filtration
function~$f$\footnote{
  \added{We drop $\theta$ in this notation since we do not require~$f$
    to have a set of continuous parameters here; the subsequent theorem
    is thus covering a more general case than
    \autoref{thm:Differentiability}. 
  }
}, thus ensuring differentiability according to
\autoref{thm:Differentiability}. 
\begin{theorem}
  Persistent homology is \emph{at least} as expressive as WL[$1$], i.e.\
  if the WL[$1$] label sequences for two graphs $G$ and $G'$ diverge,
  there exists an injective filtration~$f$ such that the corresponding
  $0$-dimensional persistence diagrams $\diagram_0$ and $\diagram'_0$
  are not equal.
  \label{thm:Expressivity}
\end{theorem}
\begin{proof}[Proof sketch]
  We first assume the existence of a sequence
  of WL[$1$] labels and show how to construct a filtration function~$f$
  from this.  While~$f$ will result in persistence diagrams
  that are different, thus serving to distinguish $G$ and $G'$, it
  does not necessarily satisfy injectivity. We therefore show that
  there is an injective function $\tilde{f}$ that is arbitrarily
  close to~$f$ and whose corresponding persistence diagrams
  $\widetilde{\diagram_0}$, $\widetilde{\diagram_0'}$ do \emph{not}
  coincide.
  Please refer to \autoref{sec:Proofs} for the detailed version of the
  proof.
\end{proof}
The preceding theorem proves the \emph{existence} of such a filtration function. 
Due to the capability of GNNs to approximate the Weisfeiler--Lehman test~\citep{xu2018how}
and the link between graph isomorphism testing and universal function approximation
capabilities~\citep[Theorem~4]{Chen19}, we can deduce that they are also able
to approximate~$f$ and~$\tilde{f}$, respectively. Yet, this does \emph{not} mean that we always end up learning~$f$ or~$\tilde{f}$.
%This architecture is known to be able to approximate
%\emph{any} function that discriminates between two
%graphs~\citep[Theorem~3]{Chen19}, hence it will also be able to
%approximate~$f$ and~$\tilde{f}$, respectively.
%
This result merely demonstrates that our layer can theoretically perform
\emph{at least as well as} WL[$1$] \added{when it comes to
  distinguishing non-isomorphic graphs; this does not generally
translate into better predictive performance, though}. 
In practice, \method may learn other filtration functions; injective filtrations
based on WL[$1$] are not necessarily optimal for a specific task~(we depict some of
the learnt filtrations in \autoref{sec:Filtrations}).

To prove that our layer is more expressive than a GCN, we show that
there are pairs of graphs $G, G'$ that cannot be distinguished
by WL[$1$] but that \emph{can} be distinguished by~$\ph(\cdot)$ and by
\method{}, respectively: let $G$ be a graph consisting of the disjoint
union of two triangles, i.e.\ 
\begin{tikzpicture}[baseline]
    \begin{scope}[scale=0.25]
      \coordinate (A) at (0.0, 0.5); 
      \coordinate (B) at (0.5, 0.0); 
      \coordinate (C) at (0.5, 1.0); 

      \coordinate (D) at (1.0, 1.0); 
      \coordinate (E) at (1.5, 0.5); 
      \coordinate (F) at (1.0, 0.0); 
    \end{scope}

    \foreach \c in {A,B,...,F}
    {
      \filldraw[black] (\c) circle (1pt);
    }

    \draw (A) -- (B) -- (C) -- cycle;
    \draw (D) -- (E) -- (F) -- cycle;
\end{tikzpicture},
and let $G'$ be a graph consisting of a hexagon, i.e.\
\begin{tikzpicture}[baseline]
  \begin{scope}[scale=0.25]
    \coordinate (A) at (0.0, 0.5); 
    \coordinate (B) at (0.5, 0.0); 
    \coordinate (C) at (0.5, 1.0); 

    \coordinate (D) at (1.0, 1.0); 
    \coordinate (E) at (1.5, 0.5); 
    \coordinate (F) at (1.0, 0.0); 
  \end{scope}

  \foreach \c in {A,B,...,F}
  {
    \filldraw[black] (\c) circle (1pt);
  }
  \draw (A) -- (C) -- (D) -- (E) -- (F) -- (B) -- cycle;
\end{tikzpicture}.
WL[$1$] will be unable to distinguish
these two graphs because all multisets in every iteration will be the
same.
Persistent homology, by contrast, can distinguish $G$ from $G'$ using
their Betti numbers. We have $\betti{0}(G) = \betti{1}(G) = 2$,
because~$G$ consists of two connected components and two cycles, whereas
$\betti{0}(G') =  \betti{1}(G') = 1$ as~$G'$ only consists of one
connected component and one cycle.
The characteristics captured by persistent homology are therefore
different from the ones captured by WL[$1$].
Together with \autoref{thm:Expressivity}, this
example implies that persistent homology is \emph{strictly} more
powerful than WL[$1$]~(see \autoref{sec:Regular Graphs} for
an extended expressivity analysis using higher-order
topological features). The expressive power of \method hinges on the
expressive power of the filtration---making it crucial that we can learn
it.

%%%%%%%%%%%%%%%%%%%%%%%%%%%%%%%%%%%%%%%%%%%%%%%%%%%%%%%%%%%%%%%%%%%%%%%%
\section{Experiments}
%%%%%%%%%%%%%%%%%%%%%%%%%%%%%%%%%%%%%%%%%%%%%%%%%%%%%%%%%%%%%%%%%%%%%%%%

We showcase the empirical performance of \method on a set of synthetic
and real-world data sets, with a primary focus on assessing in which
scenarios topology can enhance and improve learning on graph. Next to
demonstrating improved predictive performance for synthetic and
structure-based data sets~(\autoref{sec:Synthetic Data Sets} and \autoref{sec:Structured-Based Graph Classification}), we also compare
\method with existing topology-based algorithms~(\autoref{sec:Other
Topological Methods}).

%%%%%%%%%%%%%%%%%%%%%%%%%%%%%%%%%%%%%%%%%%%%%%%%%%%%%%%%%%%%%%%%%%%%%%%%
\subsection{Experimental Setup}
\label{sec:experimentsetup}
%%%%%%%%%%%%%%%%%%%%%%%%%%%%%%%%%%%%%%%%%%%%%%%%%%%%%%%%%%%%%%%%%%%%%%%%

Following the setup of \citet{dwivedi2020benchmarking}, we ran all
experiments according to a consistent training setup and a limited
parameter budget to encourage comparability between architectures. 
For further details, please refer to \autoref{sec:Training}.
In all tables and graphs, we report the mean test accuracy along with
the standard deviation computed over the different folds.
All
experiments were tracked~\citep{wandb}; experimental logs, reports, and
code will be made publicly available.

\textbf{Baselines and comparison partners.}
We compare our method to several GNN architectures from
the literature, namely
\begin{inparaenum}[(i)]
\item Graph Convolutional Networks~\citep[GCN]{Kipf17},
\item Graph Attention Networks~\citep[GAT]{velickovic2018graph},
\item Gated-GCN~\citep{bresson2017residual},
\item Graph Isomorphism Networks~\citep[GIN]{xu2018how},
\item the Weisfeiler--Lehman kernel~\citep[WL]{Shervashidze09b}, and
\item \mbox{WL-OA}~\citep{Kriege16}.
\end{inparaenum}
The performance of these methods has been assessed in benchmarking
papers~\citep{dwivedi2020benchmarking,Borgwardt20,morris2020tudataset},
whose experimental conditions are comparable.
We use the \emph{same} folds and hyperparameters as in the
corresponding benchmarking papers to ensure a fair comparison.

\textbf{\method setup.}
We add our layer to existing GNN architectures, replacing the
second layer by \method, respectively.\footnote{We investigate the
implications of different layer placements in \autoref{sec:Layer
placement} and find that the best placement depends on the data set;
placing the layer second is a compromise choice.}
For instance, \mbox{GCN-4} refers to a GCN with four layers, while 
\mbox{GCN-3-\method-1} refers to method that has been made
`topology-aware' using \method.
This setup ensures that both the original and the modified architecture
have approximately the same number of parameters.
We treat the choice of the embedding function~$\Psi$ as a hyperparameter in our
training for all subsequent experiments. \autoref{sec:Structured-Based Graph Classification Extended}
provides a comprehensive assessment of the
difference between the \texttt{DeepSets} approach~(which is capable of
capturing \emph{interactions} between tuples in a persistence diagram)
and decomposed embedding functions, which do not account for interactions.

%%%%%%%%%%%%%%%%%%%%%%%%%%%%%%%%%%%%%%%%%%%%%%%%%%%%%%%%%%%%%%%%%%%%%%%%
\subsection{Performance on Synthetic Data Sets}\label{sec:Synthetic Data Sets}
%%%%%%%%%%%%%%%%%%%%%%%%%%%%%%%%%%%%%%%%%%%%%%%%%%%%%%%%%%%%%%%%%%%%%%%%

As an illustrative example depicted in \autoref{fig:Synthetic data performance},
we use two synthetic balanced \mbox{2-class} data sets of $1000$ graphs
each. In the \data{Cycles} data set~(\autoref{sfig:Results cycles}, right), we
generate either one large cycle~(class~$0$) or multiple small
ones~(class~$1$). These graphs can be easily distinguished by \emph{any}
topological approach because they differ in the number of connected
components and the number of cycles. For the \data{Necklaces} data
set~(\autoref{sfig:Results necklaces}, right), we generate `chains' of vertices
with either two individual cycles, or a `merged' one.
Both classes have the same \emph{number} of cycles, but the incorporation of
additional connectivity information along their neighbourhoods makes them
distinguishable for \method, since appropriate filtrations for the data
can be learnt. \added{All synthetic data sets use node features consisting
of 3-dimensional vectors sampled from a Normal distribution. As PH and WL would
consider all instances of graphs being distinct and thus remove any potential signal,
we used the node degrees as features instead as commonly done~\citep{Borgwardt20}.} 
%
% TODO: do we have a study of the number of filtrations and the
% performance?

For this illustrative example, we integrate \method with a GCN.
We find that \method performs well even \emph{without} any GCN layers---thus
providing another empirical example of improved expressivity.
Moreover, we observe that standard GCNs
require at least four layers~(for \data{Cycles}) or more~(for \data{Necklaces})
to approach the performance of \method. WL[$1$], by contrast, fails to
classify \data{Cycles} and still exhibits a performance gap to the GCN
for \data{Necklaces}, thus showcasing the benefits of having access to
a learnable node representation.
\added{It also underscores the observations in \autoref{sec:Expressive Power}:
the higher expressive power of WL[$1$] as compared to a standard GCN
does not necessarily translate to higher predictive performance.}
\autoref{sec:Synthetic data sets extended} provides an extended
analysis with more configurations; we find that
\begin{inparaenum}[(i)]
  \item both cycles and connected components are crucial for reaching
    high predictive performance, and
  \item the static variant is performing slightly better than the
standard GCN, but \emph{significantly worse} than \method, due to its
very limited access to topological information. 
\end{inparaenum}
A simple static filtration~(based on node degrees), which we denote by
PH, only works for the \data{Cycles} data set~(this is a consequence of
the simpler structure of that data set, which can distinguished by Betti
number information already), whereas the more complex
structure of the \data{Necklaces} data necessitates learning
a task-based filtration.

%%%%%%%%%%%%%%%%%%%%%%%%%%%%%%%%%%%%%%%%%%%%%%%%%%%%%%%%%%%%%%%%%%%%%%%%
\subsection{Structure-based Graph and Node Classification Performance}
\label{sec:Structured-Based Graph Classification}
%%%%%%%%%%%%%%%%%%%%%%%%%%%%%%%%%%%%%%%%%%%%%%%%%%%%%%%%%%%%%%%%%%%%%%%%

A recent survey~\citep{Borgwardt20} showed that the node features of
benchmark graphs already carry substantial information, thus suppressing
the signal carried by the graph structure itself to some extent.
This motivated us to prepare a set of experiments in which we classify
graphs \emph{solely based on graph structure}. We achieve this by
replacing all node labels and node features by random ones, thus leaving only
structural/topological information for the classification task. In the case of
the \data{Pattern} dataset we use the original node features which are random
by construction~\citep{dwivedi2020benchmarking}.

\autoref{tab:Structural Results} depicts the results for graph and node classification tasks on such graphs; we observe
a \emph{clear advantage} of \method over its comparison partners:
making an existing GNN topology-aware via \method improves predictive
performance in virtually all instances.
In some cases, for instance when comparing \mbox{GCN-4} to
\mbox{GCN-3-\method-1} on \data{MNIST}, the gains are substantial with
an increase of more than 8\%.
This also transfers to other model architectures, where in most cases \method
improves the performance across datasets. Solely the \mbox{GIN} model on
\data{Proteins} and the \mbox{GAT} model on \data{Pattern} decrease in performance
when incorporating topological features.  The deterioration of \method on
\data{Patterns} is especially significant. Nevertheless, the inferior performance
is in line with the low performance of \mbox{GAT} in general compared to the other methods
we considered. In this context the lower performance of \method is not surprising as it relies
on the backbone model for the construction of a filtration.
This demonstrates the utility of \method in making additional structural
information available to improve classification performance.

%%%%%%%%%%%%%%%%%%%%%%%%%%%%%%%%%%%%%%%%%%%%%%%%%%%%%%%%%%%%%%%%%%%%%%%%
\begin{table}[tbp]
  \centering
  \caption{%
    Results for the structure-based experiments. We depict the test
    accuracy obtained on various benchmark data sets when only
    considering structural information~(i.e.\ the network has access to
    \emph{uninformative} node features). For MOLHIV, the ROC-AUC is reported. Graph classification results
    are shown on the left, while node classification results are shown on
    the right. We compare three architectures (\mbox{GCN-4},
    \mbox{GIN-4}, \mbox{GAT-4}) with corresponding models where one
    layer is replaced with \method and highlight the winner of each
    comparison in \textbf{bold}. 
  }
  \label{tab:Structural Results}
  \setlength{\tabcolsep}{3.0pt}
  \sisetup{
    detect-all              = true,
    table-format            = 2.1(2),
    separate-uncertainty    = true,
    mode                    = text,
    reset-text-shape        = false,
    table-text-alignment    = center,
    retain-zero-uncertainty = true,
    detect-weight,
  }
  \small
  \robustify\bfseries
  \robustify\itshape
  \renewrobustcmd{\bfseries}{\fontseries{b}\selectfont}
  \renewrobustcmd{\boldmath}{}
  \let\b\bfseries
  \let\i\itshape
  \setlength{\tabcolsep}{3.0pt}
  \begin{tabular}{@{}lSSSSS}
    \toprule
                            & \multicolumn{5}{c}{\scriptsize\itshape Graph classification}\\
    \midrule
    \textsc{Method}         & \multicolumn{1}{c}{\scriptsize\data{DD}}
                            & \multicolumn{1}{c}{\scriptsize\data{ENZYMES}}
                            & \multicolumn{1}{c}{\scriptsize\data{MNIST}}
                            & \multicolumn{1}{c}{\scriptsize\data{PROTEINS}} 
                            & \multicolumn{1}{c}{\scriptsize\data{MOLHIV}}\\
    \midrule
    GCN-4           &   68.0 \pm 3.6 &   22.0 \pm 3.3 &   76.2 \pm  0.5 &   68.8 \pm 2.8 & 66.4 \pm 1.8\\
    GCN-3-\method-1 &\b 75.1 \pm 2.1 & \b30.3 \pm 6.5 & \b84.8 \pm  0.4 & \b73.8 \pm 4.3 & \b69.4 \pm 1.8 \\
    \midrule
    GIN-4           &   75.6 \pm 2.8 &   21.3 \pm 6.5 &   83.4 \pm  0.9 & \b74.6 \pm 3.1 & \b68.7 \pm 0.9 \\
    GIN-3-\method-1 &\b 76.2 \pm 2.4 & \b23.7 \pm 6.9 & \b84.4 \pm  1.1 &   73.9 \pm 4.9 &   65.1 \pm 6.2 \\
    \midrule
    GAT-4           &   63.3 \pm 3.7 &   21.7 \pm 2.9 &   63.2 \pm 10.4 &   67.5 \pm 2.6 & 51.8 \pm 5.6 \\
    GAT-3-\method-1 &\b 75.7 \pm 2.1 &\b 23.5 \pm 6.1 &\b 77.2 \pm 10.5 &\b 72.4 \pm 4.6 & \b68.6 \pm 1.7\\
    \bottomrule
  \end{tabular}
  \begin{tabular}{@{}S}
    \toprule
    \multicolumn{1}{c}{\scriptsize\itshape Node classification}\\
    \midrule
    \multicolumn{1}{c}{\scriptsize\data{Pattern}}\\
    \midrule
        85.5 \pm 0.4  \\
     \b 86.6 \pm 0.1  \\
    \midrule
        84.8 \pm 0.0  \\
     \b 86.7 \pm 0.1  \\
    \midrule
     \b 73.1 \pm 1.9  \\
        59.6 \pm 3.3  \\
    \bottomrule
  \end{tabular}
\end{table}
%%%%%%%%%%%%%%%%%%%%%%%%%%%%%%%%%%%%%%%%%%%%%%%%%%%%%%%%%%%%%%%%%%%%%%%%

\begingroup
\setlength{\intextsep}{0.25\baselineskip}
%
%%%%%%%%%%%%%%%%%%%%%%%%%%%%%%%%%%%%%%%%%%%%%%%%%%%%%%%%%%%%%%%%%%%%%%%%
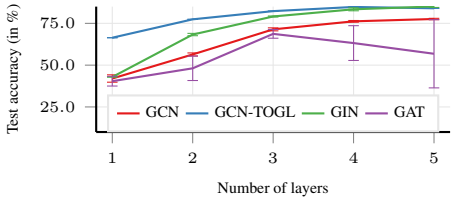
\begin{wrapfigure}{r}{0.45\linewidth}
  \centering
  % TODO: we might want to rename TOGL here to GCN-TOGL?
  \begin{tikzpicture}
    \begin{axis}[%
      perfplot,
      error bars/y dir      = both,
      error bars/y explicit = true,
      height                = 3.25cm,
      width                 = \linewidth,
      xmax                  = 5,
      xmin                  = 1,
      ymin                  = 10.0,
      ymax                  = 85.0,
      enlarge x limits      = 0.05,
      xlabel                = {Number of layers},
      legend style          = {%
        at     = {(0.50,0.30)},
        anchor = north,
        font   = \fontsize{6}{7}\selectfont,
      },
      /tikz/font = {\tiny},
    ]
      \addplot+[%
        discard if not = {data}{MNIST}, ] table[x = depth, y expr =\thisrow{gcn_acc}, y error expr = \thisrow{gcn_sdev}, col sep = comma] {Data/Structural_results_complete.csv};

      \addlegendentry{GCN}

      \addplot+[%
        discard if not = {data}{MNIST}, ] table[x = depth, y expr =\thisrow{topognn_acc}, y error expr = \thisrow{topognn_sdev}, col sep = comma] {Data/Structural_results_complete.csv};

      \addlegendentry{GCN-\method}

      \addplot+[%
        discard if not = {data}{MNIST}, ] table[x = depth, y expr =\thisrow{gin_acc}, y error expr = \thisrow{gin_sdev}, col sep = comma] {Data/Structural_results_complete.csv};

      \addlegendentry{GIN}

      \addplot+[%
        discard if not = {data}{MNIST}, ] table[x = depth, y expr =\thisrow{gat_acc}, y error expr = \thisrow{gat_sdev}, col sep = comma] {Data/Structural_results_complete.csv};

      \addlegendentry{GAT}
    \end{axis}
  \end{tikzpicture}
  % I'm not proud of what I did. Needs must as the devil drives.
  \vspace{-0.25cm}
  \caption{%
    Classification performance when analysing the structural variant of \data{MNIST}.
  }
  \label{fig:MNIST structural results}
\end{wrapfigure}
%%%%%%%%%%%%%%%%%%%%%%%%%%%%%%%%%%%%%%%%%%%%%%%%%%%%%%%%%%%%%%%%%%%%%%%%
%
Varying the number of layers in this experiment underscores the benefits
of including \method in a standard GCN architecture. Specifically,
an architecture that incorporates \method reaches high predictive
performance earlier, i.e.\ with fewer layers, than other methods, thus reducing the risk of oversmoothing~\citep{Chen20}.
\autoref{fig:MNIST structural results} depicts this for \data{MNIST}; we
observe similar effects on the other data sets.
\autoref{sec:Structured-Based Graph Classification Extended}
presents a detailed performance comparison ,and extended analysis, and a discussion of the effects of different choices for an
embedding function~$\Psi$.

\endgroup

%%%%%%%%%%%%%%%%%%%%%%%%%%%%%%%%%%%%%%%%%%%%%%%%%%%%%%%%%%%%%%%%%%%%%%%%
\subsection{Performance on Benchmark Data Sets}\label{sec:Benchmark Data Sets}
%%%%%%%%%%%%%%%%%%%%%%%%%%%%%%%%%%%%%%%%%%%%%%%%%%%%%%%%%%%%%%%%%%%%%%%%

%%%%%%%%%%%%%%%%%%%%%%%%%%%%%%%%%%%%%%%%%%%%%%%%%%%%%%%%%%%%%%%%%%%%%%%%
\begin{table}[tbp]                        
  \caption{%
    Test accuracy on benchmark data sets~(following standard
    practice, we report weighted accuracy on \data{CLUSTER} and
    \data{PATTERN}). Methods printed in
    black have been run in our setup, while methods printed in
    \textcolor{gray}{grey} are cited from the literature, i.e.\ \citet{dwivedi2020benchmarking}, \citet{morris2020tudataset} for \data{IMDB-B} and
    \data{REDDIT-B}, and \citet{Borgwardt20} for WL/\mbox{WL-OA}
    results.
    %
    %GIN-4 results printed in \emph{italics} are
    %1-layer GIN-$\epsilon$, as reported in \citet{morris2020tudataset}.
    %
    Graph classification results
    are shown on the left; node classification results are shown on
    the right.
    Following \autoref{tab:Structural Results}, we take existing
    architectures and replace their second layer by \method; we
    use \emph{italics} to denote the winner of each comparison.
    A \textbf{bold} value indicates the overall winner of a column,
    i.e.\ a data set.
  }
  \label{tab:Benchmark results}
  \centering
  \newcommand{\gr}{\rowfont{\color{gray}}}
  \newcommand{\NA}{---}
  \sisetup{
    detect-all           = true,
    table-format         = 2.1(2),
    separate-uncertainty = true,
    mode                 = text,
    reset-text-shape     = false,
    table-text-alignment = center,
    detect-weight,
  }
  \robustify\bfseries
  \robustify\itshape
  \renewrobustcmd{\bfseries}{\fontseries{b}\selectfont}
  \renewrobustcmd{\boldmath}{}
  \let\b\bfseries
  \let\i\itshape
  \setlength{\tabcolsep}{3.0pt}
  \resizebox{\linewidth}{!}{
  \begin{tabu}{lSSSSSSS}
    \toprule
                    & \multicolumn{7}{c}{\scriptsize\itshape Graph classification}\\
    \midrule
    \textsc{Method} & {\scriptsize\data{CIFAR-10}}
                    & {\scriptsize\data{DD}}
                    & {\scriptsize\data{ENZYMES}}
                    & {\scriptsize\data{MNIST}}
                    & {\scriptsize\data{PROTEINS-full}}
                    & {\scriptsize\data{IMDB-B}}
                    & {\scriptsize\data{REDDIT-B}}\\
    \midrule
    \gr GATED-GCN-4         & \b67.3 \pm 0.3 & 72.9\pm2.1 & 65.7\pm4.9
                            & \b97.3 \pm 0.1 & \b76.4\pm2.9 &    {\NA}
                            &   {\NA}\\
    \gr WL                  &      {\NA}       &   77.7\pm2.0 & 54.3\pm0.9 & {\NA}            & 73.1\pm0.5 & 71.2\pm 0.5           & 78.0\pm0.6\\
    \gr WL-OA               &      {\NA}       & \b77.8\pm1.2 & 58.9\pm0.9 & {\NA}            & 73.5\pm0.9 & 74.0\pm 0.7           & 87.6\pm0.3\\
    \midrule
    GCN-4 & 54.2 \pm 1.5 & 72.8\pm4.1 & \b65.8\pm4.6 & 90.0 \pm 0.3 & \i76.1\pm2.4 & 68.6\pm 4.9           & \b92.8\pm1.7\\
    GCN-3-\method-1& \i61.7 \pm 1.0 & \i73.2\pm4.7 & 53.0\pm9.2 & 95.5 \pm 0.2 & 76.0\pm3.9 & \i72.0\pm 2.3           & 89.4\pm2.2\\
    \midrule
    GIN-4 & 54.8 \pm 1.4 & 70.8 \pm 3.8 & \i50.0\pm12.3 & \i96.1 \pm 0.3 & 72.3 \pm 3.3 & 72.8\pm 2.5           & 81.7 \pm 6.9\\
    GIN-3-\method-1& \i61.3 \pm 0.4 & \i75.2\pm4.2 & 43.8 \pm 7.9  & 96.1\pm 0.1 & \i73.6 \pm 4.8 & \b74.2\pm 4.2  & 89.7\pm 2.5\\
    \midrule
    GAT-4 & 57.4 \pm 0.6 & 71.1 \pm 3.1 & 26.8 \pm 4.1  & 94.1 \pm 0.3 & 71.3 \pm 5.4 & \i73.2 \pm 4.1  & 44.2 \pm 6.6\\
    GAT-3-\method-1 & \i63.9 \pm 1.2 & \i73.7 \pm 2.9 & \i51.5 \pm 7.3  & \i95.9 \pm 0.3 & \i75.2 \pm 3.9 & 70.8 \pm 8.0  & \i89.5 \pm 8.7\\
    \bottomrule
  \end{tabu}
  \begin{tabu}{S}
    \toprule
    \scriptsize\itshape Node classification\\
    \midrule
    \scriptsize\data{CLUSTER}\\
    \midrule
    \gr\b 60.4 \pm 0.4        \\
    \gr {\NA}                   \\ 
    \gr {\NA}                   \\ 
    \midrule
        57.0 \pm 0.9          \\
        \i60.4 \pm 0.2        \\
    \midrule
          58.5 \pm 0.1        \\
        \i60.4 \pm 0.2        \\
    \midrule
          56.6 \pm 0.4        \\
        \i58.4 \pm 3.7        \\
    \bottomrule
  \end{tabu}
  }
\end{table}
%%%%%%%%%%%%%%%%%%%%%%%%%%%%%%%%%%%%%%%%%%%%%%%%%%%%%%%%%%%%%%%%%%%%%%%%

Having ascertained the utility of \method for classifying data
sets with topological information, we now analyse the effects of \method
in standard graph and node classification tasks. \autoref{tab:Benchmark results} depicts
the results on well-known benchmark data sets for graph and node classification.
% TODO: Rewrite when GAT is done.
% TODO: Commment on high variance of GAT, is in line with Benchmarking GNNs.
We see that \method performs better than its comparison partners~(i.e.\ \mbox{GCN-4}, \mbox{GIN-4}, \mbox{GAT-4}) on most of the data sets, showcasing the benefits of substituting a layer of GNN with \method.
Concerning \data{ENZYMES} performance, we experienced a severe degree of
overfitting during training. This was exacerbated by the fact that
\data{ENZYMES} is the smallest of the compared data sets and we eschewed
the tuning of regularisation hyperparameters such as `dropout' for the
sake of being comparable with the benchmark
results~\citep{dwivedi2020benchmarking}. With the GAT-based GNN, we
generally observe an high degree of variance, as already reported in
previous studies \citep{dwivedi2020benchmarking}. In particular, we
experienced issues in training it on the \data{REDDIT-B} dataset. The
addition of \method seems to address this issue, which we consider to
underline the overall potential of topological features~(at least for
data sets with distinct topological information).

%%%%%%%%%%%%%%%%%%%%%%%%%%%%%%%%%%%%%%%%%%%%%%%%%%%%%%%%%%%%%%%%%%%%%%%%
\begin{table}[tbp]
  \caption{%
    Test accuracy when comparing \method~(integrated into a simplified
    architecture) with existing topology-based embedding functions or
    \texttt{readout} functions. Results shown in \textcolor{gray}{grey}
    are cited from the respective papers~\citep{Carriere20a, Hofer20}. For GFL, we cite degree-based results so that its performance values pertain to the same scenario.
  }
  \label{tab:Other Topological Methods}
  \centering
  \small
  \newcommand{\gr}{\rowfont{\color{gray}}}
  \newcommand{\NA}{---}
  \sisetup{
    detect-all           = true,
    table-format         = 2.1(2),
    separate-uncertainty = true,
    mode                 = text,
    table-text-alignment = center,
    detect-weight,
  }
  \robustify\bfseries
  \renewrobustcmd{\bfseries}{\fontseries{b}\selectfont}
  \renewrobustcmd{\boldmath}{}
  \let\b\bfseries
  \setlength{\tabcolsep}{3.0pt}
  \begin{tabu}{lSSSSS}
    \toprule
    \textsc{Method}          & {\scriptsize\data{REDDIT-5K}}
                             & {\scriptsize\data{IMDB-MULTI}}
                             & {\scriptsize\data{NCI1}}
                             & {\scriptsize\data{REDDIT-B}}
                             & {\scriptsize\data{IMDB-B}}\\
    \midrule
    \gr GFL                  &  55.7 \pm 2.1 & 49.7 \pm 2.9 & 71.2 \pm 2.1 &  90.2 \pm 2.8 &\b74.5 \pm 4.6\\
    \gr PersLay              &  55.6         & 48.8         & 73.5         &  {---}        &  71.2        \\
    \midrule
    GCN-1-\method-1          &\b56.1 \pm 1.8 &\b52.0 \pm 4.0&\b75.8 \pm 1.8&  90.1 \pm 0.8 &  74.3 \pm 3.6\\
    GCN-1-\method-1 (static) &  55.5 \pm 1.8 &  48.3 \pm 4.9&  75.1 \pm 1.2&\b90.4 \pm 1.4 &  72.2 \pm 2.1\\
    \bottomrule
  \end{tabu}
\end{table}
%%%%%%%%%%%%%%%%%%%%%%%%%%%%%%%%%%%%%%%%%%%%%%%%%%%%%%%%%%%%%%%%%%%%%%%%

%%%%%%%%%%%%%%%%%%%%%%%%%%%%%%%%%%%%%%%%%%%%%%%%%%%%%%%%%%%%%%%%%%%%%%%%
\subsection{Comparison to Other Topology-Based Algorithms}\label{sec:Other Topological Methods}
%%%%%%%%%%%%%%%%%%%%%%%%%%%%%%%%%%%%%%%%%%%%%%%%%%%%%%%%%%%%%%%%%%%%%%%%

In light of \method containing existing embedding functions~$\Psi$~\citep{Carriere20a}, we compare its performance
to other topology-based algorithms that have been used for graph classification. \autoref{tab:Other Topological
Methods} summarises the performance~(for comparison purposes, we also
show the results for a `static' variant of our layer; see
\autoref{sec:Training} for more details).
In order to permit a fair
comparison, we integrate \method into a simpler GNN architecture,
consisting of a single GCN message passing layer.

We observe that our filtration learning approach outperforms the fixed
filtrations used by \citet{Carriere20a}, highlighting the utility
of a layer specifically designed to be integrated into
GNNs. \method also fares well in comparison to the \texttt{readout}
function by \citet{Hofer20}. While large standard deviations~(due to the
small size of the data sets) preclude an assessment of significant
differences, we hypothesise that the benefits of having access to
\emph{multiple filtrations} will be more pronounced for large data sets
containing more pronounced topological features, such as molecular
graphs.
Notice that in contrast to \citet{Hofer20}, our layer can be included at
different stages of the overal GNN architecture. We investigate the
impact of different positions in Appendix \ref{sec:Layer placement} and
show that different positions can lead significantly different
performance. These results further motivate the importance of a flexible
topological \emph{layer} as opposed to \texttt{readout} functions or
static filtrations.

%%%%%%%%%%%%%%%%%%%%%%%%%%%%%%%%%%%%%%%%%%%%%%%%%%%%%%%%%%%%%%%%%%%%%%%%
\section{Conclusion}
%%%%%%%%%%%%%%%%%%%%%%%%%%%%%%%%%%%%%%%%%%%%%%%%%%%%%%%%%%%%%%%%%%%%%%%%

We presented \method, a generically-applicable layer that incorporates
topological information into any GNN architecture. 
We proved that \method, due to its filtration functions~(i.e.\ input
functions) being learnable, is more expressive than WL[$1$], the
Weisfeiler--Lehman test for graph isomorphism, and therefore also
increases expressivity of GNNs.
On data sets with pronounced topological structures, we found that our
method helps GNN\added{s} obtain substantial gains in predictive performance. 
We also saw that the choice of function for embedding topological
descriptors is crucial, with embedding functions that can handle
interactions between individual topological features typically
performing better than those that cannot.
%
% TODO: might have to adjust this?
On benchmark data sets, we observed that our topology-aware approach
can help improve predictive performance while overall exhibiting
favourable performance in comparison to the literature.
%
%Interestingly, \emph{both} our method and its static ablated variant
%perform similarly when classification performance does not appear to
%rely on topological features.
%
We \added{also} observed that topological information may sometimes lead
to overfitting issues on smaller data sets, and leave the investigation
of additional regularisation strategies for our method for future work.
Furthermore, we hypothesise that the use of different
filtration types~\citep{Milosavljevic11}, together with advanced persistent homology
algorithms, such as extended persistence~\citep{Cohen-Steiner09}, will prove beneficial for predictive performance.

% Encourages LaTeX to actually put the tables *before* this page.
\clearpage

\section*{Reproducibility Statement}

We have provided the code for our experiments along with the seeds used for training.
All experiments were run on single GPUs, which avoids additional sources of randomness.
Details are provided in~\autoref{sec:experimentsetup}.
The parameters and performance metrics during training were tracked and
will be made \added{publicly} available~\citep{wandb}. Our code is
released under a BSD-3-Clause License and can be accessed under
\url{https://github.com/BorgwardtLab/TOGL}.

%%%%%%%%%%%%%%%%%%%%%%%%%%%%%%%%%%%%%%%%%%%%%%%%%%%%%%%%%%%%%%%%%%%%%%%%
\subsection*{Acknowledgements}
%%%%%%%%%%%%%%%%%%%%%%%%%%%%%%%%%%%%%%%%%%%%%%%%%%%%%%%%%%%%%%%%%%%%%%%%

The authors would like to thank ``Weights and Biases, Inc.'' for providing us with a free
academic team account.
This project was supported by the grant \#2017-110 of the
Strategic Focal Area ``Personalized Health and Related Technologies
(PHRT)'' of the ETH Domain for the SPHN/PHRT
Driver Project ``Personalized Swiss Sepsis Study'' and the
Alfried Krupp Prize for Young University Teachers of
the Alfried Krupp von Bohlen und Halbach-Stiftung~(K.B.)
Edward De Brouwer gratefully acknowledges support by an FWO-SB grant and support from NVIDIA for GPUs.
Yves Moreau is funded by (1) Research Council KU Leuven: C14/18/092
SymBioSys3; CELSA-HIDUCTION, (2) Innovative Medicines Initiative:
MELLODY, (3) Flemish Government (ELIXIR Belgium, IWT, FWO 06260) and (4)
Impulsfonds AI: VR 2019 2203 DOC.0318/1QUATER Kenniscentrum Data en
Maatschappij.  Some computational resources and services used in this
work were provided by the VSC (Flemish Supercomputer Center), funded by
the Research Foundation - Flanders (FWO) and the Flemish Government
– department EWI.

%%%%%%%%%%%%%%%%%%%%%%%%%%%%%%%%%%%%%%%%%%%%%%%%%%%%%%%%%%%%
% References
%%%%%%%%%%%%%%%%%%%%%%%%%%%%%%%%%%%%%%%%%%%%%%%%%%%%%%%%%%%%
{
  \bibliographystyle{abbrvnat}
  \bibliography{main}
}
%%%%%%%%%%%%%%%%%%%%%%%%%%%%%%%%%%%%%%%%%%%%%%%%%%%%%%%%%%%%

% Makes it easier to snip away the supplements later on. This will
% disable hyperlinks, though.
\clearpage

\appendix

% Prefix for figures and tables in the supplementary materials
\makeatletter
\renewcommand{\thefigure}{S\@arabic\c@figure}
\renewcommand{\thetable}{S\@arabic\c@table}
\makeatletter

%%%%%%%%%%%%%%%%%%%%%%%%%%%%%%%%%%%%%%%%%%%%%%%%%%%%%%%%%%%%%%%%%%%%%%%%
\section{Topological Data Analysis}
\label{sec:TDA}
%%%%%%%%%%%%%%%%%%%%%%%%%%%%%%%%%%%%%%%%%%%%%%%%%%%%%%%%%%%%%%%%%%%%%%%%

We provide a more formal introduction to persistent homology, the
technique on which \method is fundamentally based. Persistent homology
arose as one of the flagship approaches in the field of computational
topology, which aims to make methods from this highly abstract branch of
mathematics available for data analysis purposes.

To gain a better understanding, we will briefly take a panoramic tour
through algebraic topology, starting from \emph{simplicial
homology}, an algebraic technique for `calculating' the connectivity of
topological spaces, represented in the form of simplicial complexes,
i.e.\ generalised graphs. Simplicial homology is said to assess the
connectivity of a topological space by `counting its high-dimensional
holes'. We will see how to make this description more precise.

%%%%%%%%%%%%%%%%%%%%%%%%%%%%%%%%%%%%%%%%%%%%%%%%%%%%%%%%%%%%%%%%%%%%%%%%
\subsection{Simplicial Homology}
%%%%%%%%%%%%%%%%%%%%%%%%%%%%%%%%%%%%%%%%%%%%%%%%%%%%%%%%%%%%%%%%%%%%%%%%
%
Simplicial complexes are the central concept in algebraic topology.
A simplicial complex~$\simplicialcomplex$ consists of a set of
\emph{simplices} of certain dimensions, such as
vertices~(dimension~$0$), edges~(dimension~$1$), and
triangles~(dimension~$2$).
Each simplex~$\sigma \in \simplicialcomplex$ has a set of
faces, and each face~$\tau$ has to satisfy $\tau \in
\simplicialcomplex$. Moreover, if $\sigma \cap \sigma' \neq \emptyset$
for $\sigma, \sigma' \in \simplicialcomplex$, then $\sigma \cap \sigma'
\in \simplicialcomplex$. Thus, $\simplicialcomplex$ is `closed under
calculating the faces of a simplex'. A graph can be seen as
a low-dimensional simplicial complex that only contains \mbox{$0$-simplices}~(vertices) and
\mbox{$1$-simplices}~(edges), so everything we say applies,
\emph{mutatis mutandis}, also to graphs.

\paragraph{Chain groups.}
For a simplicial complex $\simplicialcomplex$, we denote by
$\chaingroup{d}(\simplicialcomplex)$ the vector space generated over
$\mathds{Z}_2$~(the field with two elements). The elements of
$\chaingroup{d}(\simplicialcomplex)$ are the $d$-simplices in
$\simplicialcomplex$, or rather their \emph{formal sums} with
coefficients in $\mathds{Z}_2$. For example, $\sigma + \tau$ is an
element of the chain group, also called a \emph{simplicial chain}.
Addition is well-defined and easy to implement as an algorithm since
a simplex can only be present or absent over~$\mathds{Z}_2$
coefficients.\footnote{Different choices of coefficient fields would be
possible, but are rarely used for data analysis purposes.}
The use of chain groups lies in providing the underlying vector space
to formalise boundary calculations over a simplicial complex. The
boundary calculations, in turn, are necessary to quantify the
connectivity! 

\paragraph{Boundary homomorphism.}
Given a \mbox{$d$-simplex} $\sigma = (v_0,\dots,v_d) \in
\simplicialcomplex$, we can formalise its face or `boundary' calculation
by defining the boundary operator
$\boundary{d}\colon\chaingroup{d}(\simplicialcomplex)\to\chaingroup{d-1}(\simplicialcomplex)$
as a sum of the form
\begin{equation}
  \boundary{d}(\sigma) := \sum_{i=0}^{d}(v_0,\dots, v_{i-1},v_{i+1},\dots, v_d),
\end{equation}
i.e.\ we leave out every vertex~$v_i$ of the simplex once. Since only
sums are involved, this operator is seen to be a homomorphism between
the chain groups; the calculation extends to
$\chaingroup{d}(\simplicialcomplex)$ by linearity.
The boundary homomorphism gives us a way to precisely define what we
understand by connectivity. To this end, note that its \emph{kernel} and
\emph{image} are well-defined. The kernel~$\ker\boundary{d}$ contains
all \mbox{$d$-dimensional} simplicial chains that do not have
a boundary. We can make this more precise by using a construction from
group theory.

\paragraph{Homology groups.}
The last ingredient for the connectivity analysis involves calculating
a special group, the homology group. The $d$th homology group $\homologygroup{d}(\simplicialcomplex)$ of $\simplicialcomplex$ is defined as
the \emph{quotient group} $\homologygroup{d}(\simplicialcomplex) := \ker\boundary{d} / \im\boundary{d+1}$.
The quotient operation can be thought of as calculating a special
subset---the kernel of the boundary operator---and then \emph{removing}
another subset, namely the image of the boundary operator with an
increased dimension. This behoves a short explanation. The main reason
behind this operation is that \emph{just} the kernel calculation is
insufficient to properly count a hole. For example, if we take the three
edges of any triangle, their boundary will always be empty, i.e.\ they
are a part of $\ker\boundary{1}$. However, if the \emph{interior} of the
triangle is also part of the simplicial complex---in other words, if we
have the corresponding \mbox{$2$-simplex} as well---we should \emph{not}
count the edges as a hole. This is why we need to remove all elements in
the image of $\boundary{2}$. Coincidentally, this also explains why
cycles in a graph never `vanish'---there are simply no
\mbox{$2$-simplices} available since the graph is only
a \mbox{$1$-dimensional} simplicial complex.

\paragraph{Betti numbers.}
To fully close the loop, as it were, it turns out that we can calculate
Betti numbers from homology groups. Specifically, the \emph{rank} of the
$d$th homology group---in the group-theoretical sense---is precisely the $d$th Betti number~$\betti{d}$,
i.e.\ $\betti{d}(\simplicialcomplex) := \rank\homologygroup{d}(\simplicialcomplex)$.
The sequence of Betti numbers $\betti{0},\dots,\betti{d}$ of
a $d$-dimensional simplicial complex is commonly used as a 
complexity measure, and they can be used to discriminate manifolds.
For example, a $2$-sphere has Betti numbers $(1,0,1)$, while a $2$-torus
has Betti numbers $(1,2,1)$.
As we outlined in the main text, Betti numbers are of limited use when
dealing with complex graphs, however, because they are very coarse
counts of features. It was this limited expressivity that prompted the
development of persistent homology.

%%%%%%%%%%%%%%%%%%%%%%%%%%%%%%%%%%%%%%%%%%%%%%%%%%%%%%%%%%%%%%%%%%%%%%%%
\subsection{Persistent Homology}
%%%%%%%%%%%%%%%%%%%%%%%%%%%%%%%%%%%%%%%%%%%%%%%%%%%%%%%%%%%%%%%%%%%%%%%%
%
Persistent homology is an extension of simplicial homology, which
employs \emph{filtrations} to imbue a simplicial complex~$\simplicialcomplex$ with
scale information. Let us assume the existence of a function
$f\colon\simplicialcomplex\to\reals$, which only attains a finite number of function values
$f^{(0)} \leq f^{(1)} \leq \dots \leq \dots f^{(m-1)} \leq f^{(m)}$. We
may now, as in the main text, sort~$\simplicialcomplex$ according
to~$f$, leading again to a nested sequence of simplicial complexes
\begin{equation}
  \emptyset = \simplicialcomplex^{(0)} \subseteq \simplicialcomplex^{(1)} \subseteq \dots \subseteq \simplicialcomplex^{(m-1)} \subseteq \simplicialcomplex^{(m)} = \simplicialcomplex,
\end{equation}
in which $\simplicialcomplex^{(i)} := \left\{\sigma \in K \mid f(\sigma) \leq
f^{(i)} \right\}$, i.e.\ each subset
contains only those simplices whose function value is less than or equal
to the threshold.
In contrast to simplicial homology, the filtration holds potentially more
information because it can track \emph{changes}! For example,
a topological feature might be \emph{created}~(a new connected component
might arise) or \emph{destroyed}~(two connected components might merge
into one), as we go from some $\simplicialcomplex^{(i)}$ to
$\simplicialcomplex^{(i+1)}$.
At its core, persistent homology is `just' a way of tracking topological
features, representing each one by a creation and destruction value
$(f^{(i)}, f^{(j)}) \in \reals^2$, with $i \leq j$. In case
a topological feature is still present in $\simplicialcomplex^{(m)}
= \simplicialcomplex$, such as a cycle in a graph, it can also be
assigned a tuple of the form $(f^{(i)}, \infty)$. Such tuples constitute
\emph{essential features} of a simplicial complex and are usually
assigned a large destruction value or treated separately in an
algorithm~\citep{Hofer17}. While it is also possible to obtain only
tuples with finite persistence values, a process known as
\emph{extended persistence}~\citep{Cohen-Steiner09}, we focus only on
`ordinary' persistence in this paper because of the lower computational
complexity.

\paragraph{Persistent homology groups.}
The filtration above is connected by the inclusion homomorphism between
$\simplicialcomplex^{(i)} \subseteq \simplicialcomplex^{(i+1)}$. The
respective boundary homomorphisms induce a homomorphism between corresponding homology
groups of the filtration, i.e.\ a map $\mathfrak{i}_d^{(i,j)} \colon \homologygroup{d}(\simplicialcomplex_i) \to
\homologygroup{d}(\simplicialcomplex_j)$.
This family of homomorphisms gives rise to a sequence of homology groups
\begin{equation}
  \begin{split}
    &\homologygroup{d}\left(\simplicialcomplex^{(0)}\right)
    \xrightarrow{\mathfrak{i}_d^{(0,1)}}
    \homologygroup{d}\left(\simplicialcomplex^{(1)}\right)
    \xrightarrow{\mathfrak{i}_d^{(1,2)}}  \dots
    \xrightarrow{\mathfrak{i}_d^{(m-2,m-1)}}  \\
    &\homologygroup{d}\left(\simplicialcomplex^{(m-1)}\right)\xrightarrow{\mathfrak{i}_d^{(m-1,m)}}\homologygroup{d}\left(\simplicialcomplex^{(m)}\right) = \homologygroup{d}\left(\simplicialcomplex\right)
  \end{split}
\end{equation}
for every dimension $d$. For $i \leq j$, the $d$th persistent
homology group is defined as
\begin{equation}
  \persistenthomologygroup{d}{i,j} :=
  \ker\boundary{d}\left(\simplicialcomplex^{(i)}\right) / \left(
  \im\boundary{d+1}\left(\simplicialcomplex^{(j)}\right)\cap\ker\boundary{d}\left(\simplicialcomplex^{(i)}\right)\right).
\end{equation}
Intuitively, this group contains all homology classes \emph{created} in
$\simplicialcomplex^{(i)}$ that are \emph{also} present in
$\simplicialcomplex^{(j)}$.
Similar to the `ordinary' Betti number, we may now define the $d$th
persistent Betti number as the rank of this group, i.e.\
\begin{equation}
  \persistentbetti{d}{i,j} := \rank \persistenthomologygroup{d}{i,j}.
\end{equation}
Noting the set of indices $i, j$, we can see persistent homology as
generating a \emph{sequence} of Betti numbers, as opposed to just
calculating a single number. This makes it possible for us to describe
topological features in more detail, and summarise them in
a \emph{persistence diagram}.

%%%%%%%%%%%%%%%%%%%%%%%%%%%%%%%%%%%%%%%%%%%%%%%%%%%%%%%%%%%%%%%%%%%%%%%%
\paragraph{Persistence diagrams.}
%%%%%%%%%%%%%%%%%%%%%%%%%%%%%%%%%%%%%%%%%%%%%%%%%%%%%%%%%%%%%%%%%%%%%%%%
%
Given a filtration induced by a function $f\colon\simplicialcomplex\to\reals$ as described above, 
we store each tuple $(f^{(i)}, f^{(j)})$ with multiplicity
\begin{equation}
  \mu_{i,j}^{(d)} := \left( \persistentbetti{d}{i,j-1} - \persistentbetti{d}{i,j} \right) - \left( \persistentbetti{d}{i-1,j-1} - \persistentbetti{d}{i-1,j} \right)
\end{equation}
in the $d$th persistence diagram $\diagram_d$. Notice that for most
pairs of indices, $\mu_{i,j}^{(d)} = 0$.
Given a point $(x,y) \in \diagram_d$, we refer to the quantity
$\persistence(x,y) := |y-x|$ as its \emph{persistence}.
The idea of persistence arose in multiple contexts~\citep{Barannikov94,
Edelsbrunner02, Verri93}, but it is nowadays commonly used to analyse
functions on manifolds, where high persistence is seen to correspond to
\emph{features} of the function, while low persistence is typically
considered \emph{noise}.

%%%%%%%%%%%%%%%%%%%%%%%%%%%%%%%%%%%%%%%%%%%%%%%%%%%%%%%%%%%%%%%%%%%%%%%%
\subsection{Examples of Graph Filtrations}
\label{sec:Examples of Graph Filtrations}
%%%%%%%%%%%%%%%%%%%%%%%%%%%%%%%%%%%%%%%%%%%%%%%%%%%%%%%%%%%%%%%%%%%%%%%%

Before discussing the computational complexity of persistent homology,
we briefly provide some worked examples that highlight the impact of
choosing different filtrations for analysing graphs.
Filtrations are most conveniently thought of as arising from a function
$f\colon G \to \reals$, which assigns a scalar-valued function value to
each node and edge of the graph by means of setting $f(u, v) :=
\max\{f(u), f(v)\}$ for an edge~$(u, v)$.
In this context, $f$ is often picked to measure certain salient vertex features of~$G$, such
as the degree~\citep{Hofer17}, or its structural role in terms of a heat
kernel~\citep{Carriere20a}. The resulting topological features are then
assigned the respective function values, i.e.\ $(i,j) \mapsto \left(f(u_i), f(u_j)\right)$.
{%
  \tikzset{%
    every label/.append style = {%
      font = \scriptsize
    },
    vertices/.pic = {%
      \coordinate (A) at (0.00, 0.40);
      \coordinate (B) at (0.00, 0.10);
      \coordinate (C) at (0.40, 0.00);
      \coordinate (D) at (0.40, 0.50);
      \coordinate (E) at (0.75, 0.25);
      \foreach \c in {A,B,C,D,E}
      {
        \filldraw[vertex] (\c) circle (1pt);
      }
    },
    edges/.pic = {%
      \draw[edge] (A) -- (D);
      \draw[edge] (B) -- (C);
      \draw[edge] (C) -- (D);
      \draw[edge] (C) -- (E);
      \draw[edge] (D) -- (E);
    },
    vertex/.style = {
      draw = black,
      fill = black,
    },
    edge/.style = {
      draw = black
    },
    % This is more of an offset rather than a distance because the
    % labels will *not* overlap.
    label distance = -1mm,
  }
  As a brief example, consider a degree-based filtration of a simple
  graph. The filtration values are shown as numbers next to a vertex; we
  use $a^{(j)}$ to denote the filtration value at step~$j$~(this
  notation will be useful later when dealing with multiple filtrations).
  In each filtration step, the new nodes and edges added to the
  graph are shown in \textcolor{cardinal}{red}, while black elements
  indicate the structure of the graph that already exists at this step.
  \begin{center}
    \begin{tikzpicture}[baseline]
      \begin{scope}
        \pic{vertices};
        \pic{edges};

        % Add labels here to indicate the filtration
        \node[label=above:$1$] at (A) {};
        \node[label=below:$1$] at (B) {};
        \node[label=below:$3$] at (C) {};
        \node[label=above:$3$] at (D) {};
        \node[label=right:$2$] at (E) {};
      \end{scope}

      % Change default drawing behaviour of all edges
      \tikzset{%
        vertex/.style = {
          fill = lightgrey,
          draw = lightgrey,
        },
        edge/.style = {
          draw = lightgrey,
          dashed,
        }
      }

      \begin{scope}[xshift = 1.25cm]
        \pic{vertices};
        \pic{edges};
      
        \node[label={$G^{(1)}$}] at (0.5,-0.5) {};

        \filldraw[cardinal] (A) circle (1pt);
        \filldraw[cardinal] (B) circle (1pt);
      \end{scope}

      \begin{scope}[xshift = 2.25cm]
        \pic{vertices};
        \pic{edges};

        \node[label={$G^{(2)}$}] at (0.5,-0.5) {};

        \filldraw[black] (A) circle (1pt);
        \filldraw[black] (B) circle (1pt);

        \filldraw[cardinal] (E) circle (1pt);
      \end{scope}

     \begin{scope}[xshift = 3.25cm]
        \pic{vertices};
        \pic{edges};

        \node[label={$G^{(3)}$}] at (0.5,-0.5) {};

        \draw[cardinal] (A) -- (D);
        \draw[cardinal] (B) -- (C);
        \draw[cardinal] (C) -- (D);
        \draw[cardinal] (C) -- (E);
        \draw[cardinal] (D) -- (E);

        \foreach \c in {A,B,E}
        {
          \filldraw[black] (\c) circle (1pt);
        }

        \filldraw[cardinal] (C) circle (1pt);
        \filldraw[cardinal] (D) circle (1pt);
      \end{scope}
    \end{tikzpicture}
  \end{center}

  Since all edges are inserted at $a^{(3)} = 3$, we obtain the following \mbox{$0$-dimensional} persistence diagram $\diagram_0 = \left\{(1, \infty), (1, 3), (2, 3), (3, 3), (3, 3)\right\}$.
  The existence of a single tuple of the form $(\cdot, \infty)$
  indicates that $\betti{0} = 1$. Similarly, there is only one cycle in
  the data set, which is created at $a^{(3)}$, leading to $\diagram_1
  = \left\{(3, \infty)\right\}$ and $\betti{1} = 1$.
  If we change the filtration such that each vertex has a \emph{unique}
  filtration value, we obtain a different ordering and different
  persistence tuples, as well as more filtration steps:
  \begin{center}
    \begin{tikzpicture}
      \begin{scope}
        \pic{vertices};
        \pic{edges};

        % Add labels here to indicate the filtration
        \node[label=above:$1$] at (A) {};
        \node[label=below:$2$] at (B) {};
        \node[label=below:$3$] at (C) {};
        \node[label=above:$4$] at (D) {};
        \node[label=right:$5$] at (E) {};
      \end{scope}

      % Change default drawing behaviour of all edges and vertices,
      % necessitating the individual placement of nodes.  
      \tikzset{%
        vertex/.style = {
          fill = lightgrey,
          draw = lightgrey,
        },
        edge/.style = {
          draw = lightgrey,
          dashed,
        }
      }

      \begin{scope}[xshift = 1.25cm]
        \pic{vertices};
        \pic{edges};

        \node[label={$G^{(1)}$}] at (0.5,-0.5) {};

        \filldraw[cardinal] (A) circle (1pt);
      \end{scope}

      \begin{scope}[xshift = 2.25cm]
        \pic{vertices};
        \pic{edges};

        \node[label={$G^{(2)}$}] at (0.5,-0.5) {};

        \filldraw[black]    (A) circle (1pt);
        \filldraw[cardinal] (B) circle (1pt);
      \end{scope}

     \begin{scope}[xshift = 3.25cm]
        \pic{vertices};
        \pic{edges};

        \node[label={$G^{(3)}$}] at (0.5,-0.5) {};

        \draw[cardinal] (B) -- (C);

        \filldraw[black]    (A) circle (1pt);
        \filldraw[black]    (B) circle (1pt);
        \filldraw[cardinal] (C) circle (1pt);
      \end{scope}

     \begin{scope}[xshift = 4.25cm]
        \pic{vertices};
        \pic{edges};

        \node[label={$G^{(4)}$}] at (0.5,-0.5) {};

        \draw[cardinal] (A) -- (D);
        \draw[cardinal] (C) -- (D);

        \draw[black]    (B) -- (C);

        \filldraw[black]    (A) circle (1pt);
        \filldraw[black]    (B) circle (1pt);
        \filldraw[black]    (C) circle (1pt);
        \filldraw[cardinal] (D) circle (1pt);
      \end{scope}

      \begin{scope}[xshift = 5.25cm]
        \pic{vertices};
        \pic{edges};

        \node[label={$G^{(5)}$}] at (0.5,-0.5) {};

        \draw[cardinal] (C) -- (E);
        \draw[cardinal] (D) -- (E);

        \draw[black]    (A) -- (D);
        \draw[black]    (C) -- (D);
        \draw[black]    (B) -- (C);

        \filldraw[black]    (A) circle (1pt);
        \filldraw[black]    (B) circle (1pt);
        \filldraw[black]    (C) circle (1pt);
        \filldraw[black]    (D) circle (1pt);
        \filldraw[cardinal] (E) circle (1pt);
      \end{scope}
    \end{tikzpicture}
  \end{center}
  Here, connected components are not all created `at once', but more
  gradually, leading to $\diagram_0 = \left\{(1, \infty), (3, 3), (2,
  4), (4, 4), (5, 5)\right\}$. Of particular interest is the tuple $(2,
  4)$. It was created by the vertex with filtration value~$2$. In
  $G^{(3)}$ it merges with another connected component, namely the one
  created by the vertex with function value~$3$. This leads to the tuple
  $(3, 3)$, because in each merge, we merge from the `younger'
  component~(the one arising \emph{later} in the filtration) to the
  `older' component~(the arising \emph{earlier} in the filtration).
  Again, there is only a single cycle in the data set, which we detect
  at $a^{(5)} = 5$, leading to $\diagram_1 = \left\{(5, \infty)\right\}$.
}

%%%%%%%%%%%%%%%%%%%%%%%%%%%%%%%%%%%%%%%%%%%%%%%%%%%%%%%%%%%%%%%%%%%%%%%%
\subsection{Computational Details}
\label{sec:Computational details}
%%%%%%%%%%%%%%%%%%%%%%%%%%%%%%%%%%%%%%%%%%%%%%%%%%%%%%%%%%%%%%%%%%%%%%%%

The cardinality of the persistence diagram of dimension-0, $\diagram_0$
is equal to the number of vertices, $n$ in the graph. A natural pairing
of persistence tuples then consists in assigning each tuple to the node
that generated it. As for dimension-1, $\diagram_1$ contains as many
tuples as cycles in the graph. However, there is no bijective mapping
between dimension-1 persistence tuples and vertices. Rather, we link each
dimension-1 tuple to the edge that created that particular cycle. To
account for multiple distinct filtrations and because the same cycle can
be assigned to different edges depending on the specific filtration
function, we define a \emph{dummy} tuple for edges that are not linked to
any cycle for a particular filtration. The set of persistence diagrams
$(\diagram_1^{(l)}, \dots, \diagram_k^{(l)})$ can then be concatenated as
a matrix in $\reals^{ d_l \times 2 \cdot k}$, with $ d_l = \vert V \vert$
(the number of vertices in graph $G$) if $l=0$ and   $ d_l =  \vert
E \vert $ (the number of edges in $G$) if $l=1$. Remarkably, that leads
to $\Psi^{(l)}\left(\diagram_1^{(l)}, \dots, \diagram_k^{(l)}\right)$
being an operator on a matrix, significantly facilitating computations.

Because of its natural indexing to the vertices,
$\Psi^{(0)}\left(\diagram_1^{(0)}, \dots, \diagram_k^{(0)}\right)$ can be
mapped back to the graph as explained in \autoref{sec:Our Method}.
For $l=1$, we pool $\Psi^{(1)}\left(\diagram_1^{(1)}, \dots,
\diagram_k^{(1)}\right)$ to a graph-level embedding, and mask out the
edge indices that are not assigned a persistence pair in any of the
$\diagram_k^{(1)}$.

%%%%%%%%%%%%%%%%%%%%%%%%%%%%%%%%%%%%%%%%%%%%%%%%%%%%%%%%%%%%%%%%%%%%%%%%
\section{Theoretical Expressivity of \method: Proofs}
\label{sec:Proofs}
%%%%%%%%%%%%%%%%%%%%%%%%%%%%%%%%%%%%%%%%%%%%%%%%%%%%%%%%%%%%%%%%%%%%%%%%

This section provides more details concerning the proofs about the
expressivity of our method. We first state a `weak' variant of our
theorem, in which injectivity of the filtration function is not
guaranteed. Next, we show that the filtration function~$f$ constructed
in this theorem can be used to prove the existence of an injective
function~$\tilde{f}$ that is arbitrarily close to~$f$ and provides the
same capabilities of distinguishing graphs.

\begin{theorem}[Expressivity, weak version]
  Persistent homology is \emph{at least} as expressive as WL[$1$], i.e.\
  if the WL[$1$] label sequences for two graphs $G$ and $G'$ diverge,
  there is a filtration~$f$ such that the corresponding $0$-dimensional
  persistence diagrams $\diagram_0$ and $\diagram'_0$ are not equal.
  \label{thm:Expressivity weak}
\end{theorem}
\begin{proof}
  Assume that the label sequences of $G$ and $G'$ diverge at iteration
  $h$. Thus, $\wl{h}_G \neq \wl{h}_{G'}$ and there exists at least one
  label whose count is different. Let
  $\mathcal{L}^{(h)} := \{l_1, l_2, \dots\}$ be an enumeration of
  the finitely many hashed labels at iteration $h$. We can build
  a  filtration function~$f$ by assigning a vertex~$v$ with label
  $l_i$ to its index, i.e.\ $f(v) := i$, and setting $f(v,
  w) := \max\left\{f(v), f(w)\right\}$ for an edge~$(v, w)$. The
  resulting \mbox{0-dimensional} persistence diagram~$\diagram_0$~(and
  $\diagram_0'$ for $G'$) will contain tuples of the form~$(i, j)$, and
  each vertex is guaranteed to give rise to \emph{exactly} one such
  pair. Letting $\multiplicity{i,j}\left(\diagram_0\right)$ refer to the
  multiplicity of a tuple in~$\diagram_0$, we know that, since the label
  count is different, there is \emph{at least} one tuple $(k, l)$ with
  $\multiplicity{k,l}\mleft(\diagram_0\mright) \neq
  \multiplicity{k,l}\mleft(\diagram_0'\mright)$.
  Hence, $\diagram_0 \neq \diagram_0'$.
\end{proof}
While \autoref{thm:Expressivity weak} proves the \emph{existence} of
a filtration, the proof is constructive and relies on the existence of
the WL[$1$] labels. Moreover, the resulting filtration is typically not
\emph{injective}, thus precluding the applicability of
\autoref{thm:Differentiability}. The following Lemma discusses
a potential fix for filtration functions of arbitrary dimensionality;
our context is a special case of this.
\begin{lemma}
  For all $\epsilon > 0$ and $f\colon V \to \reals^d$ there is an
  injective function $\tilde{f}$ such that $\|f - \tilde{f}\|_\infty
  < \epsilon$.
  \label{lem:Injective function}
\end{lemma}
\begin{proof}
  Let $V = \{v_1, \dots, v_n\}$ the vertices of a graph and $\im f = \{u_1,
  \dots, u_m\} \subset \reals^d$ their images under~$f$. Since~$f$ is not
  injective, we have $m < n$. We resolve non-injective vertex pairs
  iteratively.
  For $u \in \im f$, let $V' := \{v \in V \mid f(v) = u\}$. If $V'$ only
  contains a single element, we do not have to do anything. Otherwise,
  for each $v' \in V'$, pick a new value from $\ball{\epsilon}{u}
  \setminus \im f$, where $\ball{r}{x} \subset \reals^d$ refers to the
  open ball of radius~$r$ around a point~$x$~(for our case, i.e.\ $d
  = 1$, this becomes an open interval in~$\reals$).
  Since we only ever
  remove a finite number of points, such a new value always exists, and
  we can modify~$\im f$ accordingly. The number of vertex pairs for
  which~$f$ is non-injective decreases by at least one in every
  iteration, hence after a finite number of iterations, we have
  modified~$f$ to obtain~$\tilde{f}$, an \emph{injective} approximation
  to~$f$. By always picking new values from balls of radius $\epsilon$,
  we ensure that $\|f - \tilde{f}\|_\infty < \epsilon$, as required. 
\end{proof}
Using this Lemma, we can ensure that the function~$f$ in
\autoref{thm:Expressivity weak} is injective, thus ensuring differentiability
according to \autoref{thm:Differentiability}. However, we need to make
sure that the main message of \autoref{thm:Expressivity} still holds,
i.e.\ the two graphs $G$ and $G'$ must lead to different
persistence diagrams even under the injective filtration
function~$\tilde{f}$. This is addressed by the following lemma, which
essentially states that moving from~$f$ to an injective~$\tilde{f}$ does
not result in coinciding persistence diagrams.
\begin{lemma}
  Let $G$ and $G'$ be two graphs whose $0$-dimensional persistence
  diagrams $\diagram_0$ and $\diagram'_0$ are calculated using
  a filtration function~$f$ as described in \autoref{thm:Expressivity
  weak}.
  Moreover, given $\epsilon > 0$, let $\tilde{f}$ be an injective
  filtration function with $\|f - \tilde{f}\|_\infty$ and corresponding
  $0$-dimensional persistence diagrams $\widetilde{\diagram_0}$ and
  $\widetilde{\diagram_0'}$. If $\diagram_0 \neq \diagram'_0$, we also
  have $\widetilde{\diagram_0} \neq \widetilde{\diagram_0'}$.
  \label{lem:Injective function behaviour}
\end{lemma}
\begin{proof}
Since $\tilde{f}$ is injective, each tuple in $\widetilde{\diagram_0}$
and $\widetilde{\diagram_0'}$ has multiplicity~$1$. But under~$f$, there
were differences in multiplicity for at least one tuple~$(k, l)$. Hence,
given~$\tilde{f}$, there exists at least one tuple~$(k, l) \in
\widetilde{\diagram_0} \cup \widetilde{\diagram_0'}$ with $(k, l)
\notin \widetilde{\diagram_0} \cap \widetilde{\diagram_0'}$.
As a consequence, $\widetilde{\diagram_0} \neq \widetilde{\diagram_0'}$.
\end{proof}
For readers familiar with TDA, this lemma can also be proved in
a simpler way by noting that if the bottleneck distance between
$\diagram_0$ and $\diagram'_0$ is non-zero, it will remain so if one
picks a perturbation that is sufficiently small~(this fact can also be
proved more rigorously). In any case, the preceding proofs enable us to
finally state a stronger variant of the expressivity theorem, initially
stated as \autoref{thm:Expressivity} in the main paper. For consistency
reasons with the nomenclature in this section, we slightly renamed the
filtration function.
\begin{theorem}[Expressivity, strong version]
  Persistent homology is \emph{at least} as expressive as WL[$1$], i.e.\
  if the WL[$1$] label sequences for two graphs $G$ and $G'$ diverge,
  there exists an injective filtration~$\tilde{f}$ such that the
  corresponding $0$-dimensional persistence diagrams $\widetilde{\diagram_0}$ and $\widetilde{\diagram'_0}$ are not equal.
\end{theorem}
\begin{proof}
  We obtain a non-injective filtration function~$f$ from
  \autoref{thm:Expressivity weak}, under which $\diagram_0$ and
  $\diagram'_0$ are not equal. By \autoref{lem:Injective function}, for
  every $\epsilon > 0$,  we can find an injective function~$\tilde{f}$
  with $\|f - \tilde{f}\|_\infty < \epsilon$. According to
  \autoref{lem:Injective function behaviour}, the persistence diagrams
  calculated by this function do not coincide, i.e.\
  $\widetilde{\diagram_0} \neq \widetilde{\diagram_0'}$. Hence,
  $\tilde{f}$ is a filtration function and, according to
  \autoref{thm:Differentiability}, differentiability is ensured.
\end{proof}

%%%%%%%%%%%%%%%%%%%%%%%%%%%%%%%%%%%%%%%%%%%%%%%%%%%%%%%%%%%%%%%%%%%%%%%%
\section{Empirical Expressivity: Analysis of (Strongly) Regular Graphs}
\label{sec:Regular Graphs}
%%%%%%%%%%%%%%%%%%%%%%%%%%%%%%%%%%%%%%%%%%%%%%%%%%%%%%%%%%%%%%%%%%%%%%%%

A \emph{$k$-regular graph} is a graph $G = (V, E)$ in which all vertices
have degree~$k$. For $k = \{3, 4\}$, such graphs are also known as cubic
and quartic graphs, respectively.
The Weisfeiler--Lehman test is capable of distinguishing between certain
variants of these graphs~(even though we observe that WL[$1$] is not
sufficient to do so).
Similarly, a \emph{strongly regular} graph is a graph $G = (V, E)$ with
two integers $\lambda, \mu \in \naturals$ such that each pair of
adjacent vertices has~$\lambda$ common neighbours, whereas every pair of
non-adjacent vertices has~$\mu$ common neighbours.

Persistent homology can make use of higher-order
connectivity information to distinguish between these data sets. To
demonstrate this, we use a standard degree filtration and compute
persistent homology of the graph, including all higher-order cliques. We
then calculate the \emph{total persistence} of each persistence
diagram~$\diagram$, and use it to assign a feature vector to the graph. 
This is in some sense the simplest way of employing persistent
homology; notice that we are \emph{not} learning a new filtration but
keep a fixed one. Even in this scenario, we find that there is always
a significant number of pairs of graphs whose feature vectors do not
coincide---or, conversely speaking, as \autoref{tab:Regular graphs}
shows, there are only between 14\%--22\% of
pairs of graphs that we cannot distinguish by this simple scheme. This
illustrates the general expressivity that a topology-based perspective
can yield. For the strongly-regular  graphs, we observe even lower error
rates: we only fail to distinguish about 1.2\% of all
pairs~(specifically, 8236 out of 7424731 pairs) of the 3854 strongly-regular graphs on~$35$
vertices with $\lambda = \mu = 9$~\citep{mckay2001classification}.

\paragraph{A note on computational complexity.}
At the same time, this approach is also not without its disadvantages.
Since we perform a clique counting operation, the complexity of the
calculation increases considerably, and we would not suggest to use
persistent homology of arbitrary dimensions in practice. While there
are some algorithmic improvements in topological
constructions~\citep{Zomorodian10a}, na\"ive persistent homology
calculations of arbitrary order may quickly become infeasible for larger
data sets. This can be alleviated to a certain extent by
\emph{approximating} persistent homology~\citep{Cavanna15a, Cavanna15b,
Sheehy13}, but the practical benefits of this are unclear. Nevertheless, this
experiment should therefore be considered as an indication of the
utility of topological features in general to complement and enhance
existing architectures.  

%%%%%%%%%%%%%%%%%%%%%%%%%%%%%%%%%%%%%%%%%%%%%%%%%%%%%%%%%%%%%%%%%%%%%%%%
\begin{table}
  \caption{%
    Error rates when using persistent homology with a degree filtration
    to classify pairs of \mbox{$k$-regular} on~$n$ vertices.
    \textsc{r3-n12} denotes \mbox{$3$-regular} graphs on~$12$ vertices,
    for instance. This list is by no means exhaustive, but indicates the
    general utility of persistent homology and its filtration-based
    analysis.
  }
  \label{tab:Regular graphs}
  \centering
  \begin{tabular}{lrrrl}
    \toprule
    {Data set} & {Graphs} & {Pairs} & {Error} & {Error rate}\\
    \midrule
    \textsc{r3-n12} &   85 &    3570 &     712 & 19.94\%\\
    \textsc{r3-n14} &  509 &  129286 &   26745 & 20.69\%\\
    \textsc{r3-n16} & 4060 & 8239770 & 1757385 & 21.33\%\\
    \textsc{r4-n10} &   59 &    1711 &     229 & 13.38\%\\
    \textsc{r4-n11} &  265 &   34980 &    4832 & 13.81\%\\
    \textsc{r4-n12} & 1544 & 1191196 &  170814 & 14.34\%\\
    \bottomrule
  \end{tabular}
\end{table}
%%%%%%%%%%%%%%%%%%%%%%%%%%%%%%%%%%%%%%%%%%%%%%%%%%%%%%%%%%%%%%%%%%%%%%%%

%%%%%%%%%%%%%%%%%%%%%%%%%%%%%%%%%%%%%%%%%%%%%%%%%%%%%%%%%%%%%%%%%%%%%%%%
\section{Experimental Setup \& Computational Resources}\label{sec:Training}
%%%%%%%%%%%%%%%%%%%%%%%%%%%%%%%%%%%%%%%%%%%%%%%%%%%%%%%%%%%%%%%%%%%%%%%%

Following the setup of \citet{dwivedi2020benchmarking}, we implemented
the following training procedure:  All models are initialised with an
initial learning rate $lr_{init}$, which varies between the data sets.
During training the loss on the validation split is monitored and the
learning rate is halved if the validation loss does not improve over
a period of $lr_{patience}$.  Runs are stopped when the learning rate
gets reduced to a value of lower than $lr_{min}$.  The parameters for the
different data sets are shown in \autoref{tab:training parameters}.

%%%%%%%%%%%%%%%%%%%%%%%%%%%%%%%%%%%%%%%%%%%%%%%%%%%%%%%%%%%%%%%%%%%%%%%%
\begin{table}[tbp]
  \caption{%
  Parameters of learning rate scheduling~(including `patience'
  parameters) for training of models in this work.
  }
  \label{tab:training parameters}
  \begin{tabular}{p{0.1\columnwidth}p{0.37\columnwidth}p{0.37\columnwidth}}
    \toprule
    \textsc{Name} & \data{MNIST}, \data{CIFAR-10}, \data{PATTERNS}, \data{CLUSTER} & \data{PROTEINS}, \data{ENZYMES}, \data{DD} \\
    \midrule
    $lr_{init}$ & \num{1e-3} & \num{7e-4} \\
    $lr_{min}$ & \num{1e-5} & \num{1e-6} \\
    $lr_{patience}$ & $10$ & $25$ \\
    \bottomrule
  \end{tabular}
\end{table}
%%%%%%%%%%%%%%%%%%%%%%%%%%%%%%%%%%%%%%%%%%%%%%%%%%%%%%%%%%%%%%%%%%%%%%%%

Our method is implemented in \texttt{Python}, making heavy use of the
\texttt{pytorch-geometric} library~\citep{Fey19}, licensed under the MIT
License, and the \texttt{pytorch-lightning} library~\citep{falcon2019pytorch}, licensed
under the Apache 2.0 License.
The training and hyperparameter selection was performed using `Weights
and Biases'~\citep{wandb}, resulting in additional reports, tables, and
log information, which will simplify the reproducibility of this work.
We will make our own code available, using
either the MIT License or a BSD 3-Clause license, which precludes
endorsing/promoting our work without prior approval. All licenses used
for the code are compatible with this licensing choice.

As for the data sets, we use data sets that are available in \texttt{pytorch-geometric} 
for graph learning tasks. Some of the benchmark data sets have been originally provided by
\citet{morris2020tudataset}, others~(\data{CIFAR-10}, \data{Cluster},
\data{MNIST}, \data{Pattern}) have been provided by \citet{dwivedi2020benchmarking}
in the context of a large-scale graph neural network benchmarking
effort.

%%%%%%%%%%%%%%%%%%%%%%%%%%%%%%%%%%%%%%%%%%%%%%%%%%%%%%%%%%%%%%%%%%%%%%%%
\paragraph{Ablated static variant of \method.}
%%%%%%%%%%%%%%%%%%%%%%%%%%%%%%%%%%%%%%%%%%%%%%%%%%%%%%%%%%%%%%%%%%%%%%%%
%
Next to all the experiments presented in the main paper, we also
developed a \emph{static variant} of our layer, serving as an additional
ablation method~(this nomenclature will be used in all supplementary
tables).
The static variant mimics the calculation of a filtration in terms of
the number of parameters but without taking topological features into
account. The layer uses a \emph{static} mapping~(instead of a dynamic
one based on persistent homology) of vertices to themselves~(for
dimension~$0$), and employs a random edge selection process~(for
dimension~$1$). This has the effect of learning graph-level information
that is not strictly based on topology but on node feature values. The
static variant of \method reduces to an MLP that is applied per vertex
in case interactions between tuples are not considered, and to the
application of a \texttt{DeepSet}~\cite{zaheer2017deepsets} to the
vertex representations if interactions between tuples are incorporated.
Generally, if the static variant performs well on a data set, we assume
that performance is driven much more by the availability of \emph{any}
graph-level type of information, such as the existence of certain nodes
or groups of nodes, as opposed to topological information.

\paragraph{Compute resources.}
Most of the jobs were run on our internal cluster, comprising 64
physical cores~(\texttt{Intel(R) Xeon(R) CPU E5-2620 v4 @ 2.10GHz}) with
8 GeForce GTX 1080 GPUs. A smaller fraction has been run on another
cluster, containing 40 physical cores~(\texttt{Intel(R) Xeon(R) CPU
E5-2630L v4 @ 1.80GHz}) with 2 Quadro GV100 GPUs and 1 Titan XP GPU.

%%%%%%%%%%%%%%%%%%%%%%%%%%%%%%%%%%%%%%%%%%%%%%%%%%%%%%%%%%%%%%%%%%%%%%%%
\section{Hyperparameters}
\label{sec:Hyperparameters}
%%%%%%%%%%%%%%%%%%%%%%%%%%%%%%%%%%%%%%%%%%%%%%%%%%%%%%%%%%%%%%%%%%%%%%%%

\autoref{tab:Hyperparameters} contains a listing of all hyperparameters
used to train \method. For the Weisfeiler--Lehman subtree kernel, which
we employed as a comparison partner on the synthetic data sets, we used
a support vector machine classifier with a linear kernel, whose
regularisation parameter $C \in \{10^{-4}, 10^{-3}, \dots, 10^4\}$ is
trained via $5$-fold cross validation, which is repeated $10$ times to
obtain standard deviations. This follows closely the setup in graph
classification literature.

%%%%%%%%%%%%%%%%%%%%%%%%%%%%%%%%%%%%%%%%%%%%%%%%%%%%%%%%%%%%%%%%%%%%%%%%
\begin{table}[tbp]                        
  \caption{%
    The set of hyperparameters that we use to train \method, along with
    their respective value ranges. Notice that `dropout` could be made
    configurable, but this would make our setup incomparable to the
    setup proposed by \citet{dwivedi2020benchmarking} for benchmarking
    GNNs.
  }
  \label{tab:Hyperparameters}
  \centering
  %\resizebox{\columnwidth}{!}{%
    \begin{tabular}{lc}
    \toprule
    \textsc{Name}               & \textsc{Value(s)}  \\
    \midrule
    \texttt{DeepSet}            & \{True, False\}\\
    Depth                       & $\{3,4\}$ \\
    Dim1                        & True \scriptsize (by default, we always use cycle information)\\
    Dropout                     & $0.$ \\
    Early Topo                  & (True, False) \\
    Static                      & \{True, False\} \scriptsize (to evaluate the static variant)\\
    Filtration Hidden Dimension & 32 \\
    Hidden Dimension            & 138--146 \\
    No.\ coordinate functions   & 3 \\
    No.\ filtrations            & 8 \\
    Residual and Batch Norm     & True\\
    Share filtration parameters & True\\
    \bottomrule
    \end{tabular}
  %}
\end{table}
%%%%%%%%%%%%%%%%%%%%%%%%%%%%%%%%%%%%%%%%%%%%%%%%%%%%%%%%%%%%%%%%%%%%%%%%

%%%%%%%%%%%%%%%%%%%%%%%%%%%%%%%%%%%%%%%%%%%%%%%%%%%%%%%%%%%%%%%%%%%%%%%%
\section{Extended Results for the Synthetic Data Sets}\label{sec:Synthetic data sets extended}
%%%%%%%%%%%%%%%%%%%%%%%%%%%%%%%%%%%%%%%%%%%%%%%%%%%%%%%%%%%%%%%%%%%%%%%%

In the main paper, we depicted a concise analysis of synthetic data
sets~(\autoref{fig:Synthetic data performance}). Here, we provide a more
detailed ablation of this data set, highlighting the differences between
\method and its static baseline, but also disentangling the performance
of \mbox{$0$-dimensional} and \mbox{$1$-dimensional} topological
features, i.e.\ connected components and cycles, respectively.
\autoref{fig:Synthetic data performance ablation} depicts the results.
As stated in the main paper, we observe that
\begin{inparaenum}[(i)]
  \item both cycles and connected components are crucial for reaching
    high predictive performance, and
  \item the static variant is performing slightly better than the
standard GCN, but \emph{significantly worse} than \method, due to its
very limited access to topological information. 
\end{inparaenum}

%%%%%%%%%%%%%%%%%%%%%%%%%%%%%%%%%%%%%%%%%%%%%%%%%%%%%%%%%%%%%%%%%%%%%%%%
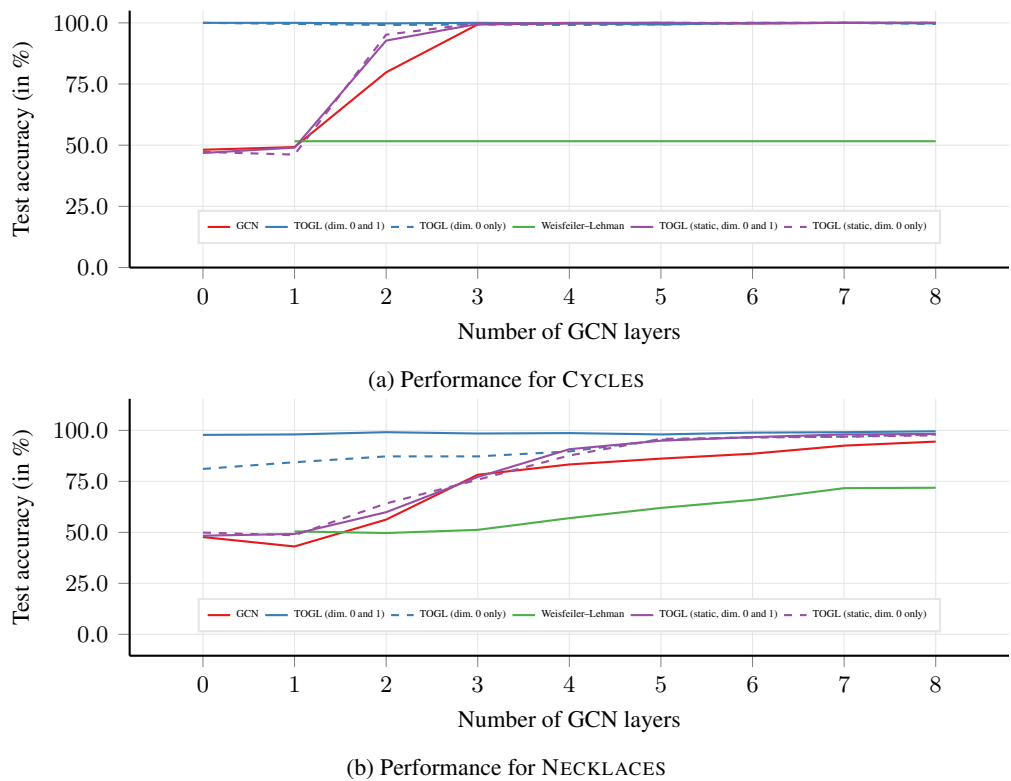
\begin{figure}[tbp]
  \centering
  \subcaptionbox{Performance for \data{Cycles}\label{sfig:Results cycles ablation}}{%
    \begin{tikzpicture}
      \begin{axis}[%
        perfplot,
        height                = 5.00cm,
        width                 = 0.95\linewidth,
      ]
        \addplot+[%
          discard if not = {dataset}{Cycles},
          ] table[x = depth, y expr = 100 * \thisrow{val_acc_mean}, y error expr = \thisrow{val_acc_std}, col sep = comma] {Data/GCN_summary.csv};

        \addlegendentry{GCN}

        \addplot+[%
          discard if not = {dataset}{Cycles},
          discard if not = {dim1}{True},
          discard if not = {fake}{False},
        ] table[x = depth, y expr = 100 * \thisrow{val_acc_mean}, y error expr = \thisrow{val_acc_std}, col sep = comma] {Data/TopoGNN_summary.csv};

        \addlegendentry{\method~(dim.\ 0 and 1)}

        \pgfplotsset{cycle list shift=-1}
        
        \addplot+[%
          discard if not = {dataset}{Cycles},
          discard if not = {dim1}{False},
          discard if not = {fake}{False},
          dashed,
        ] table[x = depth, y expr = 100 * \thisrow{val_acc_mean}, y error expr = \thisrow{val_acc_std}, col sep = comma] {Data/TopoGNN_summary.csv};

        \addlegendentry{\method~(dim.\ 0 only)}

        \addplot+[%
          discard if not = {Name}{WL_Cycles},
        ] table[x = depth, y = accuracy, col sep = comma, y = accuracy, y error = sdev] {Data/WL_synthetic.csv};

        \addlegendentry{Weisfeiler--Lehman}
        
        \addplot+[%
          discard if not = {dataset}{Cycles},
          discard if not = {dim1}{True},
          discard if not = {fake}{True},
        ] table[x = depth, y expr = 100 * \thisrow{val_acc_mean}, y error expr = \thisrow{val_acc_std}, col sep = comma] {Data/TopoGNN_summary.csv};

        \addlegendentry{\method~(static, dim.\ 0 and 1)}

        \pgfplotsset{cycle list shift=-2}
        
        \addplot+[%
          discard if not = {dataset}{Cycles},
          discard if not = {dim1}{False},
          discard if not = {fake}{True},
          dashed,
        ] table[x = depth, y expr = 100 * \thisrow{val_acc_mean}, y error expr = \thisrow{val_acc_std}, col sep = comma] {Data/TopoGNN_summary.csv};

        \addlegendentry{\method (static, dim.\ 0 only)}
     
      \end{axis}
    \end{tikzpicture}
  }
  \subcaptionbox{Performance for \data{Necklaces}\label{sfig:Results necklaces ablation}}{%
    \begin{tikzpicture}
      \begin{axis}[%
        perfplot,
        xmin                  = 0,
        xmax                  = 8,      % Ignore deeper WL results; but we have them if anyone
                                        % wants to look at them!
        enlargelimits         = true,   % Ensures that the plot is still as large
                                        % as the one preceding it.
        height                = 5.00cm,
        width                 = 0.95\linewidth,
      ]
        \addplot+[%
          discard if not = {dataset}{Necklaces},
        ] table[x = depth, y expr = 100 * \thisrow{val_acc_mean}, y error expr = \thisrow{val_acc_std}, col sep = comma] {Data/GCN_summary.csv};

        \addlegendentry{GCN}

        \addplot+[%
          discard if not = {dataset}{Necklaces},
          discard if not = {dim1}{True},
          discard if not = {fake}{False},
        ] table[x = depth, y expr = 100 * \thisrow{val_acc_mean}, y error expr = \thisrow{val_acc_std}, col sep = comma] {Data/TopoGNN_summary.csv};

        \addlegendentry{\method~(dim.\ 0 and 1)}

        \pgfplotsset{cycle list shift=-1}
        
        \addplot+[%
          discard if not = {dataset}{Necklaces},
          discard if not = {dim1}{False},
          discard if not = {fake}{False},
          dashed,
        ] table[x = depth, y expr = 100 * \thisrow{val_acc_mean}, y error expr = \thisrow{val_acc_std}, col sep = comma] {Data/TopoGNN_summary.csv};

        \addlegendentry{\method~(dim.\ 0 only)}

        \addplot+[%
          discard if not = {Name}{WL_Necklaces},
        ] table[x = depth, y = accuracy, col sep = comma, y = accuracy, y error = sdev] {Data/WL_synthetic.csv};

        \addlegendentry{Weisfeiler--Lehman}

        \addplot+[%
          discard if not = {dataset}{Necklaces},
          discard if not = {dim1}{True},
          discard if not = {fake}{True},
        ] table[x = depth, y expr = 100 * \thisrow{val_acc_mean}, y error expr = \thisrow{val_acc_std}, col sep = comma] {Data/TopoGNN_summary.csv};

        \addlegendentry{\method~(static, dim.\ 0 and 1)}
        
        \pgfplotsset{cycle list shift=-2}
        
        \addplot+[%
          discard if not = {dataset}{Necklaces},
          discard if not = {dim1}{False},
          discard if not = {fake}{True},
          dashed,
        ] table[x = depth, y expr = 100 * \thisrow{val_acc_mean}, y error expr = \thisrow{val_acc_std}, col sep = comma] {Data/TopoGNN_summary.csv};

        \addlegendentry{\method~(static, dim.\ 0 only)}
      \end{axis}
    \end{tikzpicture}
  }
  \caption{%
    Performance comparison on synthetic data sets as a function of the
    number of GCN layers or Weisfeiler--Lehman iterations. This is an
    extended version of \protect\autoref{fig:Synthetic data
    performance}. For \method, we show the performance with respect to
    the dimensionality of topological features that are being used.
    Since the standard deviations are negligible, we refrain from
    showing them here.
  }
  \label{fig:Synthetic data performance ablation}
\end{figure}
%%%%%%%%%%%%%%%%%%%%%%%%%%%%%%%%%%%%%%%%%%%%%%%%%%%%%%%%%%%%%%%%%%%%%%%%

%%%%%%%%%%%%%%%%%%%%%%%%%%%%%%%%%%%%%%%%%%%%%%%%%%%%%%%%%%%%%%%%%%%%%%%%
\section{Visualisation of the Learnt Filtrations}
\label{sec:Filtrations}
%%%%%%%%%%%%%%%%%%%%%%%%%%%%%%%%%%%%%%%%%%%%%%%%%%%%%%%%%%%%%%%%%%%%%%%%

\method can learn an arbitrary number of filtrations on the graphs. In
\autoref{fig:Filtrations}, we display 3 different filtrations learnt on
an randomly picked graph from the \data{DD} data set and compare it
against the classical node degree filtration. The width of the nodes is
plotted proportional to the node degree filtration while the colour
saturation is proportional to the learnt filtration. Filtration~0
appears to be correlated with the degree filtration, while other
filtrations learnt by \method seem to focus on different local
properties of the graph.

%%%%%%%%%%%%%%%%%%%%%%%%%%%%%%%%%%%%%%%%%%%%%%%%%%%%%%%%%%%%%%%%%%%%%%%%
\begin{figure}[tbp]
  \centering
  \subcaptionbox{\data{DD}}{%
    \includegraphics[width=0.33\linewidth]{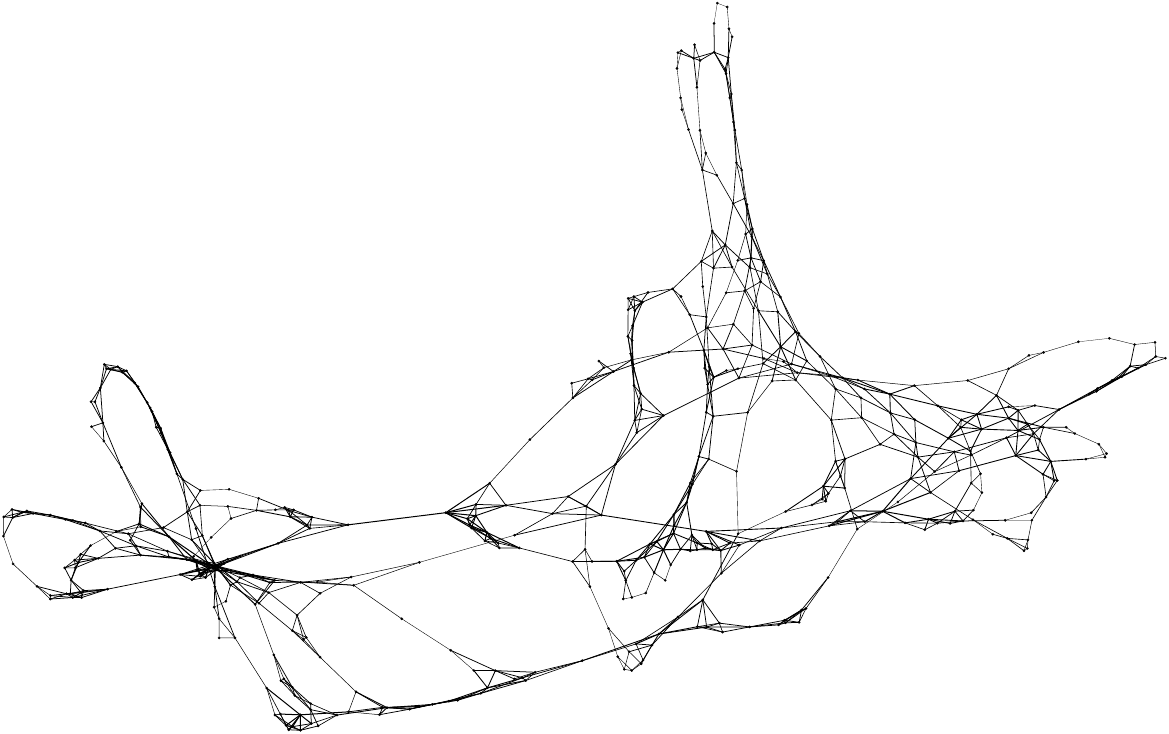}
  }%
  \subcaptionbox{\data{ENZYMES}}{%
    \includegraphics[width=0.33\linewidth]{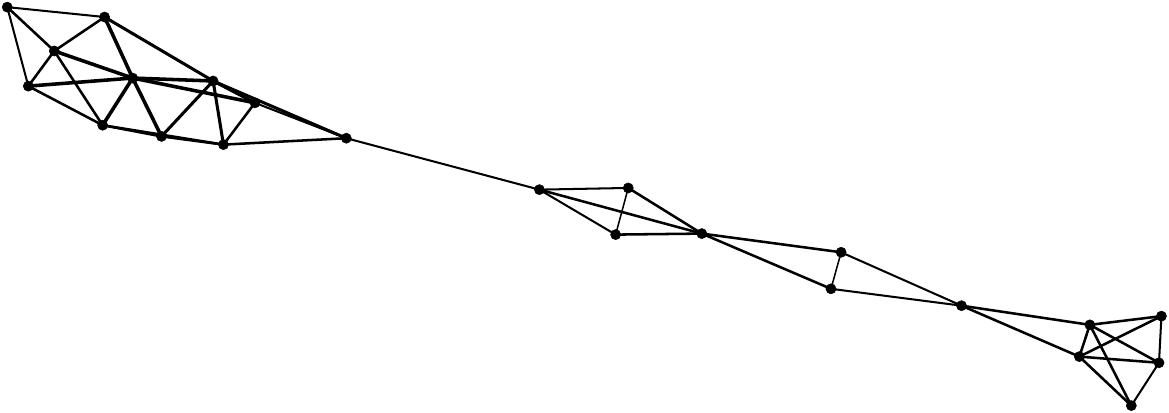}
  }%
  \subcaptionbox{\data{PROTEINS}}{%
    \includegraphics[width=0.125\linewidth, angle=90]{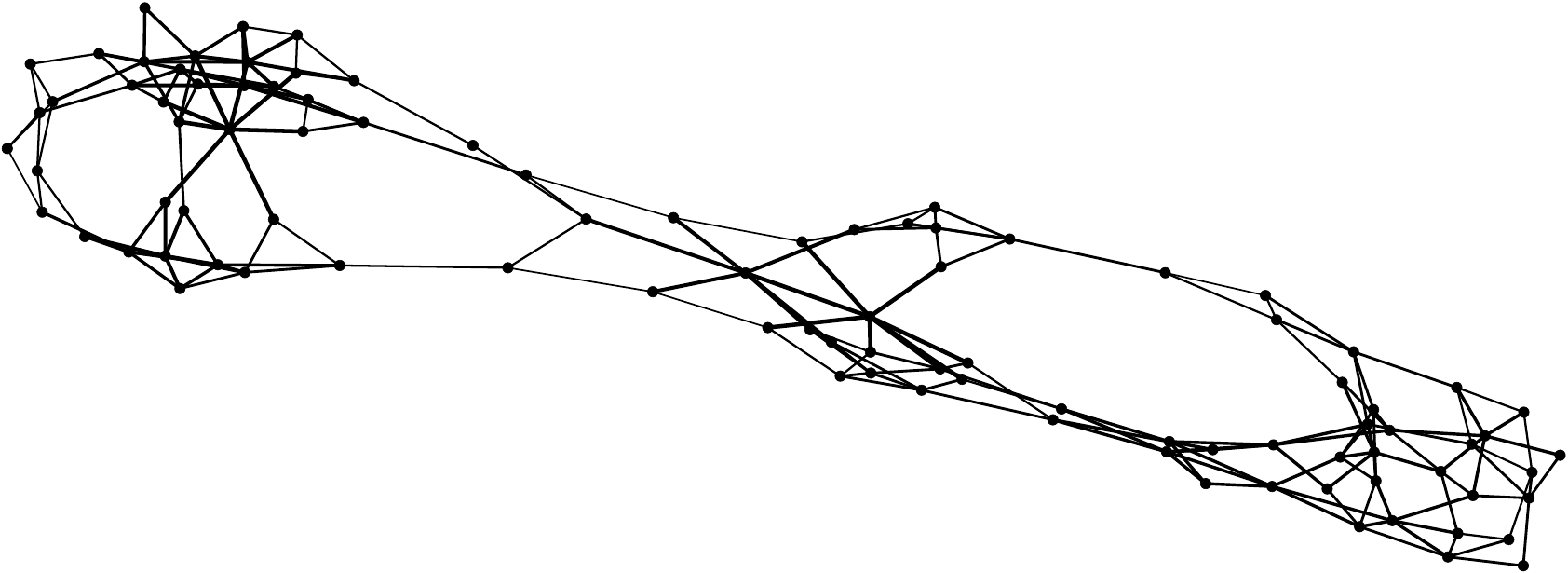}
  }%
  \caption{%
    Example graphs of the benchmark data sets that describe molecular
    structures. These graphs give rise to complex topological structures
    that can be exploited.
  }
  \label{fig:Molecular graphs}
\end{figure}
%%%%%%%%%%%%%%%%%%%%%%%%%%%%%%%%%%%%%%%%%%%%%%%%%%%%%%%%%%%%%%%%%%%%%%%%

%%%%%%%%%%%%%%%%%%%%%%%%%%%%%%%%%%%%%%%%%%%%%%%%%%%%%%%%%%%%%%%%%%%%%%%%
\begin{figure}[tbp]
  \centering
  \subcaptionbox{\data{IMDB-BINARY}}{%
    \includegraphics[width=0.25\linewidth]{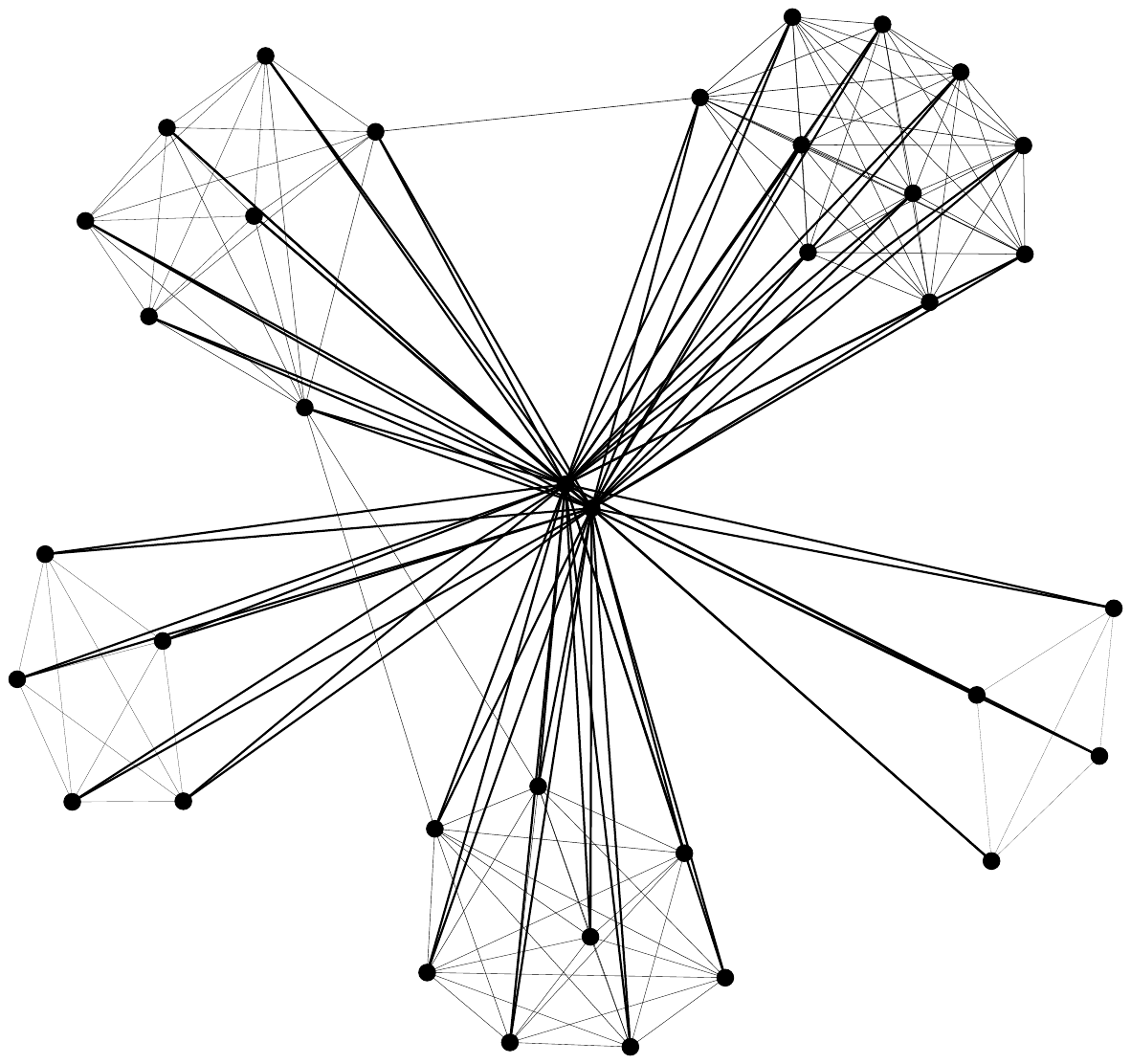}
  }%
  \subcaptionbox{\data{REDDIT-BINARY}}{%
    \includegraphics[width=0.25\linewidth]{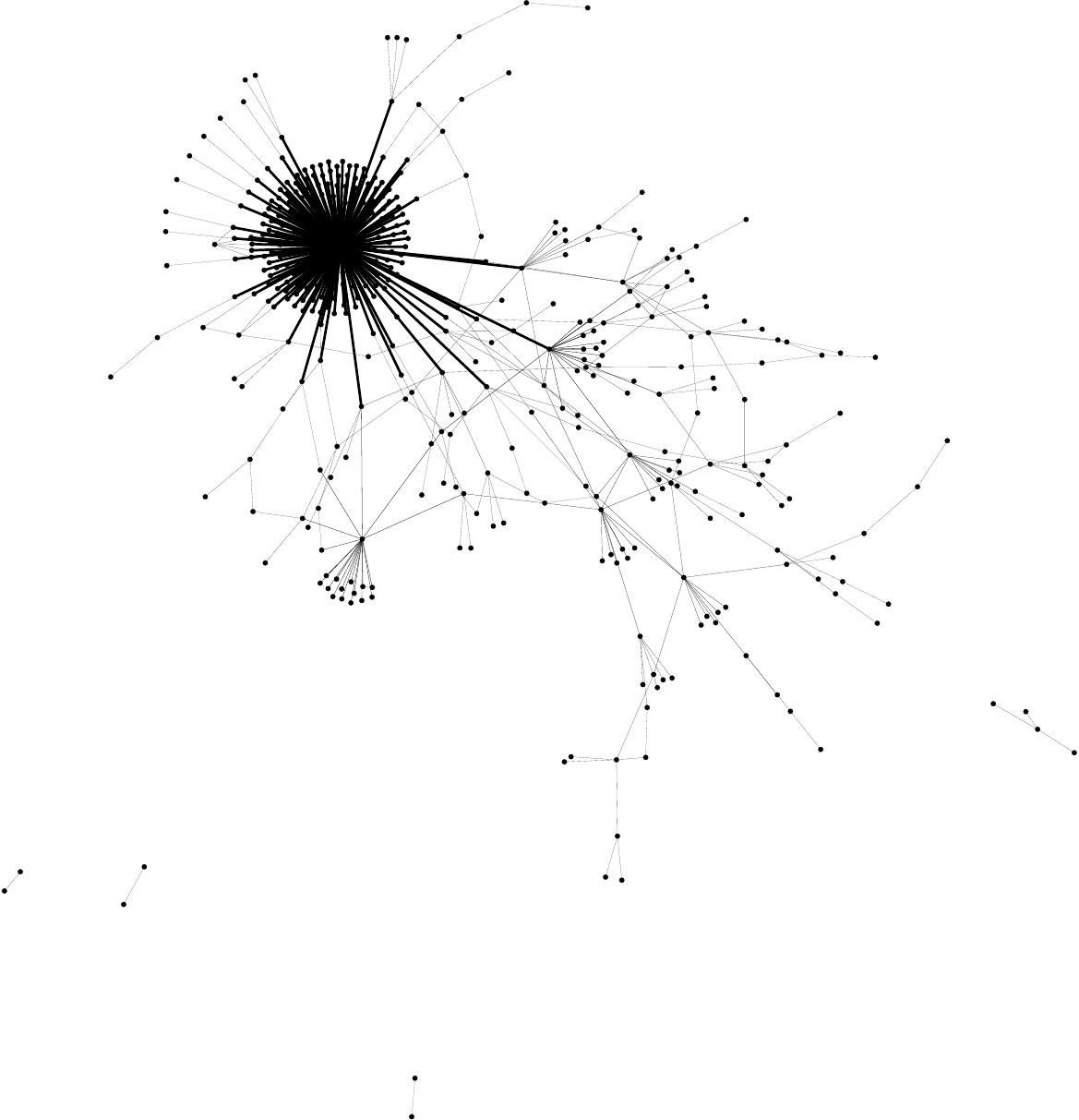}
  }
  \caption{%
    Example graphs of the benchmark data sets that correspond to
    extracted social networks. These graphs are best described in terms
    of clique connectivity; the existence of a single `hub' node does
    not give rise to a complex topological structures that can be
    exploited \emph{a priori}.
  }
  \label{fig:Social network graphs}
\end{figure}
%%%%%%%%%%%%%%%%%%%%%%%%%%%%%%%%%%%%%%%%%%%%%%%%%%%%%%%%%%%%%%%%%%%%%%%%

%%%%%%%%%%%%%%%%%%%%%%%%%%%%%%%%%%%%%%%%%%%%%%%%%%%%%%%%%%%%%%%%%%%%%%%%
\begin{figure}
  \centering
  \subcaptionbox{Filtration 0}{%
    \includegraphics[width=0.33\linewidth]{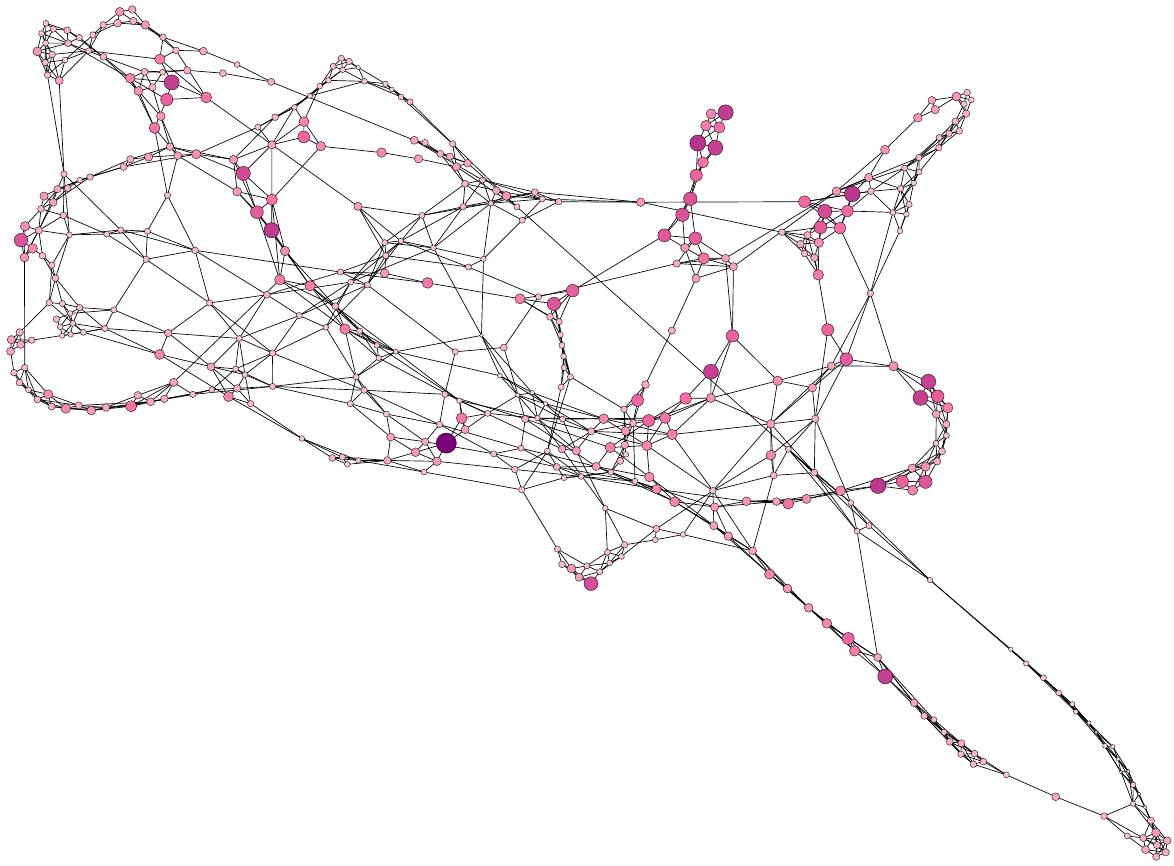}
  }%
  \subcaptionbox{Filtration 1}{%
    \includegraphics[width=0.33\linewidth]{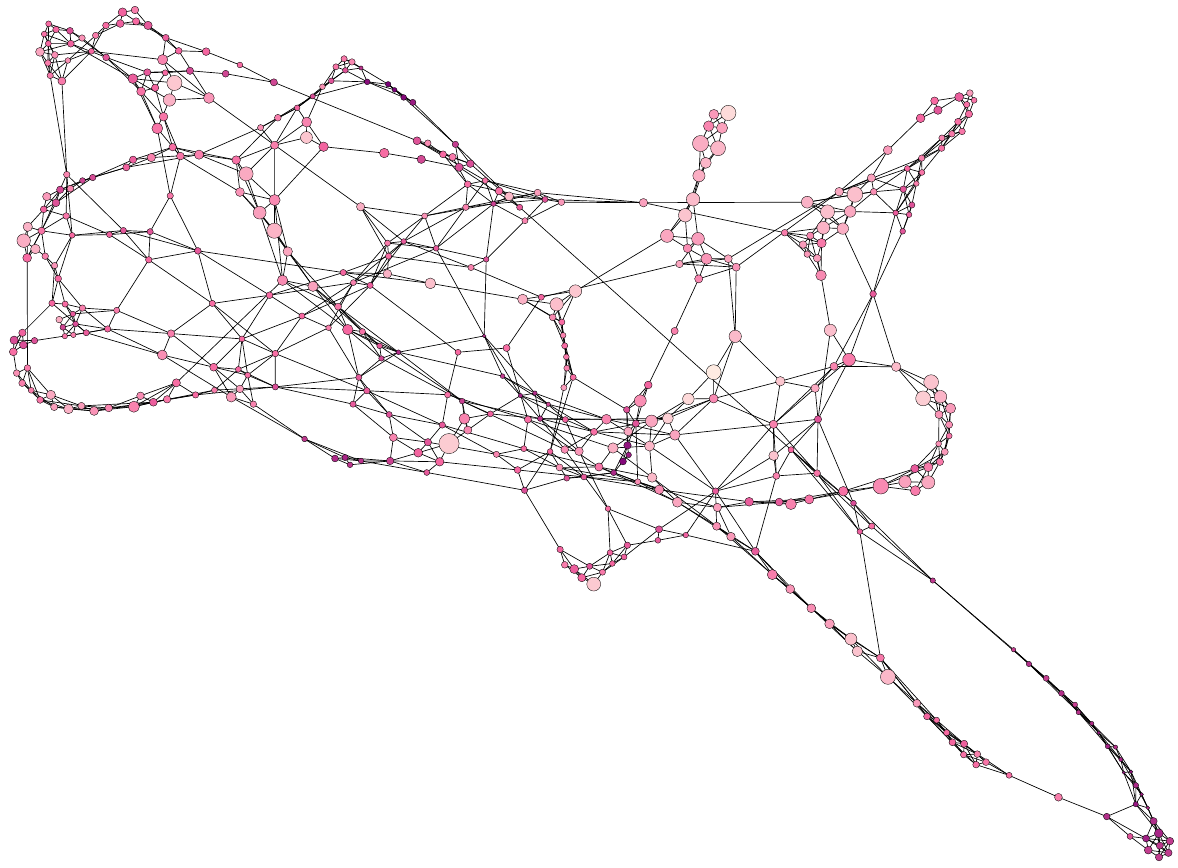}
  }%
  \subcaptionbox{Filtration 2}{%
    \includegraphics[width=0.33\linewidth]{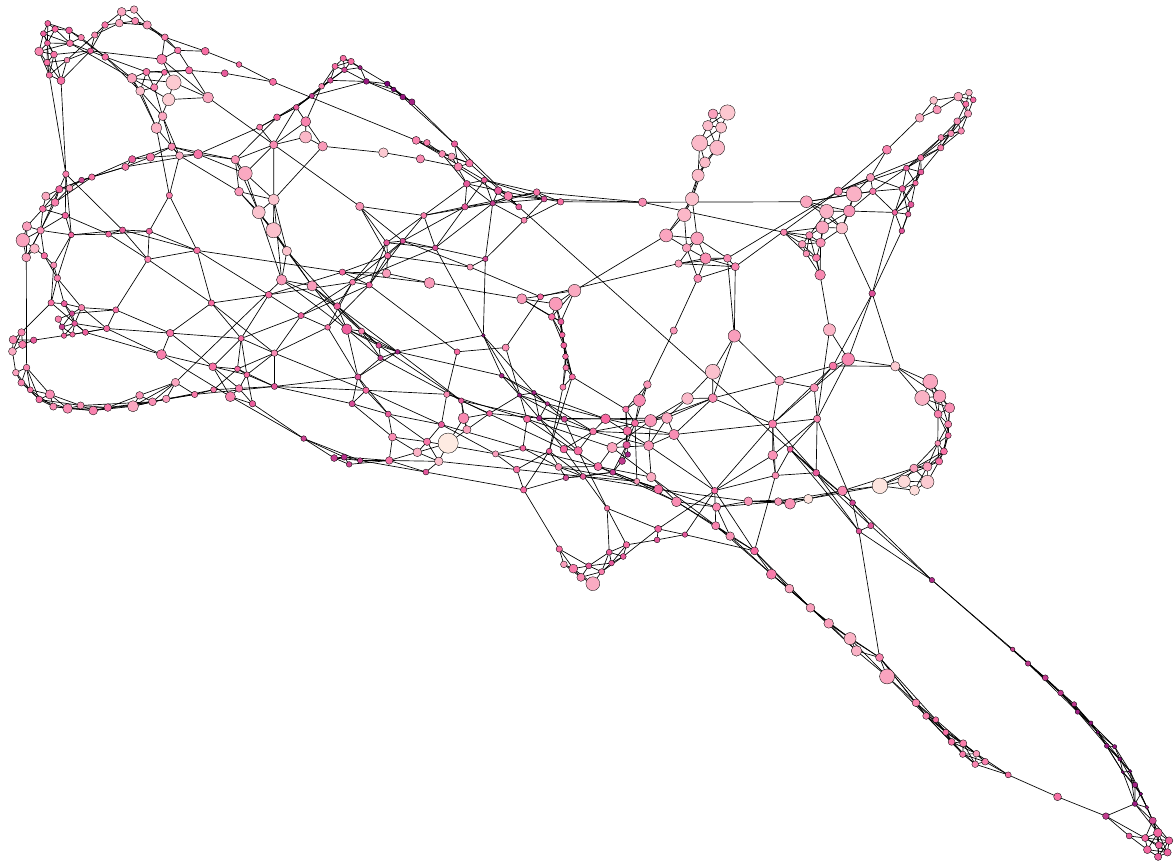}
  }%
  \caption{%
    Examples of different filtrations jointly learnt on an example graph
    randomly picked from the \data{DD} data set. The width of each node
    dot is proportional to its node degree, the colour saturation is
    proportional to the filtration value.
  }
  \label{fig:Filtrations}
\end{figure}
%%%%%%%%%%%%%%%%%%%%%%%%%%%%%%%%%%%%%%%%%%%%%%%%%%%%%%%%%%%%%%%%%%%%%%%%

%%%%%%%%%%%%%%%%%%%%%%%%%%%%%%%%%%%%%%%%%%%%%%%%%%%%%%%%%%%%%%%%%%%%%%%%
\section{Experiments on Layer Placement}\label{sec:Layer placement}
%%%%%%%%%%%%%%%%%%%%%%%%%%%%%%%%%%%%%%%%%%%%%%%%%%%%%%%%%%%%%%%%%%%%%%%%

A priori, it is not clear where to place \method. We therefore
investigated the impact of putting \method \emph{first}~(thus imbuing
the graph with topological features arising from the neighbours of
nodes) or at another position in the network. In Table \autoref{tab:layer placement necklaces}, we investigate the performance of the layer when \method is placed at different positions in a GNN architecture composed of 3 graph convolution layers. It appears that for this dataset, positioning \method before the last GNN layer leads to best performance when the non-static version is considered. Importantly, this contrasts with the approach of \citet{hofer2019learning}, where topological information is only available at the read-out level, which would lead to the worst performance on this dataset.

We complemented our experiments on layer placement with an exhaustive assessment of the performance of our model on on the  IMDB-Binary, Reddit-Binary, Proteins and DD datasets. Please refer to
\autoref{tab:Layer placement} for these results.

%%%%%%%%%%%%%%%%%%%%%%%%%%%%%%%%%%%%%%%%%%%%%%%%%%%%%%%%%%%%%%%%%%%%%%%%
\begin{table}
  \caption{%
    Test accuracies for different layer positions. \method{} can be
    placed at different positions. When placed at position~$0$, all
    downstream layers incorporate topological information, but only
    as much as can be gleaned from the nodes of graph. By contrast,
    placing \method{} \emph{last} makes a \texttt{readout} function
    that incorporates information from multiple filtrations.
  }
  \label{tab:layer placement necklaces}
  \newcommand{\gr}{\rowfont{\color{gray}}}
  \resizebox{\linewidth}{!}{%
    \begin{tabu}{llcccc}
    \toprule
    \textsc{Layer position} & \textsc{Method} & \data{Necklaces}  \\
    \midrule
    3 & GCN-3-\method-1 (no interaction)        & $97.17 \pm 0.4$  \\
      & GCN-3-\method-1 (static - no interaction) & $94.5 \pm 2.9$  \\
     
    \midrule
    2 & GCN-3-\method-1 (no interaction)        & $98.8 \pm 0.8$  \\
  & GCN-3-\method-1 (static - no interaction) & $77.2 \pm 4.5$  \\
 
    \midrule
    1 & GCN-3-\method-1 (no interaction)        & $97.4 \pm 1.8$  \\
  & GCN-3-\method-1 (static - no interaction) & $66.3 \pm 5.8$  \\
 
    \midrule
    0 & GCN-3-\method-1 (no interaction)        & $98.6 \pm 2.1$  \\
  & GCN-3-\method-1 (static - no interaction) & $61.8 \pm 5.3$  \\

    \bottomrule
        \end{tabu}
  }
\end{table}
%%%%%%%%%%%%%%%%%%%%%%%%%%%%%%%%%%%%%%%%%%%%%%%%%%%%%%%%%%%%%%%%%%%%%%%%

%%%%%%%%%%%%%%%%%%%%%%%%%%%%%%%%%%%%%%%%%%%%%%%%%%%%%%%%%%%%%%%%%%%%%%%%
\begin{table}
  \caption{%
    Test accuracies for different layer positions. \method{} can be
    placed at different positions. When placed at position~$0$, all
    downstream layers incorporate topological information, but only
    as much as can be gleaned from the nodes of graph. By contrast,
    placing \method{} \emph{last} makes a \texttt{readout} function
    that incorporates information from multiple filtrations.
  }
  \label{tab:Layer placement}
  \newcommand{\gr}{\rowfont{\color{gray}}}
  \resizebox{\linewidth}{!}{%
    \begin{tabu}{llcccc}
    \toprule
    \textsc{Layer position} & \textsc{Method} & \data{IMDB-Binary} &  \data{Reddit-Binary} & \data{Proteins-Full} & \data{DD} \\
    \midrule
    4 & GCN-3-\method-1  (interaction)         & $74.9 \pm 4.0$ & $92.7 \pm 1.4$ & $74.9 \pm 3.3$ & $71.5 \pm 4.5$  \\
      & GCN-3-\method-1 (no interaction)        & $71.6 \pm 2.1$ & $89.4 \pm 2.1$ & $75.9 \pm 4.0$ & $74.8 \pm 2.0$ \\
        & GCN-3-\method-1       & $73.4 \pm 3.2$ & $90.0 \pm 2.8$ & $75.6 \pm 4.0$ & $72.3 \pm 4.6$ \\
        
      & GCN-3-\method-1 (static - interaction) & $74.2 \pm 3.7$ & $91.9 \pm 1.6$ & $ 75.2 \pm 2.7$ & $71.1 \pm 5.1$  \\
      & GCN-3-\method-1 (static - no interaction) & $73.8 \pm 4.8$ & $89.4 \pm 2.1$ & $75.8 \pm 4.0$ & $70.9 \pm 2.6$ \\
      
    & GCN-3-\method-1 (static)      & $74.2 \pm 4.7$ & $92.3 \pm 2.3$ & $75.2 \pm 3.9$ & $70.3 \pm 5.0$\\
     
        \midrule
    3 & GCN-3-\method-1 (interaction)         & $70.6 \pm 5.6$ & $92.2 \pm 1.3$ & $75.5 \pm 3.1$ & $73.0 \pm 2.8$ \\
      & GCN-3-\method-1 (no interaction)        & $76.1 \pm 3.9$  & $90.4 \pm 1.6$ & $75.7 \pm 3.1$ & $75.5 \pm 3.4$ \\
      
    & GCN-3-\method-1       & $74.8 \pm 5.7$ & $91.0 \pm 1.7$ & $73.9 \pm 3.4$ & $73.9 \pm 3.4$ \\
      
     & GCN-3-\method-1 (static - interaction) & $72.0 \pm 3.0$ & $92.9 \pm 1.2$ & $76.0 \pm 2.1$ & $71.5 \pm 5.8$ \\
     & GCN-3-\method-1 (static - no interaction) & $74.0 \pm 6.8$  & $90.8 \pm 4.4$ & $75.5 \pm 3.8 $ & $71.8 \pm 4.1$ \\
     
    & GCN-3-\method-1 (static)       & $73.4 \pm 6.2$ & $92.2 \pm 1.3$ & $75.8 \pm 2.7$ & $71.6 \pm 4.5$ \\

         \midrule
    2 & GCN-3-\method-1 (interaction)        & $73.2 \pm 2.1$  & $91.7 \pm 0.5$ & $ 75.4 \pm 2.9$ & $74.2 \pm 2.7$  \\
        & GCN-3-\method-1 (no interaction)        & $74.9 \pm 2.3$ & $ 87.6 \pm 4.1$ & $75.5 \pm 4.3$ & $74.7 \pm 4.3$\\
    & GCN-3-\method-1       & $74.8 \pm 1.9$ & $89.6 \pm 2.2$ & $75.4 \pm 4.1$ & $74.9 \pm 3.5$\\
    
     & GCN-3-\method-1 (static - interaction) & $72.2 \pm 4.4$ & $91.6 \pm 2.9$ & $76.0 \pm 4.1$ & $72.7 \pm 3.9$ \\
     & GCN-3-\method-1 (static - no interaction) & $72.4 \pm 4.3$ & $87.6 \pm 4.1$ &$74.4 \pm 4.1$ & $70.7 \pm 3.4$ \\
     
        & GCN-3-\method-1  (static)     & $73.6 \pm 4.3$ & $92.4 \pm 0.8$ & $74.5 \pm 3.8$ & $71.5 \pm 4.1$  \\ 
    
         \midrule
    1 & GCN-3-\method-1   (interaction)       & $69.6 \pm 4.3$ & $91.1 \pm 0.7$ & $75.7 \pm 2.6$ & $74.1 \pm 4.1$  \\
        & GCN-3-\method-1 (no interaction)        & $71.8 \pm 2.7$ & $89.6 \pm 1.7$ & $75.3 \pm 4.2$ & $75.6 \pm 3.5$ \\
        
            & GCN-3-\method-1       & $70.6 \pm 3.5$ & $89.9 \pm 1.6$ & $75.2 \pm 3.9$ & $74.7 \pm 3.0$ \\

     & GCN-3-\method-1 (static - interaction) & $73.4 \pm 4.6$ & $90.5 \pm 1.4$ & $76.3 \pm 3.4$ & $73.2 \pm 3.6$\\
     & GCN-3-\method-1 (static - no interaction) & $72.4 \pm 3.1$ & $89.6 \pm 1.7$ & $75.2 \pm 3.4$ & $72.7 \pm 3.8$ \\
     
         & GCN-3-\method-1  (static)     & $73.4 \pm 4.6$ & $90.1 \pm 0.9$ & $75.6 \pm 3.8$ & $72.5 \pm 4.3$ \\
    \bottomrule
        \end{tabu}
  }
\end{table}
%%%%%%%%%%%%%%%%%%%%%%%%%%%%%%%%%%%%%%%%%%%%%%%%%%%%%%%%%%%%%%%%%%%%%%%%

%%%%%%%%%%%%%%%%%%%%%%%%%%%%%%%%%%%%%%%%%%%%%%%%%%%%%%%%%%%%%%%%%%%%%%%%
\section{Extended Experiments for Structured-Based Graph Classification}
\label{sec:Structured-Based Graph Classification Extended}
%%%%%%%%%%%%%%%%%%%%%%%%%%%%%%%%%%%%%%%%%%%%%%%%%%%%%%%%%%%%%%%%%%%%%%%%

This section contains more results for the structured-based
classification of graph data sets shown in \autoref{sec:Structured-Based
Graph Classification} and \autoref{fig:MNIST structural results}.
Table~\autoref{tab:Structural Results Extended} contains a detailed
listing of performance values under certain additional ablation
scenarios, such as a disentangled comparison of the performance with
respect to picking a certain type of embedding scheme for topological
features, i.e.\ deep sets versus known embedding functions for
persistence diagrams, the latter of which cannot handle
\emph{interactions} between tuples in a diagram. The version labelled \emph{General} consist of non-static versions where the type of feature embedding scheme is considered as an hyper-parameter. We then report the test accuracy of the version whose validation accuracy was highest. The \emph{General -Static} variant is similar but for the \emph{static} version of \method. 
\autoref{fig:Structure-Based Graph Classification} depicts the
performance on different data sets, focusing on a comparison between the
different embedding types and the static variant of our layer.

The behaviour of our static variant indicates that predictive
performance is \emph{not} driven by having access to topological
information on the benchmark data sets. To some extent, this is
surprising, as some of the data sets turn out to contain salient
topological features even visual inspection~(\autoref{fig:Molecular
graphs}), whereas the social networks exhibit more community-like
behaviour~(\autoref{fig:Social network graphs}). Together with the
results from \autoref{sec:Structured-Based Graph Classification},
this demonstrates that existing data sets often `leak' structural
information through their labels/features. Yet, for the node classification data sets \data{Cluster} and \data{Pattern}, 
 \mbox{GCN-3-\method-1} or its static variant perform best among all comparison
partners.\footnote{%
  On these data sets, our results on \mbox{GCN-4} differ slightly due to
  a known misalignment in the implementation of
  \citet{dwivedi2020benchmarking}, as compared to the original GCN
  architecture.
}

%%%%%%%%%%%%%%%%%%%%%%%%%%%%%%%%%%%%%%%%%%%%%%%%%%%%%%%%%%%%%%%%%%%%%%%%
\begin{figure}[tbp]
  \centering
  \subcaptionbox{Performance for \data{DD-struct}\label{sfig:Results DD random}}{%
    \begin{tikzpicture}
      \begin{axis}[%
        perfplot,
        error bars/y dir      = both,
        error bars/y explicit = true,
        ymin  = 55.0,
        ymax  = 80.0,
        ytick = {60, 65, 70, 75, 80},
      ]
        \addplot+[%
          discard if not = {dataset}{DD},
        ] table[x = depth, y expr = 100 * \thisrow{test_acc_mean}, y error expr = 100 * \thisrow{test_acc_std}, col sep = comma] {Data/NoFEATS_GCN_summary.csv};

        \addlegendentry{GCN}

        \addplot+[%
          discard if not = {dataset}{DD},
          discard if not = {fake}{False},
          discard if not = {deepset}{True},
        ] table[x = depth, y expr = 100 * \thisrow{test_acc_mean} , y error expr = 100 * \thisrow{test_acc_std}, col sep = comma] {Data/NoFEATS_TopoGNN_summary.csv};

        \addlegendentry{\method (interaction)}

        \pgfplotsset{cycle list shift=-1} 

        \addplot+[%
          discard if not = {dataset}{DD},
          discard if not = {fake}{False},
          discard if not = {deepset}{False},
          dashed,
        ] table[x = depth, y expr = 100 * \thisrow{test_acc_mean} , y error expr = 100 * \thisrow{test_acc_std}, col sep = comma] {Data/NoFEATS_TopoGNN_summary.csv};

        \addlegendentry{\method (no interaction)}

        \addplot+[%
          discard if not = {dataset}{DD},
          discard if not = {fake}{True},
          discard if not = {deepset}{True},
        ] table[x = depth, y expr = 100 * \thisrow{test_acc_mean} , y error expr = 100 * \thisrow{test_acc_std}, col sep = comma] {Data/NoFEATS_TopoGNN_summary.csv};

        \addlegendentry{\method (static)}
      \end{axis}
    \end{tikzpicture}
  }%
  \subcaptionbox{Performance for \data{ENZYMES-struct}\label{sfig:Results ENZYMES random}}{%
    \begin{tikzpicture}
      \begin{axis}[%
        perfplot,
        error bars/y dir      = both,
        error bars/y explicit = true,
        ymin  = 0.00,
        ymax  = 40.0,
        ytick = {0, 10, 20, 30, 40},
      ]
        \addplot+[%
          discard if not = {dataset}{Enzymes},
        ] table[x = depth, y expr = 100 * \thisrow{test_acc_mean}, y error expr = 100 * \thisrow{test_acc_std}, col sep = comma] {Data/NoFEATS_GCN_summary.csv};

        \addlegendentry{GCN}

        \addplot+[%
          discard if not = {dataset}{Enzymes},
          discard if not = {fake}{False},
          discard if not = {deepset}{True},
        ] table[x = depth, y expr = 100 * \thisrow{test_acc_mean}, y error expr = 100 * \thisrow{test_acc_std}, col sep = comma] {Data/NoFEATS_TopoGNN_summary.csv};

        \addlegendentry{\method (interaction)}

        \pgfplotsset{cycle list shift=-1} 

        \addplot+[%
          discard if not = {dataset}{Enzymes},
          discard if not = {fake}{False},
          discard if not = {deepset}{False},
          dashed,
        ] table[x = depth, y expr = 100 * \thisrow{test_acc_mean}, y error expr = 100 * \thisrow{test_acc_std}, col sep = comma] {Data/NoFEATS_TopoGNN_summary.csv};

        \addlegendentry{\method (no interaction)}

        \addplot+[%
          discard if not = {dataset}{Enzymes},
          discard if not = {fake}{True},
          discard if not = {deepset}{True},
        ] table[x = depth, y expr = 100 * \thisrow{test_acc_mean}, y error expr = 100 * \thisrow{test_acc_std}, col sep = comma] {Data/NoFEATS_TopoGNN_summary.csv};

        \addlegendentry{\method (static)}
      \end{axis}
    \end{tikzpicture}
  }
  \subcaptionbox{Performance for \data{PROTEINS-full-struct}\label{sfig:Results PROTEINS random}}{%
    \begin{tikzpicture}
      \begin{axis}[%
        perfplot,
        error bars/y dir      = both,
        error bars/y explicit = true,
        ymin  = 60,
        ymax  = 80,
        ytick = {60, 65, 70, 75, 80},
      ]
        \addplot+[%
          discard if not = {dataset}{Proteins},
        ] table[x = depth, y expr = 100 * \thisrow{test_acc_mean}, y error expr = 100 * \thisrow{test_acc_std}, col sep = comma] {Data/NoFEATS_GCN_summary.csv};

        \addlegendentry{GCN}

        \addplot+[%
          discard if not = {dataset}{Proteins},
          discard if not = {fake}{False},
          discard if not = {deepset}{True},
        ] table[x = depth, y expr = 100 * \thisrow{test_acc_mean} , y error expr = 100 * \thisrow{test_acc_std}, col sep = comma] {Data/NoFEATS_TopoGNN_summary.csv};

        \addlegendentry{\method (interaction)}

        \pgfplotsset{cycle list shift=-1} 

        \addplot+[%
          discard if not = {dataset}{Proteins},
          discard if not = {fake}{False},
          discard if not = {deepset}{False},
          dashed,
        ] table[x = depth, y expr = 100 * \thisrow{test_acc_mean} , y error expr = 100 * \thisrow{test_acc_std}, col sep = comma] {Data/NoFEATS_TopoGNN_summary.csv};

        \addlegendentry{\method (no interaction)}

        \addplot+[%
          discard if not = {dataset}{Proteins},
          discard if not = {fake}{True},
          discard if not = {deepset}{True},
        ] table[x = depth, y expr = 100 * \thisrow{test_acc_mean} , y error expr = 100 * \thisrow{test_acc_std}, col sep = comma] {Data/NoFEATS_TopoGNN_summary.csv};

        \addlegendentry{\method (static)}
      \end{axis}
    \end{tikzpicture}
  }%
  \subcaptionbox{Performance for \data{Cluster-struct}\label{sfig:Results Cluster random}}{%
    \begin{tikzpicture}
      \begin{axis}[%
        perfplot,
        error bars/y dir      = both,
        error bars/y explicit = true,
        ymin  = 16.0,
        ymax  = 17.0,
        ytick = {16.0, 16.5, 17.0},
      ]
        \addplot+[%
          discard if not = {dataset}{Cluster},
        ] table[x = depth, y expr = 100 * \thisrow{test_acc_mean}, y error expr = 100 * \thisrow{test_acc_std}, col sep = comma] {Data/NoFEATS_GCN_summary.csv};

        \addlegendentry{GCN}

        \addplot+[%
          discard if not = {dataset}{Cluster},
          discard if not = {fake}{False},
          discard if not = {deepset}{True},
        ] table[x = depth, y expr = 100 * \thisrow{test_acc_mean} , y error expr = 100 * \thisrow{test_acc_std}, col sep = comma] {Data/NoFEATS_TopoGNN_summary.csv};

        \addlegendentry{\method (interaction)}

        \pgfplotsset{cycle list shift=-1} 

        \addplot+[%
          discard if not = {dataset}{Cluster},
          discard if not = {fake}{False},
          discard if not = {deepset}{False},
          dashed,
        ] table[x = depth, y expr = 100 * \thisrow{test_acc_mean} , y error expr = 100 * \thisrow{test_acc_std}, col sep = comma] {Data/NoFEATS_TopoGNN_summary.csv};

        \addlegendentry{\method (no interaction)}

        \addplot+[%
          discard if not = {dataset}{Cluster},
          discard if not = {fake}{True},
          discard if not = {deepset}{True},
        ] table[x = depth, y expr = 100 * \thisrow{test_acc_mean} , y error expr = 100 * \thisrow{test_acc_std}, col sep = comma] {Data/NoFEATS_TopoGNN_summary.csv};

        \addlegendentry{\method (static)}
      \end{axis}
    \end{tikzpicture}
  }
  \subcaptionbox{Performance for \data{Pattern-struct}\label{sfig:Results Pattern random}}{%
    \begin{tikzpicture}
      \begin{axis}[%
        perfplot,
        error bars/y dir      = both,
        error bars/y explicit = true,
        ymin  = 83.00,
        ymax  = 87.0,
        ytick = {83, 85, 87},
      ]
        \addplot+[%
          discard if not = {dataset}{Pattern},
        ] table[x = depth, y expr = 100 * \thisrow{test_acc_mean}, y error expr = 100 * \thisrow{test_acc_std}, col sep = comma] {Data/NoFEATS_GCN_summary.csv};

        \addlegendentry{GCN}

        \addplot+[%
          discard if not = {dataset}{Pattern},
          discard if not = {fake}{False},
          discard if not = {deepset}{True},
        ] table[x = depth, y expr = 100 * \thisrow{test_acc_mean}, y error expr = 100 * \thisrow{test_acc_std}, col sep = comma] {Data/NoFEATS_TopoGNN_summary.csv};

        \addlegendentry{\method (interaction)}

        \pgfplotsset{cycle list shift=-1} 

        \addplot+[%
          discard if not = {dataset}{Pattern},
          discard if not = {fake}{False},
          discard if not = {deepset}{False},
          dashed,
        ] table[x = depth, y expr = 100 * \thisrow{test_acc_mean}, y error expr = 100 * \thisrow{test_acc_std}, col sep = comma] {Data/NoFEATS_TopoGNN_summary.csv};

        \addlegendentry{\method (no interaction)}

        \addplot+[%
          discard if not = {dataset}{Pattern},
          discard if not = {fake}{True},
          discard if not = {deepset}{True},
        ] table[x = depth, y expr = 100 * \thisrow{test_acc_mean}, y error expr = 100 * \thisrow{test_acc_std}, col sep = comma] {Data/NoFEATS_TopoGNN_summary.csv};

        \addlegendentry{\method (static)}
      \end{axis}
    \end{tikzpicture}
  }%
  \subcaptionbox{Performance for \data{MNIST-struct}\label{sfig:Results MNIST random}}{%
    \begin{tikzpicture}
      \begin{axis}[%
        perfplot,
        error bars/y dir      = both,
        error bars/y explicit = true,
        ymin  = 40,
        ymax  = 90,
        ytick = {45, 55, 65, 75, 85},
      ]
        \addplot+[%
          discard if not = {dataset}{MNIST},
        ] table[x = depth, y expr = 100 * \thisrow{test_acc_mean}, y error expr = 100 * \thisrow{test_acc_std}, col sep = comma] {Data/NoFEATS_GCN_summary.csv};

        \addlegendentry{GCN}

        \addplot+[%
          discard if not = {dataset}{MNIST},
          discard if not = {fake}{False},
          discard if not = {deepset}{True},
        ] table[x = depth, y expr = 100 * \thisrow{test_acc_mean} , y error expr = 100 * \thisrow{test_acc_std}, col sep = comma] {Data/NoFEATS_TopoGNN_summary.csv};

        \addlegendentry{\method (interaction)}

        \pgfplotsset{cycle list shift=-1} 

        \addplot+[%
          discard if not = {dataset}{MNIST},
          discard if not = {fake}{False},
          discard if not = {deepset}{False},
          dashed,
        ] table[x = depth, y expr = 100 * \thisrow{test_acc_mean} , y error expr = 100 * \thisrow{test_acc_std}, col sep = comma] {Data/NoFEATS_TopoGNN_summary.csv};

        \addlegendentry{\method (no interaction)}

        \addplot+[%
          discard if not = {dataset}{MNIST},
          discard if not = {fake}{True},
          discard if not = {deepset}{True},
        ] table[x = depth, y expr = 100 * \thisrow{test_acc_mean} , y error expr = 100 * \thisrow{test_acc_std}, col sep = comma] {Data/NoFEATS_TopoGNN_summary.csv};

        \addlegendentry{\method (static)}
      \end{axis}
    \end{tikzpicture}
  }%
  \caption{%
  Comparison of test accuracy for the structure-based variants of the
  benchmark data sets while varying network depth, i.e.\ the number of
  GCN layers. Error bars denote the standard deviation of test accuracy
  over $10$ cross-validation folds. \emph{Interaction} refers to using
  a \texttt{DeepSets} approach for embedding persistence diagrams,
  whereas \emph{no interaction} uses persistence diagrams coordinate
  functions that do not account for pairwise interactions.
  }
  \label{fig:Structure-Based Graph Classification}
\end{figure}
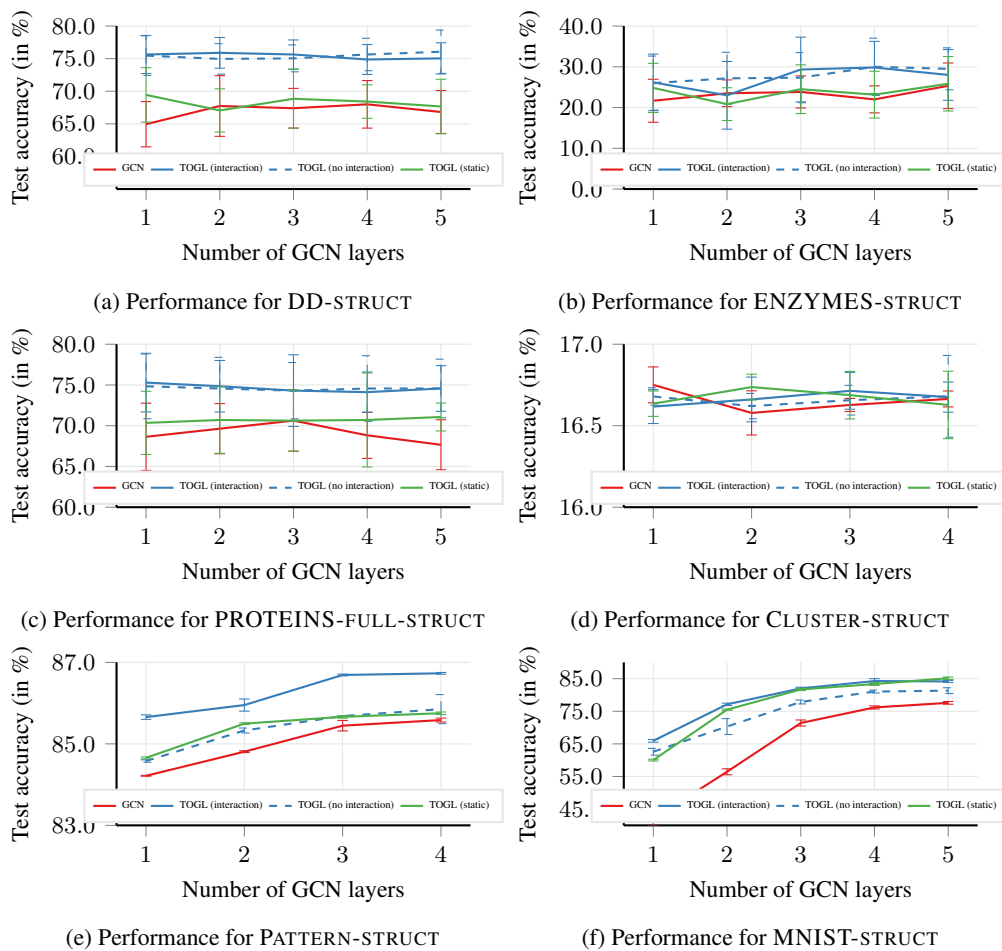
%%%%%%%%%%%%%%%%%%%%%%%%%%%%%%%%%%%%%%%%%%%%%%%%%%%%%%%%%%%%%%%%%%%%%%%%

%%%%%%%%%%%%%%%%%%%%%%%%%%%%%%%%%%%%%%%%%%%%%%%%%%%%%%%%%%%%%%%%%%%%%%%%
\begin{table}
  \caption{%
    Test Accuracies for different structural data sets.
  }
  \label{tab:Structural Results Extended}
  \newcommand{\gr}{\rowfont{\color{gray}}}
  \resizebox{\linewidth}{!}{%
\begin{tabu}{lrlllllllll}
\toprule
{} &  Depth &                     GCN &                     GIN &             GAT &       GCN-\method (General) & GCN-\method (General - Static) & GCN-\method (No Interaction) &   GCN-\method (Interaction) & GCN-\method (Static - No Interaction) & GCN-\method (Static-Interaction) \\
Data Set  &        &                         &                         &                 &                         &                            &                          &                         &                                   &                              \\
\midrule
DD       &      1 &           $64.9\pm 3.5$ &           $75.0\pm 2.3$ &   $59.8\pm 2.0$ &  $\mathbf{76.1\pm 2.7}$ &              $68.8\pm 4.1$ &            $75.5\pm 3.0$ &           $75.6\pm 2.9$ &                     $69.4\pm 4.2$ &                $66.7\pm 3.2$ \\
DD       &      2 &           $67.7\pm 4.7$ &           $75.4\pm 3.1$ &   $61.6\pm 4.3$ &           $74.8\pm 2.7$ &              $66.5\pm 2.8$ &            $75.0\pm 2.3$ &  $\mathbf{75.9\pm 2.4}$ &                     $67.1\pm 3.3$ &                $66.5\pm 3.2$ \\
DD       &      3 &           $67.4\pm 3.0$ &           $74.9\pm 3.3$ &   $62.0\pm 4.3$ &           $75.3\pm 2.4$ &              $68.8\pm 4.5$ &            $75.0\pm 2.1$ &  $\mathbf{75.6\pm 2.2}$ &                     $68.8\pm 4.5$ &                $65.6\pm 2.2$ \\
DD       &      4 &           $68.0\pm 3.6$ &  $\mathbf{75.6\pm 2.8}$ &   $63.3\pm 3.7$ &           $75.1\pm 2.1$ &              $68.0\pm 2.4$ &            $75.6\pm 2.5$ &           $74.9\pm 2.3$ &                     $68.4\pm 2.6$ &                $64.4\pm 4.8$ \\
DD       &      5 &           $66.8\pm 3.3$ &           $75.6\pm 2.1$ &   $62.6\pm 3.9$ &           $75.9\pm 3.3$ &              $67.6\pm 4.1$ &   $\mathbf{76.1\pm 3.3}$ &           $75.0\pm 2.4$ &                     $67.7\pm 4.2$ &                $68.3\pm 3.6$ \\
Proteins &      1 &           $68.6\pm 4.1$ &           $72.0\pm 3.4$ &   $66.8\pm 3.5$ &  $\mathbf{75.3\pm 3.4}$ &              $70.3\pm 2.6$ &            $74.8\pm 4.0$ &           $75.3\pm 3.6$ &                     $70.3\pm 3.9$ &                $71.0\pm 3.6$ \\
Proteins &      2 &           $69.6\pm 3.1$ &           $73.1\pm 4.1$ &   $66.8\pm 3.2$ &  $\mathbf{74.9\pm 3.5}$ &              $71.0\pm 3.5$ &            $74.6\pm 3.8$ &           $74.8\pm 3.2$ &                     $70.7\pm 4.1$ &                $71.2\pm 4.7$ \\
Proteins &      3 &           $70.6\pm 3.7$ &           $73.0\pm 4.4$ &   $67.0\pm 3.4$ &           $73.8\pm 4.3$ &              $70.5\pm 3.7$ &            $74.3\pm 3.5$ &  $\mathbf{74.3\pm 4.4}$ &                     $70.6\pm 3.7$ &                $70.0\pm 3.1$ \\
Proteins &      4 &           $68.8\pm 2.8$ &           $74.6\pm 3.1$ &   $67.5\pm 2.6$ &           $73.8\pm 3.7$ &              $71.2\pm 5.1$ &   $\mathbf{74.6\pm 4.0}$ &           $74.1\pm 2.5$ &                     $70.7\pm 5.8$ &                $70.6\pm 3.1$ \\
Proteins &      5 &           $67.7\pm 3.1$ &           $73.0\pm 3.2$ &   $65.7\pm 1.7$ &  $\mathbf{74.7\pm 3.3}$ &              $71.2\pm 1.9$ &            $74.6\pm 3.6$ &           $74.6\pm 2.8$ &                     $71.1\pm 1.7$ &                $71.1\pm 3.7$ \\
Enzymes  &      1 &           $21.7\pm 5.3$ &           $20.8\pm 4.2$ &   $15.3\pm 3.1$ &  $\mathbf{26.3\pm 6.4}$ &              $23.0\pm 6.3$ &            $26.0\pm 6.6$ &           $26.2\pm 6.9$ &                     $24.8\pm 6.0$ &                $21.8\pm 4.0$ \\
Enzymes  &      2 &           $23.5\pm 3.3$ &           $20.5\pm 5.3$ &   $17.5\pm 5.2$ &           $26.2\pm 8.1$ &              $21.2\pm 4.6$ &   $\mathbf{27.2\pm 6.4}$ &           $23.0\pm 8.3$ &                     $20.8\pm 4.0$ &                $21.5\pm 5.9$ \\
Enzymes  &      3 &           $23.8\pm 3.9$ &           $20.5\pm 6.0$ &   $16.3\pm 7.0$ &           $29.2\pm 6.1$ &              $24.3\pm 6.1$ &            $27.3\pm 6.1$ &  $\mathbf{29.3\pm 7.9}$ &                     $24.5\pm 6.0$ &                $21.2\pm 3.4$ \\
Enzymes  &      4 &           $22.0\pm 3.3$ &           $21.3\pm 6.5$ &   $21.7\pm 2.9$ &  $\mathbf{30.3\pm 6.5}$ &              $23.7\pm 5.4$ &            $30.0\pm 7.0$ &           $29.8\pm 6.4$ &                     $23.2\pm 5.7$ &                $22.7\pm 4.2$ \\
Enzymes  &      5 &           $25.3\pm 5.6$ &           $21.0\pm 4.4$ &   $19.8\pm 5.8$ &           $29.0\pm 5.2$ &              $26.8\pm 7.2$ &   $\mathbf{29.5\pm 5.2}$ &           $28.0\pm 6.2$ &                     $25.8\pm 6.7$ &                $26.8\pm 5.3$ \\
Pattern  &      1 &           $84.2\pm 0.0$ &           $69.5\pm 0.0$ &   $55.2\pm 3.6$ &  $\mathbf{85.7\pm 0.0}$ &              $84.6\pm 0.0$ &            $84.6\pm 0.0$ &           $85.7\pm 0.1$ &                     $84.7\pm 0.0$ &                $84.2\pm 0.0$ \\
Pattern  &      2 &           $84.8\pm 0.0$ &           $84.3\pm 0.0$ &   $55.1\pm 6.3$ &  $\mathbf{86.1\pm 0.0}$ &              $85.5\pm 0.0$ &            $85.3\pm 0.1$ &           $86.0\pm 0.1$ &                     $85.5\pm 0.0$ &                $84.9\pm 0.0$ \\
Pattern  &      3 &           $85.4\pm 0.1$ &           $84.7\pm 0.0$ &   $62.9\pm 5.2$ &  $\mathbf{86.7\pm 0.0}$ &              $85.7\pm 0.0$ &            $85.7\pm 0.0$ &           $86.7\pm 0.0$ &                     $85.7\pm 0.0$ &                $85.6\pm 0.0$ \\
Pattern  &      4 &           $85.6\pm 0.0$ &           $84.8\pm 0.0$ &   $58.3\pm 8.8$ &  $\mathbf{86.7\pm 0.0}$ &              $85.8\pm 0.0$ &            $85.9\pm 0.4$ &           $86.7\pm 0.0$ &                     $85.8\pm 0.0$ &                $85.7\pm 0.0$ \\
Cluster  &      1 &  $\mathbf{16.8\pm 0.1}$ &           $16.7\pm 0.1$ &   $16.7\pm 0.1$ &           $16.6\pm 0.0$ &              $16.7\pm 0.0$ &            $16.7\pm 0.1$ &           $16.6\pm 0.1$ &                     $16.6\pm 0.1$ &                $16.6\pm 0.0$ \\
Cluster  &      2 &           $16.6\pm 0.1$ &           $16.6\pm 0.1$ &   $16.6\pm 0.1$ &           $16.6\pm 0.0$ &              $16.6\pm 0.0$ &            $16.6\pm 0.1$ &           $16.7\pm 0.1$ &            $\mathbf{16.7\pm 0.1}$ &                $16.6\pm 0.1$ \\
Cluster  &      3 &           $16.6\pm 0.0$ &           $16.4\pm 0.2$ &   $16.7\pm 0.1$ &           $16.6\pm 0.0$ &              $16.6\pm 0.0$ &            $16.7\pm 0.1$ &  $\mathbf{16.7\pm 0.1}$ &                     $16.7\pm 0.1$ &                $16.6\pm 0.2$ \\
Cluster  &      4 &           $16.7\pm 0.0$ &           $16.4\pm 0.1$ &   $16.7\pm 0.0$ &           $16.8\pm 0.0$ &     $\mathbf{16.8\pm 0.0}$ &            $16.7\pm 0.3$ &           $16.7\pm 0.1$ &                     $16.6\pm 0.2$ &                $16.6\pm 0.3$ \\
MNIST    &      1 &           $42.0\pm 2.2$ &           $42.9\pm 0.0$ &   $40.4\pm 2.9$ &  $\mathbf{66.4\pm 0.0}$ &              $60.3\pm 0.0$ &            $62.6\pm 1.1$ &           $65.9\pm 0.4$ &                     $60.0\pm 0.3$ &                $56.4\pm 0.4$ \\
MNIST    &      2 &           $56.4\pm 0.9$ &           $68.3\pm 0.7$ &   $48.1\pm 7.3$ &  $\mathbf{77.4\pm 0.0}$ &              $75.3\pm 0.0$ &            $70.3\pm 2.4$ &           $77.1\pm 0.4$ &                     $75.5\pm 0.2$ &                $70.5\pm 0.9$ \\
MNIST    &      3 &           $71.4\pm 0.9$ &           $79.1\pm 0.3$ &   $68.7\pm 2.6$ &  $\mathbf{82.3\pm 0.0}$ &              $81.3\pm 0.0$ &            $77.9\pm 0.6$ &           $82.0\pm 0.3$ &                     $81.7\pm 0.3$ &                $78.8\pm 0.8$ \\
MNIST    &      4 &           $76.2\pm 0.5$ &           $83.4\pm 0.9$ &  $63.2\pm 10.4$ &  $\mathbf{84.8\pm 0.0}$ &              $82.9\pm 0.0$ &            $81.0\pm 0.4$ &           $84.3\pm 0.7$ &                     $83.4\pm 0.4$ &                $82.0\pm 0.9$ \\
MNIST    &      5 &           $77.6\pm 0.4$ &           $85.1\pm 0.6$ &  $56.8\pm 20.4$ &           $84.0\pm 0.0$ &              $85.0\pm 0.0$ &            $81.3\pm 0.9$ &           $84.2\pm 0.3$ &            $\mathbf{85.1\pm 0.4}$ &                $83.1\pm 0.5$ \\
\bottomrule
\end{tabu}
    
  }
\end{table}
%%%%%%%%%%%%%%%%%%%%%%%%%%%%%%%%%%%%%%%%%%%%%%%%%%%%%%%%%%%%%%%%%%%%%%%%

%%%%%%%%%%%%%%%%%%%%%%%%%%%%%%%%%%%%%%%%%%%%%%%%%%%%%%%%%%%%%%%%%%%%%%%%
\begin{table}[tbp]
  \centering
  \caption{%
    Results for the structure-based experiments. We depict the test
    accuracy obtained on various benchmark data sets when only
    considering structural information~(i.e.\ the network has access to
    \emph{uninformative} node features). Graph classification results
    are shown on the left, while node classification results are shown on
    the right.
  }
  \label{tab:Structural Results GIN}
  \scriptsize
  \sisetup{
    detect-all    = true,
    table-format  = 2.1,
    tight-spacing,
  }
  \setlength{\tabcolsep}{3.0pt}
  % These local re-definitions ensure that the `\pm` symbol will be
  % typeset correctly. Moreover, they correct the alignment.
  %
  \renewrobustcmd{\bfseries}{\fontseries{b}\selectfont}
  \renewrobustcmd{\boldmath}{}
  \let\b\bfseries
  \begin{tabular}{@{}lS@{$\pm$}S[table-format=1.1]S@{$\pm$}S[table-format=1.1]S@{$\pm$}S[table-format=2.1]S@{$\pm$}S[table-format=1.1]}
    \toprule
                            & \multicolumn{8}{c}{\scriptsize\itshape Graph classification}\\
    \midrule
    \textsc{Method}         & \multicolumn{2}{c}{\scriptsize\data{DD}}
                            & \multicolumn{2}{c}{\scriptsize\data{ENZYMES}}
                            & \multicolumn{2}{c}{\scriptsize\data{MNIST}}
                            & \multicolumn{2}{c}{\scriptsize\data{PROTEINS}}\\
    \midrule
    GAT-4                   &   63.3 &   3.7 &   21.7 &   2.9 &   63.2 &  10.4 &  67.5 &   2.6 \\
    GIN-4                   & 75.6 & 2.8 &   21.3 &   6.5 &   83.4 &   0.9 &\b74.6 & 3.1 \\
    \midrule
    GCN-4~(\emph{baseline})  &   68.0 &   3.6 &   22.0 &   3.3 &   76.2 &   0.5 &  68.8 &   2.8 \\
    GCN-3-\method-1          &   75.1 &   2.1 & \b 30.3 & \b 6.5 & \b84.8 & \b0.4 &  73.8 &   4.3 \\
    GCN-3-\method-1 (static) &   68.0 &   2.4 &   23.7 &   5.4 &   82.9 &   0.0 &  71.2 &   5.1 \\
    GIN-3-\method-1  &   76.2 &   2.4 & 23.7 & 6.9 & 84.4 & 1.1 &  73.9 &   4.9 \\
    GIN-3-\method-1 (static) & \b 76.3 & \b 2.8 & 25.2 & 7.0 & 83.9 & 0.1 & 74.2 & 4.2 \\
    GAT-3-\method-1  &   75.7 & 2.1 & 23.5 & 6.1 & 77.2 & 10.5 &  72.4 &   4.6 \\
    GAT-3-\method-1 (static) &  68.4 & 3.4 & 22.7 & 3.9 & 81.9 & 1.1 & 68.9 & 4.0 \\
    \bottomrule
  \end{tabular}
  \begin{tabular}{@{}S@{$\pm$}S[table-format=1.1]S@{$\pm$}S[table-format=1.1]}
    \toprule
    \multicolumn{4}{c}{\scriptsize\itshape Node classification}\\
    \midrule
                     \multicolumn{2}{c}{\scriptsize\data{Cluster}}
                   & \multicolumn{2}{c}{\scriptsize\data{Pattern}}\\
    \midrule
      16.7 &   0.0 &    58.3 &   8.8 \\
      16.4 &   0.1 &    84.8 &   0.0 \\
    \midrule                 
      16.7 &   0.0 &    85.6 &   0.0 \\
    \b16.8 & \b0.0 &  \b86.7 & \b0.0 \\
    \b16.8 & \b0.0 &    85.8 &   0.0 \\
    16.6 & 0.3 &  86.6 & 0.1 \\
    16.4 & 0.1 &    85.4 &   0.1 \\
    16.5 &   0.1 &    75.9 &   3.1 \\
    16.7 & 0.0 &    60.5 &   3.0 \\
    \bottomrule
  \end{tabular}
\end{table}
%%%%%%%%%%%%%%%%%%%%%%%%%%%%%%%%%%%%%%%%%%%%%%%%%%%%%%%%%%%%%%%%%%%%%%%%

%%%%%%%%%%%%%%%%%%%%%%%%%%%%%%%%%%%%%%%%%%%%%%%%%%%%%%%%%%%%%%%%%%%%%%%%
\begin{table}[tbp]                        
  \caption{%
    Test accuracy on benchmark data sets~(following standard
    practice, we report weighted accuracy on \data{CLUSTER} and
    \data{PATTERN}). Methods printed in
    black have been run in our setup, while methods printed in
    \textcolor{gray}{grey} are cited from the literature, i.e.\ \citet{dwivedi2020benchmarking}, \citet{morris2020tudataset} for \data{IMDB-B} and
    \data{REDDIT-B}, and \citet{Borgwardt20} for WL/\mbox{WL-OA} results
    GIN-4 results printed in \emph{italics} are
    1-layer GIN-$\epsilon$, as reported in \citet{morris2020tudataset}.
    Graph classification results
    are shown on the left, while node classification results are shown on
    the right.
  }
  \label{tab:Benchmark results GIN}
  \centering
  \newcommand{\gr}{\rowfont{\color{gray}}}
  \newcommand{\NA}{---}
  \small
  \setlength{\tabcolsep}{3.0pt}
  \resizebox{\linewidth}{!}{
  \begin{tabu}{lccccccc}
    \toprule
                    & \multicolumn{7}{c}{\scriptsize\itshape Graph classification}\\
    \midrule
    \textsc{Method} & \scriptsize\data{CIFAR-10}
                    & \scriptsize\data{DD}
                    & \scriptsize\data{ENZYMES}
                    & \scriptsize\data{MNIST}
                    & \scriptsize\data{PROTEINS-full}
                    & \scriptsize\data{IMDB-B}
                    & \scriptsize\data{REDDIT-B}\\
    \midrule
    \gr GAT-4               & $64.2 \pm 0.4$ & $\mathbf{75.9\pm3.8}$ & $\mathbf{68.5\pm5.2}$ & $95.5 \pm 0.2$ & $76.3\pm2.4$ &    \NA                  &   \NA\\
    \gr GATED-GCN-4         & $\mathbf{67.3 \pm 0.3}$ & $72.9\pm2.1$ & $65.7\pm4.9$ & $\mathbf{97.3 \pm 0.1}$ & $\mathbf{76.4\pm2.9}$ &    \NA                  &   \NA\\
    \gr GIN-4               & $55.5 \pm 1.5$ & $71.9\pm3.9$ & $65.3\pm6.8$ & $96.5 \pm 0.3$ & $74.1\pm3.4$ & $\mathit{72.9 \pm 4.7}$ & $\mathit{89.8\pm2.2}$\\
    \gr WL                  &      \NA       & $77.7\pm2.0$ & $54.3\pm0.9$ & \NA            & $73.1\pm0.5$ & $71.2\pm 0.5$           & $78.0\pm0.6$\\
    \gr WL-OA               &      \NA       & $77.8\pm1.2$ & $58.9\pm0.9$ & \NA            & $73.5\pm0.9$ & $74.0\pm 0.7$           & $87.6\pm0.3$\\
    \midrule
    GCN-4~(\emph{baseline}) & $54.2 \pm 1.5$ & $72.8\pm4.1$ & $65.8\pm4.6$ & $90.0 \pm 0.3$ & $76.1\pm2.4$ & $68.6\pm 4.9$           & $\mathbf{92.8\pm1.7}$\\
    GCN-3-\method-1& $61.7 \pm 1.0$ & $73.2\pm4.7$ & $53.0\pm9.2$
                   & $\mathit{95.5 \pm 0.2}$ & $76.0\pm3.9$ & $72.0\pm 2.3$           & $89.4\pm2.2$\\
    GCN-3-\method-1 (static)    & $62.1 \pm 0.5$ & $71.0\pm2.8$ & $49.8\pm7.0$ & $95.4 \pm 0.1$ & $75.7\pm3.6$ & $72.8\pm 5.4$           & $92.1\pm1.6$\\
    GIN-3-\method-1& $61.3 \pm 0.4$ & $75.2\pm4.2$ & $43.8 \pm 7.9$  & $96.1\pm 0.1$ & $73.6 \pm 4.8$ & $\mathbf{74.2\pm 4.2}$  & $\mathit{89.7\pm 2.5}$\\
    GIN-3-\method-1 (static)    & $61.8 \pm 0.6$ & $72.2\pm 5.3$ & $43.3 \pm 8.3$ & $96.4 \pm 0.1$ & $74.7 \pm 3.1$ & $73.8 \pm 2.4$           & $89.1 \pm 4.4$\\
    GAT-3-\method-1 & $52.8 \pm 3.4$ & $73.7 \pm 2.9$ & $51.5 \pm 7.3$  & $0.0 \pm 0.0$ & $75.2 \pm 3.9$ & $70.8 \pm 8.0$  & $82.5 \pm 8.7$\\
    GAT-3-\method-1 (static)    & $50.8 \pm 3.6$ & $72.8 \pm 3.4$ & $55.2 \pm 9.1$  & $96.2 \pm 0.3$ & $74.4 \pm 2.5$ & $68.7 \pm 9.4$  & $70.1 \pm 9.9$\\
    
    \bottomrule
  \end{tabu}
  \begin{tabu}{cc}
    \toprule
    \multicolumn{2}{c}{\scriptsize\itshape Node classification}\\
    \midrule
    \scriptsize\data{CLUSTER} & \scriptsize\data{PATTERN}\\
    \midrule
    \gr $57.7 \pm 0.3$        & $75.8 \pm 1.8$        \\
    \gr $60.4 \pm 0.4$        & $84.5 \pm 0.1$        \\
    \gr $58.4 \pm 0.2$        & $85.6 \pm 0.0$        \\
    \gr \NA                   & \NA\\ 
    \gr \NA                   & \NA\\ 
    \midrule
        $57.0 \pm 0.9$        & $85.5 \pm 0.4$        \\
        $60.4 \pm 0.2$        & $86.6 \pm 0.1$ \\
        $60.5 \pm 0.2$        & $85.6 \pm 0.1$ \\
        $60.4 \pm 0.2$        & $\mathbf{86.7 \pm 0.1}$ \\
        $\mathbf{60.6 \pm 0.3}$        & $85.5 \pm 0.1$ \\
        $58.4 \pm 3.7$        & $59.6 \pm 3.3$ \\
        $58.7 \pm 2.2$        & $64.5 \pm 14.2$ \\
    \bottomrule
  \end{tabu}
  }
\end{table}
%%%%%%%%%%%%%%%%%%%%%%%%%%%%%%%%%%%%%%%%%%%%%%%%%%%%%%%%%%%%%%%%%%%%%%%%

\section{Geometric Dataset}
\added{
In this section, we showcase the potential of \method for improving performance on tasks of geometric data sets. Tasks such as 3D object recognition are gathering increasing attention in the machine learning community and our approach represents a promising direction to explore that area. We consider a simple synthetic data set consisting of graphs linking points on geometrical objects. In particular, we create a data set with two classes where one consists of graphs lying on a sphere and the other of graphs lying on a torus. The nodes of these graphs are composed of points sampled on the geometrical objects and embedded in an higher dimensional space. The graph is then constructed as a fully connected graphs between all sampled points.  Table \ref{tab:spheres and torus} presents the classification accuracy of different methods on this experiment. We observe that substituting a GNN layer with a \method layer allows to obtain nearly perfect accuracy, showcasing the crucial importance of a topology-informed layer in geometric data sets.}

\begin{table}[tbp]
  \caption{%
    Test accuracy of the different methods on the Spheres vs. Torus classification task. We compare two architectures (\mbox{GCN-4},
    \mbox{GIN-4}) with corresponding models where one
    layer is replaced with \method and highlight the winner of each
    comparison in \textbf{bold}. 
  }
  \label{tab:spheres and torus}
  \centering
  \small
  \newcommand{\gr}{\rowfont{\color{gray}}}
  \newcommand{\NA}{---}
  \sisetup{
    detect-all           = true,
    table-format         = 2.1(2),
    separate-uncertainty = true,
    mode                 = text,
    table-text-alignment = center,
    detect-weight,
  }
  \robustify\bfseries
  \renewrobustcmd{\bfseries}{\fontseries{b}\selectfont}
  \renewrobustcmd{\boldmath}{}
  \let\b\bfseries
  \setlength{\tabcolsep}{3.0pt}
  \begin{tabu}{lS}
    \toprule
    \textsc{Method}          & {\scriptsize\data{Spheres vs. Torus}}
                             \\
    \midrule
    \midrule
    GCN-4         & 77.8 \pm 2.1 \\
    GCN-3-\method-1          &\b98.5 \pm 1.3 \\
    \midrule
    GIN-4         & 76.2 \pm 3.5 \\
    GIN-3-\method-1          &\b99.6 \pm 0.9 \\
    \bottomrule
  \end{tabu}
\end{table}
%%%%%%%%%%%%%%%%%%%%%%%%%%%%%%%%%%%%%%%%%%%%%%%%%%%%%%%%%%%%%%%%%%%%%%%%

\section{Larger number of layers in the GNN architecture}
\added{
In this section, we investigate the performance of our approach. on architecture with a larger number of layers. For computational reasons, we limit our analysis to 3 data sets, namely \data{DD}, \data{PROTEINS} and \data{CIFAR-10}. We re-use the experimental setup as detailed in Section \ref{sec:experimentsetup}. Tables \ref{tab:Structural Results more layers} and \ref{tab:Benchmark results more layers} show the classification accuracy for the structural and the benchmark setup respectively. What we observe is that the advantage of including \method is generally conserved, and similar to the one observed in the 4 layers case. In absolute terms, for the structural dataset, the performance on  the \data{DD} dataset increases with more layers (for both the baseline and the \method version) but decreases for the \data{PROTEINS} dataset. On the benchmark dataset (including node features), the performance tends to increase for all three datasets compared to the 4-layers version. 
}

%%%%%%%%%%%%%%%%%%%%%%%%%%%%%%%%%%%%%%%%%%%%%%%%%%%%%%%%%%%%%%%%%%%%%%%%
\begin{table}[tbp]
  \centering
  \caption{%
    Graph classification results for the structure-based experiments. We depict the test
    accuracy obtained on various benchmark data sets when only
    considering structural information~(i.e.\ the network has access to
    \emph{uninformative} node features). We compare two architectures (\mbox{GCN-4},
    \mbox{GIN-4}) with corresponding models where one
    layer is replaced with \method and highlight the winner of each
    comparison in \textbf{bold}. 
  }
  \label{tab:Structural Results more layers}
  \setlength{\tabcolsep}{3.0pt}
  \sisetup{
    detect-all              = true,
    table-format            = 2.1(2),
    separate-uncertainty    = true,
    mode                    = text,
    reset-text-shape        = false,
    table-text-alignment    = center,
    retain-zero-uncertainty = true,
    detect-weight,
  }
  % TODO: Remove if we have sufficient space.
  \scriptsize
  \robustify\bfseries
  \robustify\itshape
  \renewrobustcmd{\bfseries}{\fontseries{b}\selectfont}
  \renewrobustcmd{\boldmath}{}
  \let\b\bfseries
  \let\i\itshape
  \setlength{\tabcolsep}{3.0pt}
  \begin{tabular}{@{}lSSS}
    \toprule
                            & \multicolumn{2}{c}{\scriptsize\itshape Graph classification}\\
    \midrule
    \textsc{Method}         & \multicolumn{1}{c}{\scriptsize\data{DD}}
                            &  \multicolumn{1}{c}{\scriptsize\data{PROTEINS}} &  \multicolumn{1}{c}{\scriptsize\data{CIFAR-10}} \\
    \midrule
    GCN-6          &   70.5 \pm 4.9 &   70.1 \pm 4.0 & 45.4 \pm 0.5 \\
    GCN-5-\method-1 & \b75.9 \pm 2.5 & \b73.5 \pm 3.5 & \b 54.4 \pm 0.7 \\
    \midrule
    GCN-8         &   68.5 \pm 3.6 &   72.0 \pm 3.3 & 45.6 \pm 0.5 \\
    GCN-7-\method-1 & \b75.6 \pm 2.9 &  \b73.5 \pm 3.3 & \b 54.1 \pm 1.1 \\
     \midrule
    GIN-6          &   76.1 \pm 2.4 &  \b 72.9 \pm 3.4 & 53.6 \pm 1.4\\
    GIN-5-\method-1 & \b 76.5 \pm 2.0 & 72.1 \pm 3.6 & \b 54.1 \pm 1.2 \\
    \midrule
    GIN-8         & \b  76.5 \pm 2.4 & \b  72.9 \pm 4.0 & 51.4 \pm 2.4 \\
    GIN-7-\method-1 & 76.0 \pm 3.6 & 72.6 \pm 4.1 & \b 55.3 \pm 0.7\\  
   
    \bottomrule
  \end{tabular}
\end{table}
%%%%%%%%%%%%%%%%%%%%%%%%%%%%%%%%%%%%%%%%%%%%%%%%%%%%%%%%%%%%%%%%%%%%%%%%

%%%%%%%%%%%%%
\begin{table}[tbp]                        
  \caption{%
    Test accuracy on benchmark data sets.
    %
    %GIN-4 results printed in \emph{italics} are
    %1-layer GIN-$\epsilon$, as reported in \citet{morris2020tudataset}.
    %
    Graph classification results
    are shown on the left; node classification results are shown on
    the right.
    Following \autoref{tab:Structural Results}, we take existing
    architectures and replace their second layer by \method; we
    use \emph{italics} to denote the winner of each comparison.
    A \textbf{bold} value indicates the overall winner of a column,
    i.e.\ a data set.
  }
  \label{tab:Benchmark results more layers}
  \centering
  \newcommand{\gr}{\rowfont{\color{gray}}}
  \newcommand{\NA}{---}
  \sisetup{
    detect-all           = true,
    table-format         = 2.1(2),
    separate-uncertainty = true,
    mode                 = text,
    reset-text-shape     = false,
    table-text-alignment = center,
    detect-weight,
  }
  \robustify\bfseries
  \robustify\itshape
  \renewrobustcmd{\bfseries}{\fontseries{b}\selectfont}
  \renewrobustcmd{\boldmath}{}
  \let\b\bfseries
  \let\i\itshape
  \setlength{\tabcolsep}{3.0pt}
  %\resizebox{\linewidth}{!}{
  \begin{tabu}{lSSS}
    \toprule
                    & \multicolumn{3}{c}{\scriptsize\itshape Graph classification}\\
    \midrule
    \textsc{Method} & {\scriptsize\data{CIFAR-10}}
                    & {\scriptsize\data{DD}}
                    & {\scriptsize\data{PROTEINS-full}}\\
    \midrule
    %\gr GAT-4               & $64.2 \pm 0.4$ & $\mathbf{75.9\pm3.8}$ & $\mathbf{68.5\pm5.2}$ & $95.5 \pm 0.2$ & $76.3\pm2.4$ &    \NA                  &   \NA\\
  
    \midrule
    GCN-6 & 53.6 \pm 0.6 & 72.7\pm2.8 & \b76.3\pm3.6 \\
    GCN-5-\method-1& \b 63.0 \pm 0.8 & \b76.0\pm 3.2 & 76.0\pm3.2 \\
    \midrule
    GCN-8 & 52.9 \pm 0.9 & 70.7\pm4.5 & 75.4\pm2.4 \\
    GCN-7-\method-1& \i61.6 \pm 0.5 & \i72.1\pm7.3 & \i75.6\pm3.5 \\
    \midrule
    GIN-6 & 55.7 \pm 2.8 & \i73.1\pm3.7 & \i74.1\pm3.1 \\
    GIN-5-\method-1& \i61.5 \pm 1.2 & 72.7\pm4.2 & 73.8\pm3.3 \\
    \midrule
    GIN-8 & 56.0 \pm 3.5 & 71.6\pm5.8 & \i74.8\pm3.2 \\
    GIN-7-\method-1& \i61.1 \pm 0.8 & \i72.4\pm4.6 & 74.3\pm2.5 \\

    \bottomrule
  \end{tabu}
  %}
\end{table}
%%%%%%%%%%%%%%%%%%%%%%%%%%%%%%%%%%%%%%%%%%%%%%%%%%%%%%%%%%%%%%%%%%%%%%%%

\end{document}